\documentclass{article}

\usepackage[utf8]{inputenc} 
\usepackage[T1]{fontenc}    
\usepackage{hyperref}       
\usepackage{url}            
\usepackage{booktabs}       
\usepackage{amsfonts}       
\usepackage{nicefrac}       
\usepackage{microtype}      

\usepackage[dvipsnames]{xcolor}

\usepackage{times}
\usepackage{graphicx} 
\usepackage{subfigure}
\usepackage[export]{adjustbox}

\usepackage{algorithm}
\usepackage{algorithmic}

\usepackage{appendix}
\usepackage{wrapfig}

\PassOptionsToPackage{numbers}{natbib}

\usepackage[final]{nips_2017}

\usepackage{mkolar_definitions}

\newcount\Comments  
\definecolor{darkgreen}{rgb}{0,0.5,0}
\definecolor{darkred}{rgb}{0.7,0,0}
\definecolor{teal}{rgb}{0.3,0.8,0.8}
\definecolor{blue}{rgb}{0,0,1}
\newcommand{\kibitz}[2]{\ifnum\Comments=1\textcolor{#1}{#2}\fi}

\usepackage[]{color-edits}
\addauthor{md}{teal}

\newenvironment{proof-of-theorem}[1][{}]{\noindent{\bf Proof of
    Theorem~{#1}} \hspace*{1em}}{\BlackBox\smallskip\\}

\newenvironment{carlist}
 {\begin{list}{$\bullet$}
 {\setlength{\topsep}{0in} \setlength{\partopsep}{0in}
  \setlength{\parsep}{0in} \setlength{\itemsep}{\parskip}
  \setlength{\leftmargin}{0.07in} \setlength{\rightmargin}{0.08in}
  \setlength{\listparindent}{0.0in} \setlength{\labelwidth}{0.08in}
  \setlength{\labelsep}{0.1in} \setlength{\itemindent}{0.05in}}}
 {\end{list}}

\newcommand{\bcar}{\begin{carlist}}
\newcommand{\ecar}{\end{carlist}}
\newcounter{qcounter}
\newenvironment{enums}
 {\begin{list}{\arabic{qcounter}.}
 {\usecounter{qcounter} \setlength{\topsep}{0in} \setlength{\partopsep}{0in}
  \setlength{\parsep}{0in} \setlength{\itemsep}{\parskip}
  \setlength{\leftmargin}{0.21in} \setlength{\rightmargin}{0in}
  \setlength{\listparindent}{0.0in} \setlength{\labelwidth}{0.07in}
  \setlength{\labelsep}{0.1in} \setlength{\itemindent}{-0.04in}}}
 {\end{list}}

\newcommand{\benums}{\begin{enums}}
\newcommand{\eenums}{\end{enums}}

\newcommand{\figsqueeze}{}
\newcommand{\asqueeze}{\hspace{-0.5pt}}
\newcommand{\unif}{\nu}
\newcommand{\policy}{\pi}                                    
\newcommand{\loggingpolicy}{\mu}                             
\newcommand{\targetpolicy}{\pi}                              

\newcommand{\vs}{\mathbf{s}}                                 
\newcommand{\s}{s}                                           
\newcommand{\slatesize}{\ell}                                
\newcommand{\action}{a}                                      
\newcommand{\poolsize}{m}                                    

\newcommand{\assumes}{Assumption~\ref{ass:ada}\xspace}                                   

\newcommand{\hV}{\hat{V}}
\newcommand{\hVdm}{\hV_\textup{DM}}
\newcommand{\hVips}{\hV_\textup{IPS}}
\newcommand{\hVwips}{\hV_\textup{wIPS}}
\newcommand{\hVpi}{\hV_\textup{PI}}
\newcommand{\hVwpi}{\hV_\textup{wPI}}

\newcommand{\hpi}{\hat{\pi}}
\newcommand{\hr}{\hat{r}}
\newcommand{\approach}{pseudoinverse estimator\xspace}       

\newcommand{\decompvec}{\bm{\phi}}                           
\newcommand{\vphix}{\bm{\phi}_x}

\newcommand{\vpr}{\bm{\theta}}
\newcommand{\prvec}{\vpr}                                    
\newcommand{\vprx}{\vpr_{\mu,x}}
\newcommand{\prvecx}{\vpr_{\mu,x}}                                 

\newcommand{\pr}{\theta}

\newcommand{\hpri}{\hat{\theta}_i}
\newcommand{\hvpr}{\boldsymbol{\hat\theta}}
\newcommand{\hvpri}{\hvpr_i}

\newcommand{\mgamma}{\Gammab}
\newcommand{\mgammaparam}[2]{\Gammab_{\!#1,#2}}
\newcommand{\mgammalogger}{\mgammaparam{\loggingpolicy}{x}}

\newcommand{\q}[2]{\mathbf{q}_{#1,#2}}
\newcommand{\qtarget}{\q{\targetpolicy}{x}}
\newcommand{\vq}{\mathbf{q}}
\newcommand{\mlambda}{\mathbf{\Lambda}}
\newcommand{\vbeta}{\boldsymbol{\beta}}

\newcommand{\rel}{\textit{rel}}

\newcommand{\rank}{\textit{rank}}
\newcommand{\partBody}{\textup{body}}
\newcommand{\partTitle}{\textup{title}}

\newcommand{\treeTitle}{\textit{tree}_{\partTitle}}
\newcommand{\lassoTitle}{\textit{lasso}_{\partTitle}}
\newcommand{\treeBody}{\textit{tree}_{\partBody}}
\newcommand{\lassoBody}{\textit{lasso}_{\partBody}}

\newcommand{\eps}{\varepsilon}
\newcommand{\slateind}{\one_\vs}                             
\DeclareMathOperator{\supp}{supp}
\DeclareMathOperator{\diag}{diag}
\DeclareMathOperator{\Null}{Null}
\DeclareMathOperator{\Range}{Range}
\newcommand{\mP}{\mathbf{P}}
\newcommand{\mQ}{\mathbf{Q}}
\newcommand{\mM}{\mathbf{M}}
\newcommand{\mPi}{\mathbf{\Pi}}
\newcommand{\mI}{\mathbf{I}}
\newcommand{\mA}{\mathbf{A}}
\newcommand{\mB}{\mathbf{B}}
\newcommand{\mV}{\mathbf{V}}
\newcommand{\mU}{\mathbf{U}}

\newcommand{\vu}{\mathbf{u}}
\newcommand{\vw}{\mathbf{w}}
\newcommand{\vn}{\mathbf{n}}
\newcommand{\vv}{\mathbf{v}}
\newcommand{\vz}{\mathbf{z}}
\newcommand{\vy}{\mathbf{y}}
\newcommand{\vf}{\mathbf{f}}

\newcommand{\vfBody}{\vf_\partBody}
\newcommand{\vfTitle}{\vf_\partTitle}

\newcommand{\slots}{{\!\mathcal{J}}}
\newcommand{\acts}{{\!\mathcal{A}}}
\newcommand{\mIslots}{\mI_\slots^{\vphantom{T}}}
\newcommand{\oneslots}{\one_\slots^{\vphantom{T}}}
\newcommand{\oneslotsT}{\one_\slots^T}
\newcommand{\mIacts}{\mI_\acts^{\vphantom{T}}}
\newcommand{\oneacts}{\one_\acts^{\vphantom{T}}}
\newcommand{\oneactsT}{\one_\acts^T}
\newcommand{\vphi}{\bm{\phi}}
\newcommand{\brho}{\bar{\rho}}

\newcommand{\card}[1]{\lvert#1\rvert}
\newcommand{\set}[1]{\{#1\}}
\newcommand{\bigSet}[1]{\bigl\{#1\bigr\}}

\newcommand{\braces}[1]{\{#1\}}

\newcommand{\bracks}[1]{[#1]}
\newcommand{\bigBracks}[1]{\bigl[#1\bigr]}
\newcommand{\BigBracks}[1]{\Bigl[#1\Bigr]}
\newcommand{\Bracks}[1]{\left[#1\right]}
\newcommand{\Parens}[1]{\left(#1\right)}
\newcommand{\bigParens}[1]{\bigl(#1\bigr)}

\newcommand{\given}{\mathbin{\vert}}

\newcommand{\norm}[1]{\lVert#1\rVert}

\newcommand{\abs}[1]{\lvert#1\rvert}
\newcommand{\Abs}[1]{\left\lvert#1\right\rvert}

\newcommand{\Sec}[1]{Sec.~\ref{sec:#1}}
\newcommand{\Thm}[1]{Theorem~\ref{thm:#1}}
\newcommand{\Lem}[1]{Lemma~\ref{lem:#1}}
\newcommand{\Eq}[1]{Eq.~\eqref{eq:#1}}
\newcommand{\Eqs}[2]{Eqs.~\eqref{eq:#1} and~\eqref{eq:#2}}
\newcommand{\App}[1]{Appendix~\ref{app:#1}}
\newcommand{\Ex}[1]{Example~\ref{ex:#1}}
\newcommand{\Claim}[1]{Claim~\ref{claim:#1}}
\newcommand{\Fig}[1]{Fig.~\ref{fig:#1}}
\newcommand{\Prop}[1]{Prop.~\ref{prop:#1}}

\newcommand{\Direct}{DM\xspace}
\newcommand{\OnPolicy}{\textsc{OnPolicy}\xspace}
\newcommand{\UtilityRate}{\textsc{UtilityRate}\xspace}

\title{Off-policy evaluation for slate recommendation}

\author{%
Adith Swaminathan\\
Microsoft Research, Redmond\\
\texttt{adswamin@microsoft.com}
\And
Akshay Krishnamurthy\\
University of Massachusetts, Amherst\\
\texttt{akshay@cs.umass.edu}%
\And
Alekh Agarwal\\
\!\!\!Microsoft Research, New\,York\asqueeze\\
\texttt{alekha@microsoft.com}%
\And
Miroslav Dud\'{i}k\\
\asqueeze Microsoft Research, New\,York\asqueeze\\
\texttt{mdudik@microsoft.com}%
\And
John Langford\\
\asqueeze Microsoft Research, New\,York\!\!\!\\
\texttt{jcl@microsoft.com}%
\And
Damien Jose\\
Microsoft, Redmond\\
\texttt{dajose@microsoft.com}%
\And
Imed Zitouni\\
Microsoft, Redmond\\
\texttt{izitouni@microsoft.com}%
}

\begin{document}

\maketitle

\begin{abstract}
This paper studies the evaluation of policies that recommend an
ordered set of items (e.g., a ranking) based on some context---a
common scenario in web search, ads, and recommendation. We build on
techniques from combinatorial bandits to introduce a new practical
estimator that uses logged data to estimate a policy's performance.  A thorough empirical evaluation on real-world data reveals
that our estimator is accurate in a variety of settings, including
as a subroutine in a learning-to-rank task, where it achieves
competitive performance.
We derive conditions under which our estimator is unbiased---these conditions
are weaker than prior heuristics for slate evaluation---and
experimentally demonstrate a smaller bias than parametric approaches,
even when these conditions are violated. Finally, our theory and
experiments also show exponential savings in the amount of required
data compared with general unbiased estimators.
\end{abstract}

\section{Introduction}
\label{sec:intro}

In recommendation systems for e-commerce, search, or news, we would
like to use the data collected during operation to test new
content-serving algorithms (called \emph{policies}) along metrics such
as revenue and number of
clicks~\cite{Bottou13Counterfactual,li2010contextual}.  This task is
called \emph{off-policy evaluation}. General approaches, namely
\emph{inverse propensity scores}
(IPS)~\cite{dudik2014doubly,HorvitzTh52}, require unrealistically
large amounts of logged data to evaluate whole-page metrics that
depend on multiple recommended items, which happens when showing ranked
lists.  The key challenge is that the number of possible lists (called
\emph{slates}) is combinatorially large. As a result, the policy being
evaluated is likely to choose different slates from those recorded in
the logs most of the time, unless it is very similar to the
data-collection policy. This challenge is fundamental~\citep{switch},
so any off-policy evaluation method that works with large slates needs
to make some structural assumptions about the whole-page metric or the
user behavior.

Previous work on off-policy
evaluation and whole-page optimization improves the probability of match
between logging and evaluation by restricting attention to small
slate spaces~\citep{Wang2016,Li11}, introducing assumptions that allow for
partial matches between the proposed and observed
slates~\citep{LihongTopIPS}, or assuming that the policies used for
logging and evaluation are
similar~\cite{Bottou13Counterfactual,swaminathan15}. Another line of work
constructs parametric models of slate
quality~\cite{chapelle2009dynamic, Guo2009, Dupret2008} (see also
Sec. 4.3 of~\cite{hofmann2016online}). While these approaches require
less data, they can have large bias, and their use in practice
requires an expensive trial-and-error cycle involving weeks-long A/B
tests to develop new policies~\cite{kohavi2009controlled}. In this
paper we design a method more robust to problems with bias and with only modest data
requirements, with the goal of substantially shortening this cycle and
accelerating the policy development process.

We frame the slate recommendation problem as a combinatorial
generalization of \emph{contextual
  bandits}~\citep{EXP4,langford2008epoch,dudik2014doubly}. In
combinatorial contextual bandits, for each \emph{context}, a policy
selects a \emph{slate} consisting of component \emph{actions}, after
which a \emph{reward} for the entire slate is observed. In web search,
the context is the search query augmented with a user profile, the
slate is the search results page consisting of a list of retrieved
documents (actions), and example reward metrics are page-level
measures such as time-to-success, NDCG (position-weighted relevance),
or other measures of user satisfaction. As input we receive contextual
bandit data obtained by some \emph{logging policy}, and our goal
is to estimate the reward of a new \emph{target policy}. This
off-policy setup differs from online learning in contextual bandits,
where the goal is to
adaptively maximize the reward in the presence of an explore-exploit trade-off~\cite{bubeck2012regret}.

Inspired by work in \emph{combinatorial} and \emph{linear
  bandits}~\citep{cesa2012combinatorial,rusmevichientong2010linearly,DaniHaKa08},
we propose an estimator that makes only a weak assumption about the
evaluated metric, while exponentially reducing the data requirements
in comparison with IPS. Specifically, we posit a \emph{linearity
  assumption}, stating that the slate-level reward (e.g., time to
success in web search) decomposes additively across actions, but the
action-level rewards are not observed. Crucially, the action-level
rewards are allowed to depend on the context, and we do not require
that they be easily modeled from the features describing the context.
In fact, our method is completely agnostic to the representation of
contexts.

\begin{figure*}
\begin{center}
\includegraphics[width=0.8\textwidth]{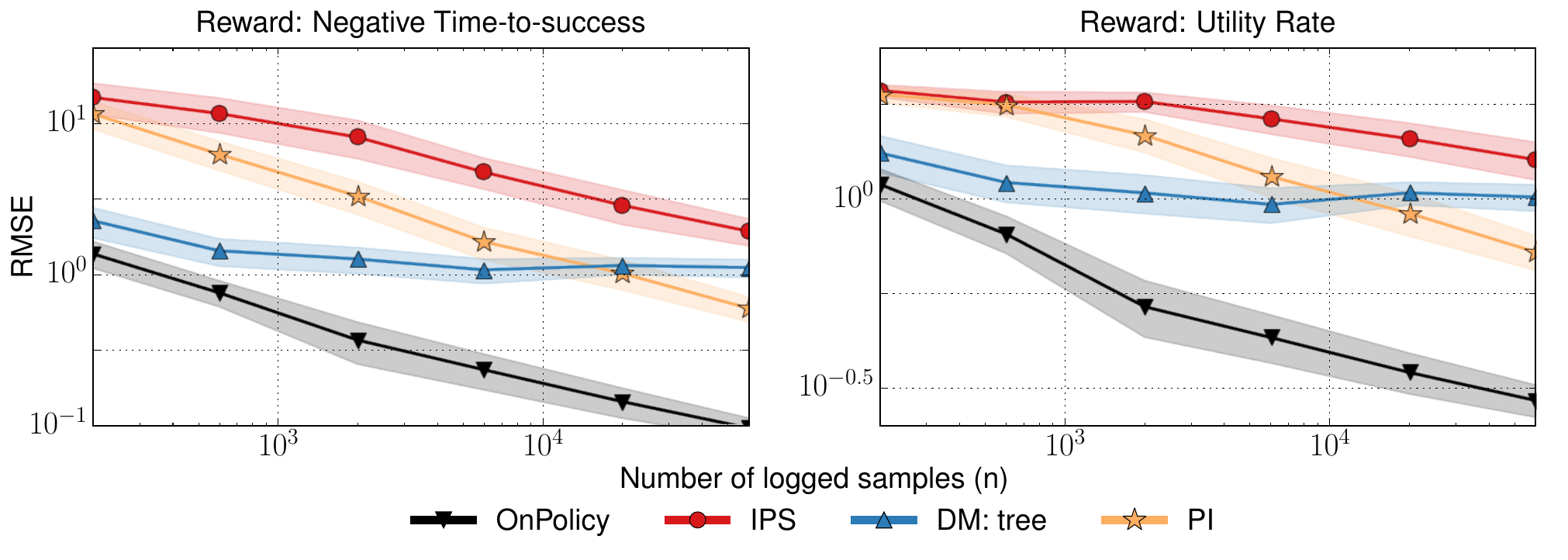}%
\end{center}
\figsqueeze%
\caption{Off-policy evaluation of two whole-page user-satisfaction
  metrics on proprietary search engine data. Average RMSE of different
  estimators over 50 runs on a log-log scale.  Our method (PI)
  achieves the best performance with moderate data sizes. The unbiased
  IPS method suffers high variance, and direct modeling (\Direct) of
  the metrics suffers high bias.  \OnPolicy is the expensive choice of
  deploying the policy, for instance, in an A/B test.}
\label{fig:bing_experiment:optimized}
\figsqueeze%
\end{figure*}

We make the following contributions:

\benums
\item The \emph{\approach} (PI) for off-policy evaluation: a
  general-purpose estimator from the combinatorial bandit literature,
  adapted for off-policy evaluation.
  When ranking $\ell$ out of $m$ items under the linearity assumption,
  PI typically requires $\mathcal{O}(\ell m/\eps^2)$ samples to achieve error at most $\eps$---an exponential gain over the $m^{\Omega(\ell)}$
  sample complexity of IPS. We provide distribution-dependent bounds based on
  the overlap between logging and target policies.

\item Experiments on real-world search ranking datasets: The strong
  performance of the PI estimator provides, to our knowledge, the
  first demonstration of high-quality off-policy evaluation of
  whole-page metrics,
  comprehensively outperforming prior baselines (see \Fig{bing_experiment:optimized}).

\item Off-policy optimization: We provide a simple procedure for
  learning to rank (L2R) using the PI estimator to impute action-level
  rewards for each context. This allows direct optimization of
  whole-page metrics via pointwise L2R approaches, \emph{without requiring
  pointwise feedback}.
\eenums

\paragraph{Related work} Large state spaces have typically been
studied in the online, or on-policy, setting. Some works assume
specific parametric (e.g., linear) models relating the metrics to the
features describing a
slate~\cite{Auer02,rusmevichientong2010linearly,filippi2010parametric,chu2011contextual,qin2014contextual};
this can lead to bias if the model is inaccurate (e.g., we might not
have access to sufficiently predictive features). Others posit the
same linearity assumption as we do, but further assume a \emph{semi-bandit}
feedback model where the rewards of all actions on the slate are revealed
\cite{kale2010non,kveton2015tight,krishnamurthy2015efficient}.
While much of the research focuses on on-policy setting,
the off-policy paradigm studied in this paper is often preferred in practice since it might
not be possible to implement low-latency updates needed for online
learning, or we might be interested in many different metrics and
require a manual review of their trade-offs before deploying new
policies.

At a technical level, the PI estimator has been used in online
learning~\citep{cesa2012combinatorial,rusmevichientong2010linearly,DaniHaKa08},
but the analysis there is tailored to the specific data collection
policies used by the learner.  In contrast, we provide
distribution-dependent bounds without any assumptions on the logging
or target policy.

\section{Setting and notation}
\label{sec:setting}

In combinatorial contextual bandits, a decision maker
repeatedly interacts as follows:
\benums
\item the decision maker observes a \emph{context} $x$ drawn from a
  distribution $D(x)$ over some space $X$;

\item based on the context, the decision maker chooses a \emph{slate}
  $\vs=(\s_1,\dotsc,\s_\ell)$ consisting of \emph{actions} $\s_j$,
  where a position $j$ is called a
  \emph{slot}, the number of slots is $\ell$, actions at position $j$
  come from some space $A_j(x)$, and the slate $\vs$ is chosen from a set
  of allowed slates $S(x)\subseteq A_1(x)\times \dotsm
  \times A_\ell(x)$;

\item given the context and slate, a reward
  $r\in[-1,1]$ is drawn from a distribution $D(r\given x,\vs)$;
  rewards in different rounds
  are independent, conditioned on contexts and slates.
  \eenums
The context space $X$ can be infinite, but the set of actions is
finite. We assume $\card{A_j(x)}=m_j$ for all contexts $x \in X$ and
define $m\coloneqq\max_j m_j$ as the maximum number of actions per
slot. The goal of the decision maker is to \emph{maximize the
  reward}. The decision maker is modeled as a \emph{stochastic policy}
$\pi$ that specifies a conditional distribution $\pi(\vs\given x)$ (a
deterministic policy is a special case).  The \emph{value} of a policy
$\pi$, denoted $V(\pi)$, is defined as the expected reward when
following $\pi$:
\begin{equation}
\label{eq:policy_value}
\textstyle
   V(\pi)\coloneqq
  \EE_{x\sim D}\EE_{\vs\sim\pi(\cdot\given x)}\EE_{r\sim D(\cdot\given x,\vs)}
  \bigBracks{r}
\enspace.
\end{equation}

To simplify derivations, we extend the conditional distribution $\pi$
into a distribution over triples $(x,\vs,r)$ as $\pi(x,\vs,r)\coloneqq
D(r\given x,\vs)\pi(\vs\given x)D(x)$.  With this shorthand, we have
$V(\pi) = \EE_\pi[r]$.

To finish this section, we introduce notation for the expected reward
for a given context and slate, which we call the \emph{slate value},
and denote as:
\begin{equation}
\label{eq:slate_value}
\textstyle
   V(x,\vs)\coloneqq\EE_{r\sim D(\cdot\given x,\vs)}[r]
\enspace.
\end{equation}

\begin{example}[Cartesian product]
\label{ex:cartesian}
Consider the optimization of a news portal where the reward is the
whole-page advertising revenue. Context $x$ is the user profile, slate
is the news-portal page with slots corresponding to news
sections,\footnote{For simplicity, we do not discuss the more general
  setting of showing multiple articles in each news section.}  and
actions are the articles.
The set of valid slates
is the Cartesian product $S(x)=\prod_{j\le\ell} A_j(x)$. The number of
valid slates is exponential in $\ell$:
$\card{S(x)}=\prod_{j\le\ell} m_j$.
\end{example}

\begin{example}[Ranking]
\label{ex:ranking}
Consider web search and ranking. Context $x$ is the query along with
user profile. Actions correspond to search items (such as
webpages). The policy chooses $\ell$ of $m$ items, where the set
$A(x)$ of $m$ items for a context $x$ is chosen from a corpus by a
filtering step (e.g., a database query). We have $A_j(x) = A(x)$ for
all $j \le \ell$, but the allowed slates $S(x)$ have no
repetitions. The number of valid slates is exponential in $\ell$:
$\card{S(x)}=m!/(m-\ell)!=m^{\Omega(\ell)}$. A reward could be the
\emph{negative time-to-success}, i.e., negative of the time taken by
the user to find a relevant item.
\end{example}

\subsection{Off-policy evaluation and optimization}

In the \emph{off-policy} setting, we have access to the \emph{logged
  data} $(x_1,\vs_1,r_1),\dotsc,(x_n,\vs_n,r_n)$ collected using a
past policy $\mu$, called the \emph{logging policy}.
\emph{Off-policy evaluation} is the task of estimating the value of a
new policy $\pi$, called the \emph{target policy}, using the logged
data.  \emph{Off-policy optimization} is the harder task of finding a
policy $\hpi$ that achieves maximal reward. 

There are two standard approaches for off-policy evaluation. The
\emph{direct method} (DM) uses the logged data to train a (parametric)
model $\hr(x,\vs)$ for predicting the expected reward for a given context
and slate. $V(\pi)$ is then estimated as
\begin{equation}
\label{eq:dm}
\textstyle
   \hVdm(\pi)=\frac{1}{n}\sum_{i=1}^n\sum_{\vs\in S(x)} \hr(x_i,\vs)\pi(\vs\given x_i)
\enspace.
\end{equation}
The direct method is often biased due to mismatch between model assumptions and ground truth.

The second approach, which is provably unbiased (under modest
assumptions), is the \emph{inverse propensity score} (IPS)
estimator~\cite{HorvitzTh52}. The IPS estimator re-weights the logged
data according to ratios of slate probabilities under the target and
logging policy. It has two common variants:
\begin{equation}
\label{eq:ips}
\textstyle
  \hVips(\pi)=\frac{1}{n}\sum_{i=1}^n r_i\cdot\frac{\pi(\vs_i\given x_i)}{\mu(\vs_i\given x_i)}
\enspace,
\quad
  \hVwips(\pi)=
  \sum_{i=1}^n r_i\cdot\frac{\pi(\vs_i\given x_i)}{\mu(\vs_i\given x_i)}
  \bigm/
  \bigParens{\sum_{i=1}^n
                \frac{\pi(\vs_i\given x_i)}{\mu(\vs_i\given x_i)}}
\enspace.
\end{equation}
wIPS generally has better variance with an asymptotically zero bias.
The variance of both estimators grows linearly with
$\tfrac{\pi(\vs\given x)}{\mu(\vs\given x)}$, which can be
$\Omega(\card{S(x)})$. This is prohibitive when
$\card{S(x)}=m^{\Omega(\ell)}$.

\section{Our approach}
\label{sec:LADA}

The IPS estimator is minimax optimal~\citep{switch},
so its exponential variance
is unavoidable in the worst case. We circumvent this
hardness by positing an assumption on the structure of rewards. Specifically,
we assume that the slate-level reward is a sum of
unobserved action-level rewards that depend on the context, the action, and the
position on the slate, but not on the other actions on the
slate.

Formally, we consider \emph{slate indicator vectors} in
$\RR^{\ell m}$ whose components are indexed by pairs $(j,a)$ of slots
and possible actions in them. A slate is described by an
indicator vector $\one_\vs\in\RR^{\ell m}$ whose entry at position
$(j,a)$ is equal to 1 if the slate $\vs$ has action $a$ in slot
$j$, i.e., if $s_j=a$. The above assumption is formalized as follows:

\begin{assumption}[Linearity Assumption]
\label{ass:ada}
For each context $x\in X$ there exists an (unknown) \emph{intrinsic
  reward} vector $\vphi_x\in\RR^{\ell m}$ such that the slate value
satisfies
$V(x,\vs)
= \one_\vs^T\vphi_x=\sum_{j=1}^\ell
\phi_x(j,\s_j)$.
\end{assumption}

The slate indicator vector can be viewed as a feature vector,
representing the slate, and $\vphi_x$ can be viewed as a
\emph{context-specific} weight vector. The assumption refers to the
fact that the value of a slate is a linear function of its feature
representation. However, note that this linear dependence is allowed
to be completely different across contexts, because we make no
assumptions on how $\vphi_x$ depends on $x$, and in fact our method
does not even attempt to accurately estimate $\vphi_x$. Being agnostic
to the form of $\vphi_x$ is the key departure from the direct method
and parametric bandits.

While Assumption~\ref{ass:ada} rules out interactions among different
actions on a slate,\footnote{We discuss limitations of \assumes and
  directions to overcome them in \Sec{discussion}.}  its ability to
vary intrinsic rewards arbitrarily across contexts captures many
common metrics in information retrieval, such as the \emph{normalized
  discounted cumulative gain} (NDCG)~\cite{ndcg2005}, a common metric
in web ranking:

\begin{example}[NDCG]
  For a slate $\vs$, we first define
  $
  \textstyle
  \textup{DCG}(x,\vs) \coloneqq
  \sum_{j=1}^\ell \frac{2^{\rel(x,s_j)} - 1}{\log_2(j+1)}$
  where $\rel(x,a)$ is the relevance of document $a$ on query
  $x$. Then $\textup{NDCG}(x,\vs)\coloneqq\textup{DCG}(x,\vs) /
  \textup{DCG}^\star(x)$ where $\textup{DCG}^\star(x)=\max_{\vs\in
    S(x)}\textup{DCG}(x,\vs)$, so NDCG takes values in $[0,1]$. Thus,
  NDCG satisfies \assumes with $\phi_x(j,a) = \Parens{2^{\rel(x,a)} -
    1}\bigm/\log_2(j+1)\textup{DCG}^\star(x)$.
  \label{ex:ndcg}
\end{example}

In addition to \assumes, we also make the standard assumption that the
logging policy puts non-zero probability on all slates that can be
potentially chosen by the target policy. This assumption is also
required for 
IPS, otherwise unbiased off-policy evaluation
is impossible~\citep{LangfordStWo08}.
\begin{assumption}[Absolute Continuity]
\label{ass:ABS}
The off-policy evaluation problem satisfies the \emph{absolute continuity}
assumption if $\mu(\vs\given x)>0$ whenever
$\pi(\vs\given x)>0$ with probability one over $x\sim D$.
\end{assumption}

\subsection{The \approach}
\label{sec:estimator}

Using Assumption~\ref{ass:ada}, we can now apply the techniques
from the combinatorial bandit literature to our problem. In
particular, our estimator closely follows the recipe
of~\citet{cesa2012combinatorial}, albeit with some differences to
account for the off-policy and contextual nature of our setup. Under
Assumption~\ref{ass:ada}, we can view the recovery of $\vphix$ for a
given context $x$ as a linear regression problem.  The covariates
$\one_\vs$ are drawn according to $\mu(\cdot\given x)$, and the reward
follows a linear model, conditional on $\vs$ and $x$, with $\vphix$ as
the ``weight vector''. Thus, we can write the MSE of an estimate
$\vw$ as $\EE_{\vs \sim \mu(\cdot\given x)} \EE_{r\sim D(\cdot\given
  \vs,x)} \bracks{(\one_\vs^T\vw - r)^2} $, or more compactly as
$\EE_{\mu}\bracks{(\one_\vs^T\vw - r)^2\given x}$, using our
definition of $\mu$ as a distribution over triples $(x,\vs,r)$.  We
estimate $\vphix$ by the MSE minimizer with the smallest norm,
which can be written in closed form as
\begin{equation}
\label{eq:phi-ls}
\bar{\vphi}_x
= \Parens{\EE_\mu[\one_\vs\one_\vs^T\given x]}^{\!\dagger}
\EE_\mu[r\one_\vs\given x]
\enspace,
\end{equation}
where $\mM^\dagger$ is the Moore-Penrose pseudoinverse of a matrix $\mM$.
Note that this idealized ``estimator'' $\bar{\vphi}_x$ uses
conditional expectations over $\vs \sim \mu(\cdot\given x)$ and $r\sim
D(\cdot\given\vs,x)$. To simplify the notation, we write
$\mgammalogger\coloneqq\EE_\mu[\one_\vs\one_\vs^T\given x]\in\RR^{\ell
  m\times\ell m}$ to denote the (uncentered) covariance matrix for our
regression problem, appearing on the right-hand side of \Eq{phi-ls}.
We also introduce notation for the second term in \Eq{phi-ls} and
its empirical estimate:
$\vprx \coloneqq \EE_\mu[r\one_\vs\given x]$, and
$\hvpri\coloneqq r_i\one_{\vs_i}$.

Thus, our regression estimator~\eqref{eq:phi-ls} is simply
$\bar{\vphi}_x=\mgammaparam{\loggingpolicy}{x}^\dagger \vprx$. Under
Assumptions~\ref{ass:ada} and~\ref{ass:ABS}, it is easy to show that
$V(x,\vs)=\one_\vs^T\bar{\vphi}_x = \slateind^T\mgammalogger^\dagger\vprx$. Replacing $\vprx$
with $\hvpri$
motivates
the following estimator for $V(\pi)$, which we call the
\emph{pseudoinverse estimator} or PI:
\begin{align}
\label{eq:hVpi}
  \hVpi(\pi)
\;\;=\;\;
 \frac{1}{n} \sum_{i=1}^n \sum_{\vs\in S} \pi(\vs\given x_i)\slateind^T\mgammaparam{\mu}{x_i}^\dagger\hvpri
\;\;=\;\;
  \frac{1}{n} \sum_{i=1}^n
  r_i\cdot\q{\pi}{x_i}^T\mgammaparam{\mu}{x_i}^\dagger\one_{\vs_i}
\enspace.
\end{align}
In \Eq{hVpi} we have expanded the definition of $\hvpri$ and
introduced the notation $\q{\pi}{x}$ for the expected slate indicator
under $\pi$ conditional on $x$,
$\q{\pi}{x}\coloneqq\EE_\pi[\one_\vs\given x]$.
The summation over $\vs$ required to obtain $\q{\pi}{x_i}$ in
\Eq{hVpi} can be replaced by a small sample. We can also derive a
weighted variant of PI:
\begin{align}
\label{eq:hVwpi}
\hVwpi(\pi)
&
  =
  \frac
  {\sum_{i=1}^n r_i\cdot\q{\pi}{x_i}^T\mgammaparam{\mu}{x_i}^\dagger\one_{\vs_i}}
  {\sum_{i=1}^n \q{\pi}{x_i}^T\mgammaparam{\mu}{x_i}^\dagger\one_{\vs_i}}
\enspace.
\end{align}

We prove the following unbiasedness property in
Appendix~\ref{app:proofs:estimator}.
\begin{proposition}
\label{prop:unbiased}
If Assumptions~\ref{ass:ada} and~\ref{ass:ABS} hold, then the
estimator $\hVpi$ is unbiased, i.e.,
  $\EE_{\mu^n}[\hVpi]=V(\pi)$,
where the expectation is over the $n$ logged examples sampled
i.i.d.\ from $\mu$.
\end{proposition}

As special cases, PI reduces to IPS when $\ell=1$, and simplifies to
$\sum_{i=1}^n r_i/n$ when $\pi=\mu$ (see \App{pi:eq:mu}). To build
further intuition, we consider the settings of
Examples~\ref{ex:cartesian} and~\ref{ex:ranking},
and simplify the PI estimator to highlight the improvement
over IPS.

\begin{example}[PI for a Cartesian product when $\mu$ is a product distribution]
The PI estimator for the Cartesian product slate space, when $\mu$
factorizes across slots as $\mu(\vs\given x)=\prod_j\mu(s_j\given x)$,
 simplifies to
\[
\textstyle
   \hVpi(\policy)=\frac{1}{n}\sum_{i=1}^n
  r_i\cdot\Parens{
  \sum_{j=1}^\ell \frac{\pi(s_{ij}\given x_i)}{\mu(s_{ij}\given x_i)}
  -\ell+1}
\enspace,
\]
by Prop.~\ref{prop:product} in Appendix~\ref{app:product_bounds}. Note that unlike IPS, which divides by probabilities of whole slates, the PI estimator only divides by probabilities of actions appearing in individual slots. Thus, the magnitude
of each term of the outer summation is only $\Ocal(\ell m)$, whereas
the IPS terms are $m^{\Omega(\ell)}$.
\end{example}

\begin{example}[PI for rankings with $\ell=m$ and uniform logging]
In this case,
\[
\textstyle
\hVpi(\policy)=
\frac{1}{n}\sum_{i=1}^n
  r_i\cdot\Parens{
  \sum_{j=1}^\ell \frac{\pi(s_{ij}\given x_i)}{1/(m-1)}-m+2}
\enspace,
\]
by Prop.~\ref{prop:ranking:spectrum} in \App{bounds:uniform}. The
summands are again $\Ocal(\ell m)=\Ocal(m^2)$.
\end{example}

\subsection{Deviation analysis}
\label{sec:deviation}

So far, we have shown that PI is unbiased under our assumptions
and overcomes the deficiencies of IPS in specific examples. We
now derive a finite-sample error bound,
based on the overlap between $\targetpolicy$ and $\loggingpolicy$.
We use Bernstein's
inequality,
for which we define the variance and range terms:
\begin{align}
   \sigma^2
   \coloneqq
   \EE_{x\sim D}\Bracks{\q{\pi}{x}^T \mgammalogger^\dagger\q{\pi}{x}^{\vphantom{T}}}
\;,\quad
\textstyle
   \rho
   \coloneqq
   \adjustlimits\sup_x\sup_{\vs:\mu(\vs\given x)>0}
   \Abs{\q{\pi}{x}^T \mgammalogger^\dagger\one_\vs}
\enspace.
\label{eq:sigma2:def}
\end{align}
The quantity $\sigma^2$ bounds the
variance whereas $\rho$ bounds the range.  They capture the
``average'' and ``worst-case'' mismatch between $\loggingpolicy$ and
$\targetpolicy$. They equal one when $\pi=\mu$ (see \App{pi:eq:mu}),
and yield the following deviation bound:
\begin{theorem}
\label{thm:deviation:gen}
Under Assumptions~\ref{ass:ada} and~\ref{ass:ABS}, let
$\sigma^2$ and $\rho$ be defined as in \Eq{sigma2:def}.  Then, for any
$\delta \in (0,1)$, with probability at least $1-\delta$,
\begin{align*}
\Abs{\hVpi(\pi) - V(\pi)} \le
  \sqrt{ \frac{2\sigma^2\ln(2/\delta)}{n}}
         + \frac{2 (\rho + 1) \ln(2/\delta)}{3n}
\enspace.
\end{align*}
\end{theorem}
We observe that this finite sample bound is structurally different
from the regret bounds studied in the prior works on combinatorial
bandits. The bound incorporates the extent of overlap between $\pi$
and $\mu$ so that we have higher confidence in our estimates when
the logging and evaluation policies are similar---an important
consideration in off-policy evaluation.

While the bound might look complicated, it simplifies if we
consider the class of $\eps$-uniform logging policies. Formally,
for any policy $\mu$, define $\mu_\eps(\vs\given x) =
(1-\eps)\mu(\vs\given x) + \eps \nu (\vs\given x)$, with $\nu$ being
the uniform distribution over the set $S(x)$. For suitably small
$\eps$, such logging policies are widely used in practice. We have
the following corollary for these policies, proved in \App{bounds}:
\begin{corollary}
\label{prop:kappa}
In the settings of
\Ex{cartesian} or
\Ex{ranking},
if the logging is done with $\mu_\eps$ for some $\eps >
0$, we have $\abs{\hVpi(\pi) - V(\pi)} \le
\Ocal\bigParens{\sqrt{\eps^{-1}\ell m/n}}$.
\end{corollary}
Again, this turns the $\Omega(m^\ell)$ data dependence of IPS into
$O(m\ell)$. The key step in the proof is the bound on a certain norm of $\mgamma_{\unif}^\dagger$,
similar to the bounds of~\citet{cesa2012combinatorial}, but our results are
a bit sharper.

\section{Experiments}
\label{sec:experiments}

We empirically evaluate the performance of the \approach
for ranking problems.
We first show that PI
outperforms prior works in a comprehensive semi-synthetic study
using a public dataset. We then use our estimator for \emph{off-policy
  optimization}, i.e., to learn ranking policies, competitively
with supervised learning that uses more information.
Finally, we
demonstrate substantial improvements on proprietary data from search
engine logs for two user-satisfaction metrics used in practice:
\emph{time-to-success} and \emph{utility rate}, which
do not
satisfy the linearity assumption.
More detailed results are deferred to Appendices~\ref{app:approach_family}
and~\ref{sec:opt_family}. All of our code is available
online.\footnote{\href{https://github.com/adith387/slates_semisynth_expts}{https://github.com/adith387/slates\_semisynth\_expts}}

\subsection{Semi-synthetic evaluation}
\label{sec:expt_ssynth}


Our semi-synthetic evaluation uses labeled data from the Microsoft
Learning to Rank Challenge dataset~\cite{letor2013} (MSLR-WEB30K) to
create a contextual bandit instance. Queries form the contexts $x$ and
actions $a$ are the available documents. The dataset contains over 31K
queries, each with up to 1251 judged documents,
where the query-document pairs are judged on
a 5-point scale, $\rel(x,a)\in\set{0,\ldots,4}$.
Each pair $(x,a)$ has a
feature vector $\vf(x,a)$, which can be partitioned into title and body features
($\vfTitle$ and $\vfBody$).  We consider two
slate rewards: \textup{NDCG} from Example~\ref{ex:ndcg},
and the \emph{expected reciprocal rank},
\textup{ERR}~\cite{Chapelle2009}, which \emph{does not} satisfy linearity, and is defined as
\begin{align*}
  \textstyle
  \textup{ERR}(x,\vs) \coloneqq
  \sum_{r=1}^\ell \frac{1}{r} \prod_{i=1}^{r-1} (1-R(s_i))R(s_r)
   \enspace, \quad \text{where }R(a) = \frac{2^{\rel(x,a)}-1}{2^{\textit{maxrel}}}
   \text{ with $\textit{maxrel}=4$.}
\end{align*}

To derive several distinct logging and target policies, we first
train two lasso regression models, called $\lassoTitle$ and
$\lassoBody$, and two regression tree models, called $\treeTitle$ and
$\treeBody$, to predict relevances from $\vfTitle$ and $\vfBody$,
respectively.  To create the logs, queries $x$ are sampled uniformly,
and the set $A(x)$ consists of the top $m$ documents according to
$\treeTitle$. The logging policy is parametrized by a model,
either $\treeTitle$ or $\lassoTitle$, and a scalar $\alpha\ge0$.
It samples from a multinomial
distribution over documents
$p_\alpha(a|x) \propto 2^{-\alpha \lfloor \log_2
  \rank(x,a)\rfloor }$ where $\rank(x,a)$ is the rank of document $a$
for query $x$ according to the corresponding model.
Slates are constructed slot-by-slot, sampling \emph{without
  replacement} according to $p_\alpha$.
Varying $\alpha$ interpolates between uniformly random and
  deterministic logging.
Thus, all logging policies are based on the models derived from $\vfTitle$.
We consider two deterministic
target policies based on the two models derived from $\vfBody$, i.e., $\treeBody$ and $\lassoBody$,
which select the top $\ell$ documents
  according to the corresponding model. The four base models are
fairly distinct: on average fewer than 2.75 documents
overlap among top 10 (see \App{overlaps}).

\begin{figure*}
\begin{center}
\includegraphics[width=1.0\textwidth]{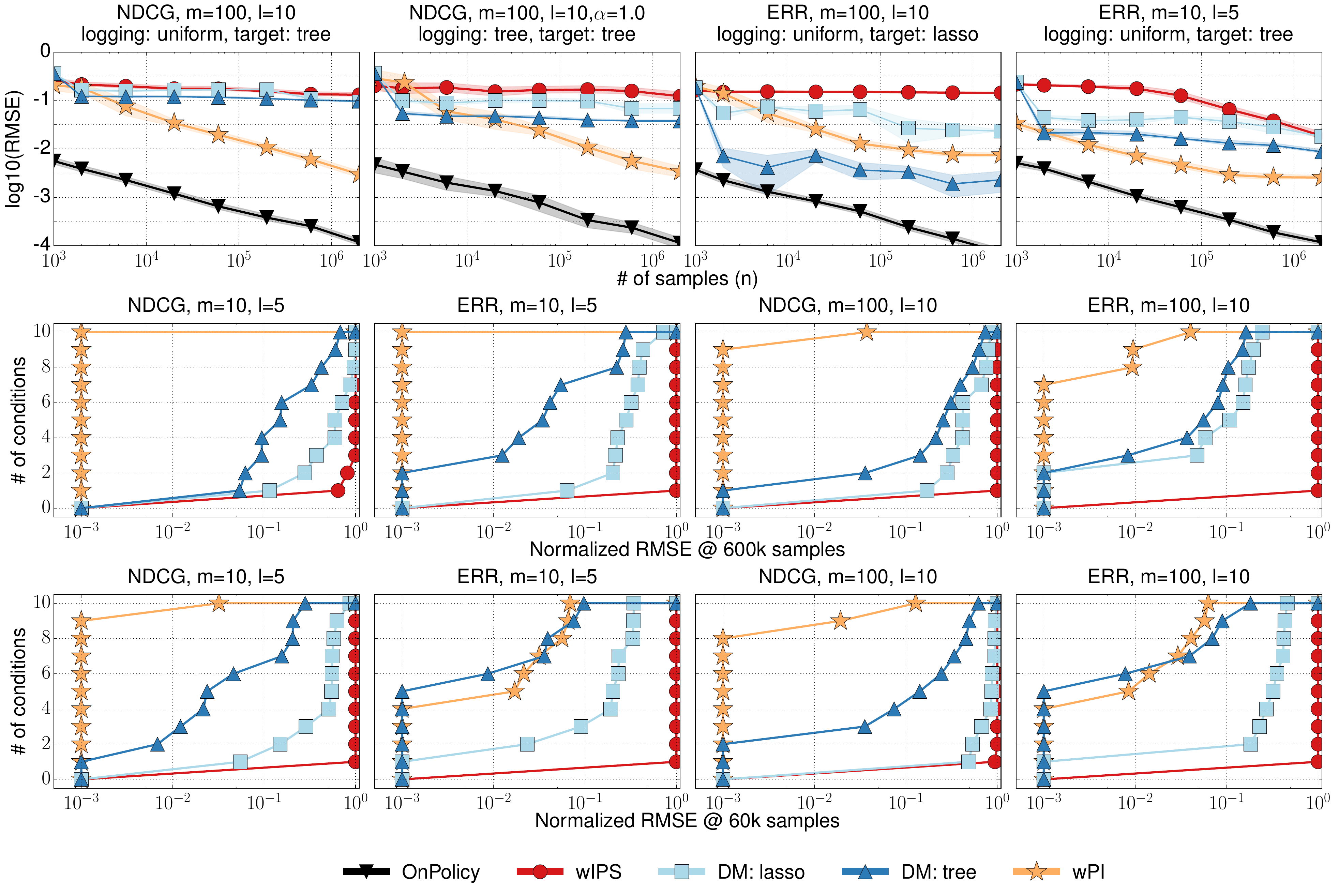}\\
\end{center}
\caption{\textit{Top:} RMSE of various estimators under four experimental
  conditions (see \App{approach_family} for all 40 conditions).
  \textit{Middle:} CDF of normalized RMSE at 600k samples;
  each plot aggregates over 10 logging-target combinations; closer to top-left is better.
  \textit{Bottom:} Same as middle but at 60k samples.}
\label{fig:synth_expt1}
\figsqueeze
\end{figure*}

We compare the weighted estimator wPI
with the direct method (DM) and weighted IPS (wIPS).
(Weighted variants outperformed the unweighted ones.)
We implement two variants of DM: regression trees and lasso,
each trained on the first $n/2$ examples and using the remaining $n/2$ examples
for evaluation according to \Eq{dm}.
We also include an aspirational baseline, \OnPolicy,
which corresponds to deploying the target policy as in an A/B test and returning
the average of observed rewards.
This is the expensive alternative we wish to avoid.

We evaluate the estimators by recording the root mean square error
(RMSE) as a function of the number of samples, averaged over at least
25 independent runs. We do this for 40 different experimental
conditions,
considering two reward metrics, two slate-space sizes, and 10 combinations of target and logging policies (including the choice of $\alpha$).
The top row of \Fig{synth_expt1} shows results for four
representative conditions (see \App{approach_family} for all results),
while the middle and bottom rows aggregate across conditions.
To produce the aggregates, we shift and rescale the RMSE of all methods,
at 600k (middle row) or 60k (bottom row) samples, so the best performance is at 0.001 and
the worst is at 1.0 (excluding \OnPolicy). (We use 0.001 instead of 0.0 to allow
plotting on a log scale.) The aggregate plots display the cumulative distribution
function of these normalized RMSE values across 10 target-logging combinations,
keeping the metric and the slate-space size fixed.


The pseudoinverse estimator wPI easily dominates wIPS across all experimental conditions,
as can be seen in \Fig{synth_expt1} (top) and in \App{approach_family}.
While wIPS and IPS are (asymptotically) unbiased even without linearity assumption, they both suffer from a large variance caused by the slate size. The variance and hence the mean square error of wIPS and IPS grows exponentially with the slate size, so they perform poorly beyond the smallest slate sizes.
DM performs well in
some cases, especially with few samples, but often plateaus or degrades
eventually as it overfits on the logging distribution, which is
different from the target.
While wPI does not always outperform DM methods (e.g.,
\Fig{synth_expt1}, top row, second from right), it is the only method
that works robustly across all conditions, as can be seen in the aggregate plots.
In general, choosing between DM and wPI is largely a matter of bias-variance tradeoff. DM can be particularly good with very small data sizes, because of its low variance, and in those settings it is often the best choice. However, PI performs comprehensively better given enough data (see \Fig{synth_expt1}, middle row).

In the top row of \Fig{synth_expt1}, we see that, as expected, wPI is
biased for the \textup{ERR} metric since \textup{ERR} does not satisfy
linearity. The right two panels also demonstrate the
effect of varying $m$ and $\ell$. While wPI
deteriorates somewhat
for the larger slate space, it still gives a meaningful estimate.
In contrast,
wIPS fails to produce any meaningful estimate in the larger slate
space and its RMSE barely improves with more data.
Finally, the left two plots in the top row show
that wPI is fairly insensitive to the amount of stochasticity in logging,
whereas DM improves with more overlap between logging and target.

\subsection{Semi-synthetic policy optimization}
\label{sec:expt_opt}

We now show how to use the \approach for off-policy optimization.  We
leverage pointwise learning to rank (L2R) algorithms, which learn a
scoring function for query-document pairs by fitting to relevance
labels. We call this the \emph{supervised} approach, as it requires
relevance labels.

Instead of requiring relevance labels, we use the \approach to convert
page-level reward into per-slot reward components---the estimates of
$\phi_x(j,a)$---and these become targets for regression.  Thus,
  the \approach enables pointwise L2R to optimize whole-page metrics even without relevance labels.
Given a contextual bandit dataset $\set{(x_i,\vs_i,r_i)}_{i\le n}$
collected by the logging policy $\mu$, we begin by creating the
estimates of $\vphi_{x_i}$:
 $\hat{\vphi}_i=\mgammaparam{\mu}{x_i}^\dagger\hvpri$,
turning the $i$-th contextual bandit example into $\ell m$ regression
examples.
The trained regression model is used to create a slate, starting with
the highest scoring slot-action pair, and continuing greedily
(excluding the pairs with the already chosen slots or actions).
This procedure is detailed in Appendix~\ref{sec:opt_family}.
Note that without the linearity assumptions, our imputed regression targets might not lead to the best possible learned policy, but we still expect to adapt somewhat to the slate-level metric.

We use the MSLR-WEB10K dataset~\cite{letor2013} to compare our approach with
benchmarked results~\cite{2015CCL} for NDCG@3 (i.e., $\ell=3$).%
\footnote{Our dataset here differs from the dataset MSLR-WEB30K used in \Sec{expt_ssynth}. There our goal was to study realistic problem dimensions, e.g., constructing length-10 rankings out of 100 candidates.
Here, we use MSLR-WEB10K, because it is the largest dataset with public benchmark numbers by state-of-the-art approaches (specifically LambdaMART).}
This
dataset contains 10k queries, over 1.2M relevance judgments, and
up to 908 judged documents per query.  The
state-of-the-art \emph{listwise} L2R method on this dataset
is a highly tuned variant of LambdaMART~\cite{Asadi2013} (with an ensemble
of 1000 trees, each with up to 70 leaves).

We use the
provided 5-fold split and always train on bandit data collected by
uniform logging from four folds, while evaluating with
supervised data on the fifth.  We compare our approach, titled PI-OPT,
against the supervised approach (SUP), trained to predict the
\emph{gains}, equal to $2^{\rel(x,a)} - 1$, computed using annotated
relevance judgements in the training folds (predicting raw relevances
was inferior). Both PI-OPT and SUP train gradient boosted regression
trees (with 1000 trees, each with up to 70 leaves).
Additionally, we also experimented
with the ERR metric.




The average test-set performance (computed using
ground-truth relevance judgments for each test set) across the
$5$-folds is reported in Table~\ref{tbl:opt_results}. Our method,
PI-OPT is competitive with the supervised baseline SUP for NDCG, and is
substantially superior for ERR. A different transformation
instead of gains might yield a stronger supervised baseline for ERR,
but this only illustrates the key benefit of PI-OPT: \emph{the right pointwise targets are automatically inferred for any whole-page metric.}
Both PI-OPT and SUP are slightly worse than LambdaMART for NDCG@3, but they are arguably not as highly tuned,
and PI-OPT only uses the slate-level metric.

\begin{table}[ht]
\begin{center}\begin{small}
\caption{Comparison of L2R approaches optimizing NDCG@3 and ERR@3.
  LambdaMART is a tuned list-wise approach.
  SUP and PI-OPT use the same pointwise L2R learner; SUP uses
  $8\times 10^5$ relevance judgments, PI-OPT uses
  $10^7$ samples (under uniform logging) with page-level rewards.}
\label{tbl:opt_results}
\begin{tabular}{|l||c|c|c|c|}
\hline
Metric & LambdaMART & uniformly random & SUP & PI-OPT \\
\hline
NDCG@3 & $0.457$ & $0.152$ & $0.438$ & $0.421$ \\
ERR@3 & --- & $0.096$ & $0.311$ & $0.321$\\
\hline
\end{tabular}
\end{small}
\end{center}
\end{table}

\subsection{Real-world experiments}
\label{sec:expt_bing}

We finally evaluate all methods using logs collected from a popular
search engine.
The dataset consists of search queries, for which the logging policy
randomly (non-uniformly) chooses a slate of size $\ell=5$ from a small
pre-filtered set of documents of size $m \leq 8$.  After
preprocessing, there are 77 unique queries and 22K total examples,
meaning that for each query, we have logged impressions for many of
the available slates.  As before, we create the logs by sampling
queries uniformly at random, and using a logging policy that samples
uniformly from the slates shown for this query.


We consider two page-level metrics: time-to-success (TTS) and
\UtilityRate.  TTS measures the number of seconds between presenting
the results and the first satisfied click from the user, defined as
any click for which the user stays on the linked page for sufficiently
long. TTS value is capped and scaled to $[0,1]$.
\UtilityRate is a more complex page-level metric of user
satisfaction. It captures the interaction of a user with the page as a
timeline of events (such as clicks) and their durations.  The
events are classified as revealing a positive or negative utility to
the user and their contribution is proportional to their duration.
\UtilityRate takes values in $[-1,1]$.

We evaluate a target policy based on a logistic regression classifier
trained to predict clicks and using the predicted probabilities to
score slates. We restrict the target policy to pick among the slates
in our logs, so we know the ground truth slate-level reward.  Since we
know the query distribution, we can calculate the target policy's
value exactly, and measure RMSE relative to this true value.

We compare our estimator (PI) with three baselines similar
to those from \Sec{expt_ssynth}:
\Direct, IPS and \OnPolicy.
\Direct uses regression
trees over roughly 20,000 slate-level features.

\Fig{bing_experiment:optimized} from the introduction shows that PI
provides a consistent multiplicative improvement in RMSE over IPS,
which suffers due to high variance. Starting at moderate sample sizes,
PI also outperforms DM, which suffers due to substantial bias.

\section{Discussion}
\label{sec:discussion}

In this paper we have introduced a new estimator (PI) for off-policy
evaluation in combinatorial contextual bandits under a linearity
assumption on the slate-level rewards. Our theoretical and empirical
analysis demonstrates the merits of the approach. The empirical
results show a favorable bias-variance tradeoff. Even in datasets and
metrics where our assumptions are violated, the PI estimator typically
outperforms all baselines. Its performance, especially at smaller sample sizes,
could be further improved by designing doubly-robust variants~\citep{dudik11doublyrobust}
and possibly also incorporating weight clipping~\citep{switch}.

One promising approach to relax \assumes is to posit a decomposition
over pairs (or tuples) of slots to capture higher-order interactions
such as diversity. More generally, one could replace slate spaces by
arbitrary compact convex sets, as done in linear bandits. In these
settings, the pseudoinverse estimator could still be applied, but
tight sample-complexity analysis is open for future research.

\begin{small}
\bibliographystyle{plainnat}
\bibliography{bibliography}

\begin{thebibliography}{35}
\providecommand{\natexlab}[1]{#1}
\providecommand{\url}[1]{\texttt{#1}}
\expandafter\ifx\csname urlstyle\endcsname\relax
  \providecommand{\doi}[1]{doi: #1}\else
  \providecommand{\doi}{doi: \begingroup \urlstyle{rm}\Url}\fi

\bibitem[Asadi and Lin(2013)]{Asadi2013}
Nima Asadi and Jimmy Lin.
\newblock Training efficient tree-based models for document ranking.
\newblock In \emph{European Conference on Advances in Information Retrieval},
  2013.

\bibitem[Auer(2002)]{Auer02}
Peter Auer.
\newblock Using confidence bounds for exploitation-exploration trade-offs.
\newblock \emph{Journal of Machine Learning Research}, 2002.

\bibitem[Auer et~al.(2002)Auer, Cesa-Bianchi, Freund, and Schapire]{EXP4}
Peter Auer, Nicol{\`o} Cesa-Bianchi, Yoav Freund, and Robert~E Schapire.
\newblock The nonstochastic multiarmed bandit problem.
\newblock \emph{SIAM Journal on Computing}, 2002.

\bibitem[Bottou et~al.(2013)Bottou, Peters, Qui{\~n}onero-Candela, Charles,
  Chickering, Portugaly, Ray, Simard, and Snelson]{Bottou13Counterfactual}
L\'{e}on Bottou, Jonas Peters, Joaquin Qui{\~n}onero-Candela, Denis Charles,
  Max Chickering, Elon Portugaly, Dipankar Ray, Patrice Simard, and Ed~Snelson.
\newblock Counterfactual reasoning and learning systems: The example of
  computational advertising.
\newblock \emph{Journal of Machine Learning Research}, 2013.

\bibitem[Bubeck and Cesa-Bianchi(2012)]{bubeck2012regret}
S{\'e}bastien Bubeck and Nicol{\`o} Cesa-Bianchi.
\newblock Regret analysis of stochastic and nonstochastic multi-armed bandit
  problems.
\newblock \emph{Foundations and Trends{\textregistered} in Machine Learning},
  2012.

\bibitem[Burges et~al.(2005)Burges, Shaked, Renshaw, Lazier, Deeds, Hamilton,
  and Hullender]{ndcg2005}
Chris Burges, Tal Shaked, Erin Renshaw, Ari Lazier, Matt Deeds, Nicole
  Hamilton, and Greg Hullender.
\newblock Learning to rank using gradient descent.
\newblock In \emph{International Conference on Machine Learning}, 2005.

\bibitem[Cesa-Bianchi and Lugosi(2012)]{cesa2012combinatorial}
Nicolo Cesa-Bianchi and G{\'a}bor Lugosi.
\newblock Combinatorial bandits.
\newblock \emph{Journal of Computer and System Sciences}, 2012.

\bibitem[Chapelle and Zhang(2009)]{chapelle2009dynamic}
Olivier Chapelle and Ya~Zhang.
\newblock A dynamic {B}ayesian network click model for web search ranking.
\newblock In \emph{International Conference on World Wide Web}, 2009.

\bibitem[Chapelle et~al.(2009)Chapelle, Metlzer, Zhang, and
  Grinspan]{Chapelle2009}
Olivier Chapelle, Donald Metlzer, Ya~Zhang, and Pierre Grinspan.
\newblock Expected reciprocal rank for graded relevance.
\newblock In \emph{Conference on Information and Knowledge Management}, 2009.

\bibitem[Chu et~al.(2011)Chu, Li, Reyzin, and Schapire]{chu2011contextual}
Wei Chu, Lihong Li, Lev Reyzin, and Robert~E Schapire.
\newblock Contextual bandits with linear payoff functions.
\newblock In \emph{Artificial Intelligence and Statistics}, 2011.

\bibitem[Dani et~al.(2008)Dani, Hayes, and Kakade]{DaniHaKa08}
Varsha Dani, Thomas~P. Hayes, and Sham~M. Kakade.
\newblock The price of bandit information for online optimization.
\newblock In \emph{Advances in Neural Information Processing Systems}, 2008.

\bibitem[Dud{\'i}k et~al.(2011)Dud{\'i}k, Langford, and
  Li]{dudik11doublyrobust}
Miroslav Dud{\'i}k, John Langford, and Lihong Li.
\newblock Doubly robust policy evaluation and learning.
\newblock In \emph{International Conference on Machine Learning}, 2011.

\bibitem[Dud{\'\i}k et~al.(2014)Dud{\'\i}k, Erhan, Langford, and
  Li]{dudik2014doubly}
Miroslav Dud{\'\i}k, Dumitru Erhan, John Langford, and Lihong Li.
\newblock Doubly robust policy evaluation and optimization.
\newblock \emph{Statistical Science}, 2014.

\bibitem[Dupret and Piwowarski(2008)]{Dupret2008}
Georges~E. Dupret and Benjamin Piwowarski.
\newblock A user browsing model to predict search engine click data from past
  observations.
\newblock In \emph{SIGIR Conference on Research and Development in Information
  Retrieval}, 2008.

\bibitem[Filippi et~al.(2010)Filippi, Cappe, Garivier, and
  Szepesv{\'a}ri]{filippi2010parametric}
Sarah Filippi, Olivier Cappe, Aur{\'e}lien Garivier, and Csaba Szepesv{\'a}ri.
\newblock Parametric bandits: The generalized linear case.
\newblock In \emph{Advances in Neural Information Processing Systems}, 2010.

\bibitem[Guo et~al.(2009)Guo, Liu, Kannan, Minka, Taylor, Wang, and
  Faloutsos]{Guo2009}
Fan Guo, Chao Liu, Anitha Kannan, Tom Minka, Michael Taylor, Yi-Min Wang, and
  Christos Faloutsos.
\newblock Click chain model in web search.
\newblock In \emph{International Conference on World Wide Web}, 2009.

\bibitem[Hofmann et~al.(2016)Hofmann, Li, Radlinski, et~al.]{hofmann2016online}
Katja Hofmann, Lihong Li, Filip Radlinski, et~al.
\newblock Online evaluation for information retrieval.
\newblock \emph{Foundations and Trends in Information Retrieval}, 2016.

\bibitem[Horvitz and Thompson(1952)]{HorvitzTh52}
Daniel~G Horvitz and Donovan~J Thompson.
\newblock A generalization of sampling without replacement from a finite
  universe.
\newblock \emph{Journal of the American Statistical Association}, 1952.

\bibitem[Kale et~al.(2010)Kale, Reyzin, and Schapire]{kale2010non}
Satyen Kale, Lev Reyzin, and Robert~E Schapire.
\newblock Non-stochastic bandit slate problems.
\newblock In \emph{Advances in Neural Information Processing Systems}, 2010.

\bibitem[Kohavi et~al.(2009)Kohavi, Longbotham, Sommerfield, and
  Henne]{kohavi2009controlled}
Ron Kohavi, Roger Longbotham, Dan Sommerfield, and Randal~M Henne.
\newblock Controlled experiments on the web: survey and practical guide.
\newblock \emph{Knowledge Discovery and Data Mining}, 2009.

\bibitem[Krishnamurthy et~al.(2016)Krishnamurthy, Agarwal, and
  Dud{\'i}k]{krishnamurthy2015efficient}
Akshay Krishnamurthy, Alekh Agarwal, and Miroslav Dud{\'i}k.
\newblock Efficient contextual semi-bandit learning.
\newblock \emph{Advances in Neural Information Processing Systems}, 2016.

\bibitem[Kveton et~al.(2015)Kveton, Wen, Ashkan, and
  Szepesv{\'a}ri]{kveton2015tight}
Branislav Kveton, Zheng Wen, Azin Ashkan, and Csaba Szepesv{\'a}ri.
\newblock Tight regret bounds for stochastic combinatorial semi-bandits.
\newblock In \emph{Artificial Intelligence and Statistics}, 2015.

\bibitem[Langford and Zhang(2008)]{langford2008epoch}
John Langford and Tong Zhang.
\newblock The epoch-greedy algorithm for multi-armed bandits with side
  information.
\newblock In \emph{Advances in Neural Information Processing Systems}, 2008.

\bibitem[Langford et~al.(2008)Langford, Strehl, and Wortman]{LangfordStWo08}
John Langford, Alexander Strehl, and Jennifer Wortman.
\newblock Exploration scavenging.
\newblock In \emph{International Conference on Machine Learning}, 2008.

\bibitem[Li et~al.(2010)Li, Chu, Langford, and Schapire]{li2010contextual}
Lihong Li, Wei Chu, John Langford, and Robert~E Schapire.
\newblock A contextual-bandit approach to personalized news article
  recommendation.
\newblock In \emph{International Conference on World Wide Web}, 2010.

\bibitem[Li et~al.(2011)Li, Chu, Langford, and Wang]{Li11}
Lihong Li, Wei Chu, John Langford, and Xuanhui Wang.
\newblock Unbiased offline evaluation of contextual-bandit-based news article
  recommendation algorithms.
\newblock In \emph{International Conference on Web Search and Data Mining},
  2011.

\bibitem[Li et~al.(2015)Li, Zitouni, and Kim]{LihongTopIPS}
Lihong Li, Imed Zitouni, and Jin~Young Kim.
\newblock Toward predicting the outcome of an a/b experiment for search
  relevance.
\newblock In \emph{International Conference on Web Search and Data Mining},
  2015.

\bibitem[Petersen et~al.(2008)Petersen, Pedersen, et~al.]{petersen2008matrix}
Kaare~Brandt Petersen, Michael~Syskind Pedersen, et~al.
\newblock The matrix cookbook.
\newblock \emph{Technical University of Denmark}, 2008.

\bibitem[Qin et~al.(2014)Qin, Chen, and Zhu]{qin2014contextual}
Lijing Qin, Shouyuan Chen, and Xiaoyan Zhu.
\newblock Contextual combinatorial bandit and its application on diversified
  online recommendation.
\newblock In \emph{International Conference on Data Mining}, 2014.

\bibitem[Qin and Liu(2013)]{letor2013}
Tao Qin and Tie{-}Yan Liu.
\newblock Introducing {LETOR} 4.0 datasets.
\newblock \emph{arXiv:1306.2597}, 2013.

\bibitem[Rusmevichientong and Tsitsiklis(2010)]{rusmevichientong2010linearly}
Paat Rusmevichientong and John~N Tsitsiklis.
\newblock Linearly parameterized bandits.
\newblock \emph{Mathematics of Operations Research}, 2010.

\bibitem[Swaminathan and Joachims(2015)]{swaminathan15}
Adith Swaminathan and Thorsten Joachims.
\newblock Counterfactual risk minimization: Learning from logged bandit
  feedback.
\newblock In \emph{International Conference on Machine Learning}, 2015.

\bibitem[Tax et~al.(2015)Tax, Bockting, and Hiemstra]{2015CCL}
Niek Tax, Sander Bockting, and Djoerd Hiemstra.
\newblock A cross-benchmark comparison of 87 learning to rank methods.
\newblock \emph{Information Processing and Management}, 2015.

\bibitem[{Wang} et~al.(2017){Wang}, {Agarwal}, and {Dudik}]{switch}
Yu-Xiang {Wang}, Alekh {Agarwal}, and Miroslav {Dudik}.
\newblock Optimal and adaptive off-policy evaluation in contextual bandits.
\newblock In \emph{International Conference on Machine Learning}, 2017.

\bibitem[Wang et~al.(2016)Wang, Yin, Jie, Wang, Yamada, Chang, and
  Mei]{Wang2016}
Yue Wang, Dawei Yin, Luo Jie, Pengyuan Wang, Makoto Yamada, Yi~Chang, and
  Qiaozhu Mei.
\newblock Beyond ranking: Optimizing whole-page presentation.
\newblock In \emph{International Conference on Web Search and Data Mining},
  pages 103--112, 2016.

\end{thebibliography}
\end{small}

\vfill
\newpage
\onecolumn
\begin{appendices}
\section{Proof of Proposition~\ref{prop:unbiased}}
\label{app:proofs:estimator}

\begin{lemma}
\label{lem:slatevalue}
If \assumes holds and $\mu(\vs\given x) > 0$,
then $V(x,\vs) = \slateind^T\mgammalogger^\dagger\vprx$.
\end{lemma}

\begin{proof}
Fix one $x$ for the entirety of the proof. Recall from \Sec{estimator} that
\[
   V(x,\vs) = \slateind^T\decompvec_x
\enspace.
\]
Let $N = \card{\supp\mu(\cdot\given x)}$ be the size of the support of $\mu(\cdot\given x)$
and let $\mM \in \{0,1\}^{N\times \poolsize\slatesize}$ denote the binary matrix with rows $\slateind^T$ for each $\vs \in \supp\mu(\cdot\given x)$.
Thus $\mM\decompvec_x$ is the vector enumerating $V(x,\vs)$ over $\vs$ for which $\mu(\vs\given x)>0$.
Let $\Null(\mM)$ denote the null space of $\mM$ and $\mPi$ be the projection on $\Null(\mM)$. Let $\decompvec_x^\star= (\mI - \mPi)\decompvec_x$.
Then clearly, $\mM\decompvec_x = \mM\decompvec_x^\star$, and hence, for any $\vs \in \supp\loggingpolicy(\cdot\given x)$,
\begin{equation}
\label{eq:phi:star}
   V(x,\vs) = \slateind^T\decompvec_x^\star
\enspace.
\end{equation}
We will now show that $\decompvec_x^\star=\mgammalogger^\dagger\prvecx$, which will complete the proof.

Recall from \Sec{estimator} that
\begin{equation}
\label{eq:vprx}
  \vprx = \mgammalogger \decompvec_x
\enspace.
\end{equation}
Next
note that $\mgammalogger$ in symmetric positive semidefinite by definition, so
\[
  \Null(\mgammalogger)
  = \set{\vv:\:\vv^T\mgammalogger\vv=0}
  = \set{\vv:\:\slateind^T\vv=0\text{ for all }\vs\in\supp\mu(\cdot\given x)}
  = \Null(\mM)
\]
where the first step follows by positive semi definiteness of $\mgammalogger$,
the second step is from the definition of $\mgammalogger$,
and the final step from the definition of $\mM$.
Since $\Null(\mgammalogger)=\Null(\mM)$, we have from \Eq{vprx} that
$\prvec_x = \mgammalogger \decompvec_x^\star$, but, importantly, this also implies $\decompvec_x^\star \perp \Null(\mgammalogger)$, so by the definition of the pseudoinverse,
\[
  \mgammalogger^\dagger \prvec_x = \decompvec_x^\star.
\]
This proves \Lem{slatevalue}, since for any $\vs$ with $\mu(\vs\given x) > 0$, we argued that $V(x,\vs) = \slateind^T \decompvec_x^\star = \slateind^T \mgammalogger^\dagger \prvec_x$.
\end{proof}

\begin{proof}[Proof of \Prop{unbiased}]
Note that it suffices to analyze the expectation of a single term in the estimator, that is
\[
  \sum_{\vs\in S} \pi(\vs\given x_i)\slateind^T\mgammaparam{\mu}{x_i}^\dagger\hvpri
\enspace.
\]
First note that $\EE_{(\vs_i,r_i)\sim\mu(\cdot,\cdot\given x_i)}\bigBracks{\hvpri}=\vpr_{x_i}$, because
\[
  \EE_{(\vs_i,r_i)\sim\mu(\cdot,\cdot\given x_i)}
    \bigBracks{\hpri(j,a)}
  =
    \EE_{(\vs_i,r_i)\sim\mu(\cdot,\cdot\given x_i)}
    \bigBracks{r_i\one\braces{s_j=a}}
  = \pr_{x_i}(j,a)
\enspace.
\]
The remainder follows by \Lem{slatevalue}:
\begin{align}
\notag
\EE\Bracks{
  \sum_{\vs\in S} \pi(\vs\given x_i)\slateind^T\mgammaparam{\mu}{x_i}^\dagger\hvpri
 }
&=\EE_{x_i\sim D}
  \Bracks{
    \sum_{\vs\in S}
    \pi(\vs\given x_i)\slateind^T\mgammaparam{\mu}{x_i}^\dagger
    \;
    \EE_{(\vs_i,r_i)\sim\mu(\cdot,\cdot\given x_i)}\bigBracks{\hvpri}
  }
\\
\notag
&=\EE_{x_i\sim D}
  \Bracks{
    \sum_{\vs\in S}
    \pi(\vs\given x_i)\slateind^T\mgammaparam{\mu}{x_i}^\dagger
    \vpr_{x_i}
  }
\\
\tag*{\qed}
&=\EE_{x_i\sim D}
  \Bracks{
    \sum_{\vs\in S}
    \pi(\vs\given x_i) V(x_i,\vs)
  }
 =V(\pi)
\enspace.
\end{align}%
\renewcommand{\qed}{}%
\end{proof}

\section{Proof of Theorem~\ref{thm:deviation:gen}}
\label{app:deviation}

\begin{proof}
The proof is based on an application of Bernstein's inequality to the centered sum
\[
\sum_{i=1}^n \Bracks{ \q{\pi}{x_i}^T \mgammaparam{\loggingpolicy}{x_i}^\dagger \hvpri - V(\pi)
             }
\enspace.
\]
The fact that this quantity is centered is directly from \Prop{unbiased}.
We must compute both the second moment and the range to apply Bernstein's inequality.
By independence, we can focus on just one term, so we will drop the subscript $i$.
First, bound the variance:
\begin{align*}
\Var\Bracks{
  \qtarget^T \mgammalogger^\dagger \hvpr
}
&\le
  \EE_\mu\Bracks{
     \Parens{\qtarget^T \mgammalogger^\dagger \hvpr}^2
  }
\\
&=
  \EE_\mu\Bracks{
     \Parens{\qtarget^T \mgammalogger^\dagger\,r\slateind}^2
  }
\\
&\le
  \EE_\mu\Bracks{
     \Parens{\qtarget^T \mgammalogger^\dagger\slateind}^2
  }
\\
&=
  \EE_{x \sim D}\BigBracks{
     \qtarget^T\mgammalogger^\dagger
        \;\EE_{\vs\sim\loggingpolicy(\cdot\given x)}\!\Bracks{\slateind\slateind^T}\,
     \mgammalogger^\dagger\qtarget
  }
\\
&=
  \EE_{x \sim D}\BigBracks{
     \qtarget^T\mgammalogger^\dagger
        \mgammalogger
     \mgammalogger^\dagger\qtarget
  }
\\
&=
  \EE_{x \sim D}\Bracks{
     \qtarget^T\mgammalogger^\dagger\qtarget
  }
\\
&=
  \sigma^2
\enspace.
\end{align*}
Thus the per-term variance is at most $\sigma^2$.
We now bound the range, again focusing on one term,
\begin{align*}
\Abs{ \qtarget^T \mgammalogger^\dagger \hvpr - V(\pi) }
&\le
  \Abs{ \qtarget^T \mgammalogger^\dagger \hvpr } + 1
\\
&=
  \Abs{ \qtarget^T \mgammalogger^\dagger r\one_\vs } + 1
\\
&\le
  \Abs{ \qtarget^T \mgammalogger^\dagger \one_\vs } + 1
\\
&\le
  \rho + 1
\end{align*}
The first line here is the triangle inequality, coupled with the fact that since rewards are bounded in $[-1,1]$, so is $V(\pi)$.
The second line is from the definition of $\hvpr$, while the third follows
because $r\in[-1,1]$.
The final line follows from the definition of $\rho$.

Now, we may apply Bernstein's inequality, which says that for any $\delta \in (0,1)$, with probability at least $1-\delta$,
\[
  \Abs{
     \sum_{i=1}^n\Bracks{
        \q{\pi}{x_i}^T \mgammaparam{\loggingpolicy}{x_i}^\dagger \hvpri - V(\pi)
     }
  }
  \le \sqrt{2n\sigma^2\ln(2/\delta)} + \frac{2 (\rho + 1) \ln(2/\delta)}{3}
\enspace.
\]
The theorem follows by dividing by $n$.
\end{proof}

\section{Pseudo-inverse estimator when \texorpdfstring{$\pi=\mu$}{pi=mu}}
\label{app:pi:eq:mu}

In this section we show that when the target policy coincides with logging (i.e., $\pi=\mu$), we have $\sigma^2=\rho=1$, i.e., the bound of \Thm{deviation:gen} is independent of the number of actions and slots. Indeed,
in \Claim{pi:eq:mu} we will see that the estimator actually simplifies to taking an empirical average of rewards which are bounded in $[-1,1]$.
Before proving \Claim{pi:eq:mu} we prove one supporting claim:

\begin{claim}
\label{claim:qmux}
For any policy $\mu$ and context $x$,
we have
$
  \q{\mu}{x}^T\mgammalogger^\dagger\one_\vs=1
$
for all $\vs\in\supp\mu(\cdot\given x)$.
\end{claim}
\begin{proof}
To simplify the exposition, write $\vq$ and $\mgamma$ instead of a more verbose $\q{\mu}{x}$ and $\mgammalogger$.

The bulk of the proof is in deriving an explicit expression for $\mgamma^\dagger$.
We begin by expressing $\mgamma$ in a suitable basis.
Since $\mgamma$ is the matrix of second moments and $\vq$ is the vector of first moments of $\one_\vs$, the matrix $\mgamma$ can be written as
\[
  \mgamma=\mV+\vq\vq^T
\]
where $\mV$ is the covariance matrix of $\one_\vs$, i.e., $\mV\coloneqq\EE_{\vs\sim\mu(\cdot\given x)}\bigBracks{
  (\one_\vs-\vq)(\one_\vs-\vq)^T
  }$.
Assume that the rank of $\mV$ is $r$ and consider the eigenvalue decomposition of $\mV$
\[
  \mV=\sum_{i=1}^r \lambda_i\vu_i\vu_i^T= \mU\mlambda\mU^T
\enspace,
\]
where $\lambda_i>0$ and vectors $\vu_i$ are orthonormal; we have grouped the eigenvalues into the diagonal matrix $\mlambda\coloneqq\diag(\lambda_1,\dotsc,\lambda_r)$ and eigenvectors into the matrix $\mU\coloneqq(\vu_1\;\vu_2\;\dotsc\;\vu_r)$.

We next argue that $\vq\not\in\Range(\mV)$. To see this, note that the all-ones-vector $\one$ is in the null space of $\mV$ because, for any valid slate $\vs$, we have $\one_\vs^T\one=\ell$ and thus also for the convex combination $\vq$ we have $\vq^T\one=\ell$, which means that
\[
  \one^T\mV\one=\EE_{\vs\sim\mu(\cdot\given x)}\bigBracks{
  \one^T(\one_\vs-\vq)(\one_\vs-\vq)^T\one
  }=0
\enspace.
\]
Now, since $\one\perp\Range(\mV)$ and $\vq^T\one=\ell$, we have that $\vq\not\in\Range(\mV)$. In particular, we can write $\vq$ in the form
\begin{equation}
\label{eq:vq}
  \vq=\sum_{i=1}^r \beta_i\vu_i + \alpha\vn
  =
  \begin{pmatrix}
     \mU & \vn
  \end{pmatrix}
  \begin{pmatrix}
     \vbeta
  \\
     \alpha
  \end{pmatrix}
\end{equation}
where $\alpha\ne 0$ and $\vn\in\Null(\mV)$ is a unit vector. Note that $\vn\perp\vu_i$ since $\vu_i\perp\Null(\mV)$. Thus, the second moment matrix
$\mgamma$ can be written as
\begin{equation}
\label{eq:mgamma:factor}
  \mgamma=\mV+\vq\vq^T
         =
  \begin{pmatrix}
     \mU & \vn
  \end{pmatrix}
  \begin{pmatrix}
     \mlambda+\vbeta\vbeta^T & \alpha\vbeta
  \\
     \alpha\vbeta^T & \alpha^2
  \end{pmatrix}
  \begin{pmatrix}
     \mU & \vn
  \end{pmatrix}^T
\enspace.
\end{equation}
Let $\mQ\in\RR^{(r+1)\times(r+1)}$ denote the middle matrix in the factorization of \Eq{mgamma:factor}:
\begin{equation}
\label{eq:mQ}
  \mQ\coloneqq
  \begin{pmatrix}
     \mlambda+\vbeta\vbeta^T & \alpha\vbeta
  \\
     \alpha\vbeta^T & \alpha^2
  \end{pmatrix}
\enspace.
\end{equation}
This matrix is a representation of $\mgamma$ with respect to the basis
$\set{\vu_1,\dotsc,\vu_r,\vn}$. Since $\vq\not\in\Range(\mV)$, the rank
of $\mgamma$ and that of $\mQ$ is $r+1$. Thus, $\mQ$ is invertible and
\begin{equation}
\label{eq:inv:factor}
  \mgamma^\dagger
         =
  \begin{pmatrix}
     \mU & \vn
  \end{pmatrix}
  \mQ^{-1}
  \begin{pmatrix}
     \mU & \vn
  \end{pmatrix}^T
\enspace.
\end{equation}
To obtain $\mQ^{-1}$, we use the following identity (see~\cite{petersen2008matrix}):
\begin{equation}
\label{eq:inv}
  \begin{pmatrix}
  \mA_{11} & \mA_{12}
\\
  \mA_{21} & \mA_{22}
  \end{pmatrix}^{-1}
=
  \begin{pmatrix}
  \mM^{-1}
  & -\mM^{-1}\mA_{12}\mA_{22}^{-1}
\\
  -\mA_{22}^{-1}\mA_{21}\mM^{-1}
  & \mA_{22}^{-1}\mA_{21}\mM^{-1}\mA_{12}\mA_{22}^{-1} + \mA_{22}^{-1}
  \end{pmatrix}
\enspace,
\end{equation}
where $\mM\coloneqq\mA_{11}-\mA_{12}\mA_{22}^{-1}\mA_{21}$ is
the Schur complement of $\mA_{22}$. The identity of \Eq{inv} holds
whenever $\mA_{22}$ and its Schur complement $\mM$ are both invertible.
In the block representation of \Eq{mQ}, we have $\mA_{22}=\alpha^2\ne 0$ and
\[
  \mM = (\mlambda+\vbeta\vbeta^T) - (\alpha\vbeta)\alpha^{-2}(\alpha\vbeta^T)
      = \mlambda
\enspace,
\]
so \Eq{inv} can be applied to obtain $\mQ^{-1}$:
\begin{align}
\notag
  \mQ^{-1}
=
  \begin{pmatrix}
     \mlambda+\vbeta\vbeta^T & \alpha\vbeta
  \\
     \alpha\vbeta^T & \alpha^2
  \end{pmatrix}^{-1}
&=
  \begin{pmatrix}
     \mlambda^{-1}
     &
    -\mlambda^{-1}(\alpha\vbeta)\alpha^{-2}
\\
    -\alpha^{-2}(\alpha\vbeta^T)\mlambda^{-1}
    &
    \alpha^{-2}(\alpha\vbeta^T)\mlambda^{-1}(\alpha\vbeta)\alpha^{-2}
    +\alpha^{-2}
  \end{pmatrix}
\\
\label{eq:mQ:inv}
&=
  \begin{pmatrix}
     \mlambda^{-1}
     &
    -\alpha^{-1}\mlambda^{-1}\vbeta
\\
    -\alpha^{-1}\vbeta^T\mlambda^{-1}
    &
    \alpha^{-2}(1+\vbeta^T\mlambda^{-1}\vbeta)
  \end{pmatrix}
\enspace.
\end{align}

Next, we will evaluate $\mgamma^\dagger\vq$, using the factorizations in
Eqs.~\eqref{eq:inv:factor} and~\eqref{eq:vq}, and substituting \Eq{mQ:inv} for $\mQ^{-1}$:
\begin{align*}
  \mgamma^\dagger\vq
&=
  \begin{pmatrix}
     \mU & \vn
  \end{pmatrix}
  \mQ^{-1}
  \begin{pmatrix}
     \mU & \vn
  \end{pmatrix}^T
  \begin{pmatrix}
     \mU & \vn
  \end{pmatrix}
  \begin{pmatrix}
     \vbeta \\ \alpha
  \end{pmatrix}
\\
&=
  \begin{pmatrix}
     \mU & \vn
  \end{pmatrix}
  \mQ^{-1}
  \begin{pmatrix}
     \vbeta \\ \alpha
  \end{pmatrix}
\\
&=
  \begin{pmatrix}
     \mU & \vn
  \end{pmatrix}
  \begin{pmatrix}
     \mlambda^{-1}\vbeta
     -\mlambda^{-1}\vbeta
     \\
    -\alpha^{-1}\vbeta^T\mlambda^{-1}\vbeta
    +
    \alpha^{-1}(1+\vbeta^T\mlambda^{-1}\vbeta)
  \end{pmatrix}
\\
&=
  \begin{pmatrix}
     \mU & \vn
  \end{pmatrix}
  \begin{pmatrix}
     \zero
     \\
    \alpha^{-1}
  \end{pmatrix}
\\
&= \alpha^{-1}\vn
\enspace.
\end{align*}

To finish the proof, we consider any $\vs\in\supp\mu(\cdot\given x)$ and consider the decomposition of $\one_\vs$ in the basis $\set{\vu_1,\dotsc,\vu_r,\vn}$. First, note that $(\one_\vs-\vq)\perp\Null(\mV)$ since
\[
  \Null(\mV)=
  \bigSet{
    \vv:\:\EE_{\vs\sim\mu(\cdot\given x)}
    \bigBracks{
        \bigParens{(\one_\vs-\vq)^T\vv}^2
    }=0
  }
  =
  \bigSet{
    \vv:\:(\one_\vs-\vq)^T\vv=0
    \text{ for all }\vs\in\supp\mu(\cdot\given x)}
  \;.
\]
Thus, $(\one_\vs-\vq)\in\Range(\mV)$. Therefore, we obtain
\[
  \vq^T\mgammalogger^\dagger\one_\vs
  =
  \alpha^{-1}\vn^T\one_\vs
  =
  \alpha^{-1}\vn^T(\one_\vs-\vq)
  +
  \alpha^{-1}\vn^T\vq
  =0+\alpha^{-1}\alpha=1
\enspace,
\]
where the third equality follows because $(\one_\vs-\vq)\perp\vn$ and
the decomposition in \Eq{vq} shows that $\vn^T\vq=\alpha$.
\end{proof}

\begin{claim}
\label{claim:pi:eq:mu}
If $\pi=\mu$ then $\sigma^2=\rho=1$ and
$\hVpi(\pi)=\hVpi(\mu)=\frac1n\sum_{i=1}^n r_i$.
\end{claim}
\begin{proof}
From \Claim{qmux}
\[
  \q{\mu}{x}^T\mgammalogger^\dagger
  \q{\mu}{x}
  =
  \EE_{\vs\sim\mu(\cdot\given x)}[
  \q{\mu}{x}^T\mgammalogger^\dagger
  \one_\vs]
  =
  1
\enspace.
\]
Taking expectation over $x$ then yields $\sigma^2=1$. Equality $\rho=1$
follows immediately from plugging \Claim{qmux} into the definition of $\rho$.
The final statement of \Claim{pi:eq:mu} follows by applying \Claim{qmux}
to a single term of $\hVpi(\mu)$:
\begin{equation}
\tag*{\qed}
  \q{\mu}{x_i}^T\mgammaparam{\mu}{x_i}^\dagger\,r_i\one_{\vs_i}
  =r_i
\enspace.
\end{equation}
\renewcommand{\qed}{}
\end{proof}

\section{A product slate space under a product logging distribution}
\label{app:product_bounds}

\begin{proposition}
\label{prop:product}
Consider the product slate space where $S(x)=A_1(x)\times\cdots\times A_\ell(x)$
and assume that the logging policy picks any $\vs\in S(x)$ with non-zero probability
and factorizes across the slots as $\mu(\vs\given x)=\prod_j\mu(s_j\given x)$. For any
policy $\pi$, any $\vs\in S(x)$, and any $r\in[-1,1]$ we then have
\begin{equation}
\label{eq:tVpi:product}
  \qtarget^T\mgammalogger^\dagger r \one_\vs
  =
  r\cdot\Bracks{
  \sum_{j=1}^\ell \frac{\pi(s_j\given x)}{\mu(s_j\given x)}
  -\ell+1
  }
\enspace.
\end{equation}
\end{proposition}
\begin{proof}
The proof uses \Claim{qmux} and the identities introduced in its proof. As in the proof of \Claim{qmux}, write $\vq$ and $\mgamma$ instead of $\q{\mu}{x}$ and $\mgammalogger$,
and let $\mV\coloneqq\EE_{\vs\sim\mu(\cdot\given x)}\bigBracks{
  (\one_\vs-\vq)(\one_\vs-\vq)^T
  }$.
Thus, $\mgamma=\mV+\vq\vq^T$. It suffices to show that for any $\vs,\vs'\in S(x)$,
\begin{equation}
\label{eq:ss':product}
  \one_{\vs'}^T\mgamma^\dagger\one_\vs
  =
  \sum_{j=1}^\ell \frac{\one\braces{s'_j=s_j}}{\mu(s_j\given x)}
  -\ell+1
\enspace.
\end{equation}
Pick $\vs,\vs'\in S(x)=\supp \mu(\cdot\given x)$. By \Claim{qmux}, we have
$\vq^T\mgamma^\dagger\one_\vs=\one_{\vs'}\mgamma^\dagger\vq=\vq^T\mgamma^\dagger\vq=1$, so
\begin{align}
\notag
  \one_{\vs'}^T\mgamma^\dagger\one_\vs
&
  =
  \Parens{\one_{\vs'}-\vq}^T\mgamma^\dagger(\one_\vs-\vq)
  +\underbrace{\vq^T\mgamma^\dagger\one_\vs}_{=1}
  +\underbrace{\one_{\vs'}\mgamma^\dagger\vq}_{=1}
  -\underbrace{\vq^T\mgamma^\dagger\vq}_{=1}
\\
\label{eq:prod:1}
& =
  \Parens{\one_{\vs'}-\vq}^T\mgamma^\dagger(\one_\vs-\vq)
  +
  1
\enspace.
\end{align}
Similar to the reasoning at the end of the proof of \Claim{qmux}, we know that $(\one_\vs-\vq)\in\Range(\mV)$ and
$(\one_{\vs'}-\vq)\in\Range(\mV)$. The factorization of $\mgamma^\dagger$ in \Eqs{inv:factor}{mQ:inv} therefore yields
\begin{align}
\notag
  \Parens{\one_{\vs'}-\vq}^T\mgamma^\dagger(\one_\vs-\vq)
&=
  \Parens{\one_{\vs'}-\vq}^T\Parens{\mU\mlambda^{-1}\mU^T}(\one_\vs-\vq)
\\
&=
\label{eq:prod:2}
  \Parens{\one_{\vs'}-\vq}^T\mV^\dagger(\one_\vs-\vq)
\enspace,
\end{align}
where the last step follows from the fact that $\mV=\mU\mlambda\mU^T$, and so $\mV^\dagger=\mU\mlambda^{-1}\mU^T$.

To finish the proof, we study the structure of $\mV$ and $\mV^\dagger$. First, let $\vq_j$ denote the block of $\vq$ corresponding
to the $j$th slot. Its $a$th entry corresponds to the probability $\mu(s_j=a\given x)$. Since the values $s_j$ are conditionally
independent, conditioned on $x$, the covariance matrix $\mV$ takes form $\mV=\diag_{j=1,\dotsc,\ell}\mV_j$, where $\mV_j=\bigParens{\diag_{a\in A_j(x)} q_{j,a}}-\vq_j\vq_j^T$ is the covariance matrix of the multinomial distribution described by $\vq_j$. Thus,
\begin{equation}
\label{eq:prod:3}
  \mV^\dagger=\diag_{j=1,\dotsc,\ell}\mV^\dagger_j
\enspace.
\end{equation}
It can be directly verified that the pseudoinverse of $\mV_j$ takes form
\begin{equation}
\label{eq:multi:pinv}
  \mV_j^\dagger=\mP_{j}\bigParens{\diag_{a\in A_j(x)} q_{j,a}^{-1}}\mP_j
\enspace,
\end{equation}
where $\mP_j\coloneqq\mI_{j}-\one_{j}\one_{j}^T/m_j$, and $\mI_{j}$ is the $m_j\times m_j$ identity matrix,
and $\one_{j}$ the $m_j$-dimensional all-ones vector. To verify that \Eq{multi:pinv} holds, first note that $\mP_j$ is the projection
matrix on $\Range(\mV_j)$. Then set $\mV_j'\coloneqq\mP_{j}\bigParens{\diag_{a\in A_j(x)} q_{j,a}^{-1}}\mP_j$, and
directly verify that $\mV'_j\mV_j=\mP_j$ and $\mV_j\mV'_j=\mP_j$. The first identity can be verified as follows:
\[
  \mV'_j\mV_j = \mP_{j}\bigParens{\diag_{a\in A_j(x)} q_{j,a}^{-1}}\mP_j\mV_j
  =\mP_{j}\bigParens{\diag_{a\in A_j(x)} q_{j,a}^{-1}}\mV_j
  =\mP_{j}\Parens{\mI_j-\one_{j}\vq_j^T}=\mP_j
\enspace.
\]
The second identity follows similarly.

Combining Eqs.~\eqref{eq:prod:2}, \eqref{eq:prod:3} and \eqref{eq:multi:pinv} yields
\begin{align*}
\Parens{\one_{\vs'}-\vq}^T\mV^\dagger(\one_\vs-\vq)
&=
\sum_{j=1}^\ell \bigParens{\one_{s'_j}-\vq_j}^T\mV^\dagger_j(\one_{s_j}-\vq_j)
\\
&=
\sum_{j=1}^\ell \bigParens{\one_{s'_j}-\vq_j}^T\mP_{j}\bigParens{\diag_{a\in A_j(x)} q_{j,a}^{-1}}\mP_j(\one_{s_j}-\vq_j)
\\
&=
\sum_{j=1}^\ell \bigParens{\one_{s'_j}-\vq_j}^T\bigParens{\diag_{a\in A_j(x)} q_{j,a}^{-1}}(\one_{s_j}-\vq_j)
\\
&=
\sum_{j=1}^\ell \bigParens{\one\braces{s'_j=s_j}q_{j,s_j}^{-1}-1}
\\
&=
\sum_{j=1}^\ell \Parens{\frac{\one\braces{s'_j=s_j}}{\mu(s_j\given x)}-1}
\enspace.
\end{align*}
Plugging this back into \Eq{prod:1} then proves \Eq{ss':product}.
\end{proof}

\section{Proof of Corollary~\ref{prop:kappa}}
\label{app:bounds}

For a given logging policy $\mu$ and context $x$, let
\[
  \brho_{\mu,x}\coloneqq\sup_{\vs\in\supp\mu(\cdot\given x)} \one_\vs^T\mgammalogger^\dagger\one_\vs^{\vphantom{T}}
\enspace.
\]
This quantity can be viewed as a norm of $\mgammalogger^\dagger$ with respect to the set of slates chosen by $\mu$ with non-zero probability. It can be used to bound $\sigma^2$ and $\rho$, and thus to bound an error of $\hVpi$:
\begin{proposition}
\label{prop:brho:bound}
For any logging policy $\mu$ and target policy $\pi$ that is absolutely continuous with respect to $\mu$, we have
\[
  \sigma^2\le\rho\le\sup_x\brho_{\mu,x}
\enspace.
\]
\end{proposition}
\begin{proof}
Recall that
\[
   \sigma^2
   =
   \EE_{x\sim D}\Bracks{\q{\pi}{x}^T \mgammalogger^\dagger\q{\pi}{x}^{\vphantom{T}}}
\;,\quad
\textstyle
   \rho
   =
   \adjustlimits\sup_x\sup_{\vs\in\supp\mu(\cdot\given x)}
   \Abs{\q{\pi}{x}^T \mgammalogger^\dagger\one_\vs}
\enspace.
\]
To see that $\sigma^2\le\rho$ note that
\[
  \q{\pi}{x}^T \mgammalogger^\dagger\q{\pi}{x}^{\vphantom{T}}
  =
  \EE_{\vs\sim\pi(\cdot\given x)}
  \bigBracks{
  \q{\pi}{x}^T \mgammalogger^\dagger\one_\vs
  }
  \le
  \rho
\]
where the last inequality follows by the absolute continuity of $\pi$ with respect to $\mu$. It remains to show that $\rho\le\sup_x\brho_{\mu,x}$.

First, by positive semi-definiteness of $\mgammalogger^\dagger$ and
from the definition of $\brho_{\mu,x}$,
we have that for any slates $\vs,\vs'\in\supp\mu(\cdot\given x)$ and any $z\in\set{-1,1}$
\[
  z\one_{\vs'}^T\mgammalogger^\dagger\one_\vs^{\vphantom{T}}
  \le
  \frac{\one_{\vs}^T\mgammalogger^\dagger\one_\vs^{\vphantom{T}}
      + \one_{\vs'}^T\mgammalogger^\dagger\one_{\vs'}^{\vphantom{T}}}
       {2}
  \le
  \max\set{\one_{\vs}^T\mgammalogger^\dagger\one_\vs^{\vphantom{T}},\,
           \one_{\vs'}^T\mgammalogger^\dagger\one_{\vs'}^{\vphantom{T}}}
  \le
  \brho_{\mu,x}
\enspace.
\]
Therefore, for any $\pi$ absolutely continuous with respect to $\mu$ and any $\vs\in\supp\mu(\cdot\given x)$, we have
\[
  \Abs{\q{\pi}{x}^T \mgammalogger^\dagger\one_\vs}
= \max_{z\in\set{-1,1}}
  \EE_{\vs'\sim\pi(\cdot\given x)}
     \bigBracks{z\one_{\vs'}^T \mgammalogger^\dagger\one_\vs^{\vphantom{T}}}
\le
  \brho_{\mu,x}
\enspace.
\]
Taking a supremum over $x$ and $\vs\in\supp\mu(\cdot\given x)$, we obtain
$\rho\le\sup_x\brho_{\mu,x}$.
\end{proof}
We next derive bounds on $\brho_{\mu,x}$ for uniformly-random policies
in the ranking example.  Then we prove a
translation theorem, which allows translating the bound for uniform
distributions into a bound for the $\eps$-uniform
distributions. Finally, we put these results together to prove
Corollary~\ref{prop:kappa}.

\subsection{Uniform logging distribution over rankings}
\label{app:bounds:uniform}

Let $\one_j\in\RR^{\ell m}$ be the vector that is all-ones on the actions in the $j$-th position and zeros elsewhere. Similarly,
let $\one_a\in\RR^{\ell m}$ be the vector that is all-ones on the action $a$ in all positions and zeros elsewhere. Finally,
let $\one\in\RR^{\ell m}$ be the all-ones vector. We also use $\mI_j=\diag(\one_j)$ to denote the diagonal matrix with all-ones
on the actions in the $j$-th position and zeros elsewhere.

\begin{proposition}
\label{prop:ranking:spectrum}
Consider the ranking setting where for each $x$ there is a set $A(x)$ such that $A_j(x)=A(x)$ and where all slates $\vs \in A(x)^\ell$ without repetitions
are legal.
Let $\unif$ denote the uniform logging policy over these slates.
If $\slatesize<\poolsize$, then $\brho_{\unif,x}=m\ell-\ell+1$ and
\[
  \mgamma_{\unif,x}^\dagger
  =
  \Parens{\frac{1}{\ell^2}-\frac{m-1}{m(m-\ell)}}\cdot\one\one^T
  +(m-1)\mI
  -\frac{m-1}{m}\sum_j\one_j\one_j^T
  +\frac{m-1}{m-\ell}\sum_a\one_a\one_a^T
\enspace,
\]
and for $\ell=m$, we have $\brho_{\unif,x}=m^2-2m+2$ and
\[
  \mgamma_{\unif,x}^\dagger
  =
  \frac{1}{m}\cdot\one\one^T
  +(m-1)\mI  - \frac{m-1}{m}\sum_j\one_j\one_j^T-\frac{m-1}{m}\sum_a\one_a\one_a^T
\enspace.
\]
For $\ell=m$, we have for any policy $\pi$, any $\vs\in S(x)$, and any $r\in[-1,1]$ that
\begin{equation}
\label{eq:tVpi:ranking}
  \qtarget^T\mgammaparam{\unif}{x}^\dagger r \one_\vs
  =r\cdot\Bracks{
  \sum_{j=1}^\ell \frac{\pi(s_j\given x)}{1/(m-1)}-m+2
  }
\enspace.
\end{equation}
\end{proposition}
\begin{proof}
Throughout the proof we will write $\mgamma$ instead of the more verbose $\mgammaparam{\unif}{x}$.
Note that for ranking and the uniform distribution we have
\[
  \mathnormal{\Gamma}(j,a;k,a')=
  \begin{cases}
    \frac{1}{m}&\text{if $j=k$ and $a=a'$}
\\
    \frac{1}{m(m-1)}&\text{if $j\ne k$ and $a\ne a'$}
\\
    0&\text{otherwise.}
  \end{cases}
\]
Thus, for any $\vz$
\begin{align}
\notag
\vz^T\mgamma\vz
 &= \sum_{j,a}\frac{z_{j,\action}^2}{\poolsize}
   + \frac{1}{\poolsize(\poolsize-1)}\sum_{j\ne k, a\ne a'}z_{j,\action}z_{k,\action'}
\\
\notag
 &= \frac{1}{\poolsize} \norm{\vz}_2^2
   + \frac{1}{\poolsize(\poolsize-1)}\Parens{
     (\vz^T\one)^2 - \sum_{j} (\vz^T\one_j)^2 - \sum_{\action} (\vz^T\one_\action)^2 + \norm{\vz}_2^2
   }
\\
\label{eq:rank:decomp:1}
 &=  \frac{1}{\poolsize(\poolsize-1)}\Parens{
     (\vz^T\one)^2 - \sum_{j} (\vz^T\one_j)^2 - \sum_{\action} (\vz^T\one_\action)^2 + m\norm{\vz}_2^2
   }
\enspace.
\end{align}
Let $\oneslots\in\RR^\ell$ and $\oneacts\in\RR^m$ be all-ones vectors in the respective spaces
and $\mIslots\in\RR^{\ell\times\ell}$ and $\mIacts\in\RR^{m\times m}$ be identity matrices in the respective spaces.
We can rewrite the quadratic form described by $\mgamma$ as
\begin{align}
\notag
&
 m(m-1)\mgamma
 =\one\one^T-\sum_j\one_j\one_j^T-\sum_a\one_a\one_a^T+m\mI
\\
\notag
&\quad{}
=(\oneslots\oneslotsT)\otimes(\oneacts\oneactsT)
  -\mIslots\otimes(\oneacts\oneactsT)
  -(\oneslots\oneslotsT)\otimes\mIacts
  +m\cdot\mIslots\otimes\mIacts
\notag
\\[6pt]
&\quad{}
=
\ell m\cdot\frac{\oneslots\oneslotsT}{\ell}\otimes\frac{\oneacts\oneactsT}{m}
  -m\cdot\mIslots\otimes\frac{\oneacts\oneactsT}{m}
  -\ell\cdot\frac{\oneslots\oneslotsT}{\ell}\otimes\mIacts
  +m\cdot\mIslots\otimes\mIacts
\notag
\\[6pt]
&\quad{}
=\ell (m-1)\cdot\frac{\oneslots\oneslotsT}{\ell}\otimes\frac{\oneacts\oneactsT}{m}
  -m\cdot\mIslots\otimes\Parens{\frac{\oneacts\oneactsT}{m}-\mIacts}
  -\ell\cdot\frac{\oneslots\oneslotsT}{\ell}\otimes\Parens{\mIacts-\frac{\oneacts\oneactsT}{m}}
\notag
\\[6pt]
\notag
&\quad{}
=
\ell (m-1)\cdot\frac{\oneslots\oneslotsT}{\ell}\otimes\frac{\oneacts\oneactsT}{m}
\\
\label{eq:decomp:kron}
&\qquad\qquad{}
  + m\cdot\Parens{\mIslots-\frac{\oneslots\oneslotsT}{\ell}} \otimes
    \Parens{\mIacts-\frac{\oneacts\oneactsT}{m}}
  +(m-\ell)\cdot\frac{\oneslots\oneslotsT}{\ell}\otimes\Parens{\mIacts-\frac{\oneacts\oneactsT}{m}}
.
\end{align}
Next, we would like to argue that \Eq{decomp:kron} is an eigendecomposition. For this, we just need to show
that each of the three Kronecker products in \Eq{decomp:kron} equals a projection matrix in $\RR^{\ell m}$,
and that the ranges of the projection matrices are orthogonal. The first property follows, because if
$\mP_1$ and $\mP_2$ are projection matrices then so is $\mP_1\otimes\mP_2$. The second property follows,
because for $\mP_1,\mP'_1$ (square of the same dimension) and $\mP_2,\mP'_2$ (square of the same dimension)
such that either ranges of $\mP_1$ and $\mP'_1$ are orthogonal or ranges of $\mP_2$ and $\mP'_2$ are orthogonal,
we obtain that the ranges of $\mP_1\otimes\mP_2$ and $\mP'_1\otimes\mP'_2$ are orthogonal.

Now we are ready to derive the pseudo-inverse. We distinguish two cases.

\paragraph{Case $\ell<m$:}
We directly invert the eigenvalues in \Eq{decomp:kron} to obtain
\begin{align*}
  \mgamma^\dagger
&=
  \frac{m}{\ell}\cdot\frac{\oneslots\oneslotsT}{\ell}\otimes\frac{\oneacts\oneactsT}{m}
  +(m-1)\cdot \Parens{\mIslots-\frac{\oneslots\oneslotsT}{\ell}} \otimes
    \Parens{\mIacts-\frac{\oneacts\oneactsT}{m}}
\\&\hspace{2.25in}{}
  +\frac{m-1}{1-\ell/m}\cdot\frac{\oneslots\oneslotsT}{\ell}\otimes\Parens{\mIacts-\frac{\oneacts\oneactsT}{m}}
\\[4pt]
&=
  \frac{1}{\ell^2}\cdot\one\one^T
  +(m-1)\cdot \Parens{\mIslots+\frac{\oneslots\oneslotsT}{m-\ell}} \otimes
    \Parens{\mIacts-\frac{\oneacts\oneactsT}{m}}
\\[4pt]
&=
  \Parens{\frac{1}{\ell^2}-\frac{m-1}{m(m-\ell)}}\cdot\one\one^T
  +(m-1)\mI
  -\frac{m-1}{m}\sum_j\one_j\one_j^T
  +\frac{m-1}{m-\ell}\sum_a\one_a\one_a^T
\enspace.
\end{align*}
Recall that \Eq{decomp:kron} involves $m(m-1)\mgamma$.
To obtain $\brho$, we again evaluate $\one_{\vs'}^T\mgamma^\dagger\one_{\vs}$ for any $\vs\in S(x)$. We write $A_\vs$ for the set of actions appearing on the slate $\vs$:
\begin{align}
\notag
\one_{\vs'}^T\mgamma^\dagger\one_{\vs}
&=
  \Parens{\frac{1}{\ell^2}-\frac{m-1}{m(m-\ell)}}\cdot(\one_{\vs'}^T\one)(\one^T\one_{\vs})
  +(m-1)\one_{\vs'}^T\one_{\vs}
  -\frac{m-1}{m}\sum_j(\one_{\vs'}^T\one_j)(\one_j^T\one_{\vs})
\\[-4pt]
\notag
&\hspace{2.25in}{}
  +\frac{m-1}{m-\ell}\sum_a(\one_{\vs'}^T\one_a)(\one_a^T\one_{\vs})
\\[4pt]
\notag
&=
  \Parens{\frac{1}{\ell^2}-\frac{m-1}{m(m-\ell)}}\cdot\ell^2
  +\sum_j\frac{\one\braces{s'_j=s_j}}{1/(m-1)}
  -\frac{m-1}{m}\cdot\ell
\\[-4pt]
\label{eq:ranking:1}
&\hspace{2.25in}{}
  +\frac{m-1}{m-\ell}\sum_a\one\braces{a\in A_{\vs'}}\one\braces{a\in A_\vs}
\\[4pt]
\notag
&=
  1-\frac{(m-1)(\ell^2+m\ell-\ell^2)}{m(m-\ell)}
  +\sum_j\frac{\one\braces{s'_j=s_j}}{1/(m-1)}
  +\frac{m-1}{m-\ell}\cdot\card{A_{\vs'}\cap A_{\vs}}
\\[4pt]
\notag
&=
  1-\frac{m-1}{m-\ell}\cdot\ell
  +\sum_j\frac{\one\braces{s'_j=s_j}}{1/(m-1)}
  +\frac{m-1}{m-\ell}\cdot\card{A_{\vs}\cap A_{\vs'}}
\enspace,
\end{align}
where \Eq{ranking:1} follows because $\one^T\one_\vs=\ell$ and $\one_j^T\one_\vs=1$ for any valid slate $\vs$.
By setting $\vs'=\vs$, we obtain $\brho=1+\ell(m-1)=m\ell-\ell+1$.

\paragraph{Case $\ell=m$:}
Again, we directly invert the eigenvalues in \Eq{decomp:kron} to obtain
\begin{align*}
  \mgamma^\dagger
&=
  \frac{1}{\ell^2}\cdot\one\one^T
  +(m-1)\cdot \Parens{\mIslots-\frac{\oneslots\oneslotsT}{\ell}} \otimes
    \Parens{\mIacts-\frac{\oneacts\oneactsT}{m}}
\\
&=
  \frac{1}{m}\cdot\one\one^T
  +(m-1)\mI  - \frac{m-1}{m}\sum_j\one_j\one_j^T-\frac{m-1}{m}\sum_a\one_a\one_a^T
\enspace.
\end{align*}
We finish the theorem by evaluating $\one_{\vs'}^T\mgamma^\dagger\one_{\vs}$:
\begin{align*}
\one_{\vs'}^T\mgamma^\dagger\one_{\vs}
&=
  \frac{1}{m}\cdot(\one_{\vs'}^T\one)(\one^T\one_{\vs})
  +(m-1)\one_{\vs'}^T\one_{\vs}
  -\frac{m-1}{m}\sum_j(\one_{\vs'}^T\one_j)(\one_j^T\one_{\vs})
\\
&\hspace{2.25in}{}
  -\frac{m-1}{m}\sum_a(\one_{\vs'}^T\one_a)(\one_a^T\one_{\vs})
\\
&=
  \frac{1}{m}\cdot m^2
  +\sum_j\frac{\one\braces{s'_j=s_j}}{1/(m-1)}
  -\frac{m-1}{m}\cdot m
  -\frac{m-1}{m}\cdot m
\\
&=
  \sum_j\frac{\one\braces{s'_j=s_j}}{1/(m-1)}
  -m+2
\enspace.
\end{align*}
We obtain $\brho=m^2-2m+2$ by setting $\vs'=\vs$ and \Eq{tVpi:ranking} by taking an expectation over $\vs'\sim\pi(\cdot\given x)$.
\end{proof}

\subsection{Proof of Corollary~\ref{prop:kappa}}

We need one last technical result in order to establish the
proposition.

\begin{claim}
\label{claim:frac:inverse}
Let $\mA,\mB$ be two symmetric positive semi-definite matrices with $\Null(\mA) \subseteq \Null(\mB)$. Then
\[
  \max_{\vz\perp\Null(\mB),\,\vz\ne\zero}\;\frac{\vz^T\mB^\dagger\vz}{\vz^{T\!}\mA^\dagger\vz}
\le
  \max_{\vz\perp\Null(\mB),\,\vz\ne\zero}\;\frac{\vz^{T\!}\mA\vz}{\vz^T\mB\vz}
\enspace.
\]
\end{claim}

We now provide the proof of Corollary~\ref{prop:kappa}, following which we will
prove Claim~\ref{claim:frac:inverse}.

\begin{proof}[Proof of Corollary~\ref{prop:kappa}]
The corollary follows by \Prop{brho:bound}. The key step is to bound
$\brho_{\mu_\eps, x}$, for which we invoke
Claim~\ref{claim:frac:inverse}. Specifically, we apply the claim with
$\mA = \mgammaparam{\unif}{x}$ and $\mB =
\mgammaparam{\mu_\eps}{x}$. Since
$$\mgamma_{\mu_\eps,x} = (1-\eps)\mgamma_{\mu,x} +
\eps\mgamma_{\unif,x} = (1-\eps)\EE_{\vs\sim
  \mu(\cdot\given x)}[\one_\vs\one_\vs^T] + \eps\EE_{\vs\sim
  \unif(\cdot\given x)}[\one_\vs\one_\vs^T],$$
we observe that $\Null(\mgamma_{\unif,x}) =
\Null(\mgamma_{\mu_\eps,x})$, because the support of $\mu(\cdot\given x)$ is always
included in the support of $\unif(\cdot\given x)$.
Now we can invoke Claim~\ref{claim:frac:inverse} with these
choices to see that

\begin{align*}
  \brho_{\mu_\eps,x}
&= \sup_{\vs \in \supp \mu_\eps(\cdot | x)}
  \one_\vs^T\mgammaparam{\mu_\eps}{x}^\dagger\one_\vs
\\
  &\leq \sup_{\vs \in \supp \unif(\cdot | x)}
  \one_\vs^T\mgammaparam{\unif}{x}^\dagger\one_\vs \,
    \sup_{\vs \in \supp \mu_\eps(\cdot | x)}
    \frac{\one_\vs^T\mgammaparam{\mu_\eps}{x}^\dagger\one_\vs}
         {\one_\vs^T\mgammaparam{\unif}{x}^\dagger\one_\vs}
\\
  &\leq \brho_{\unif,x}\,
    \max_{\vz \perp \Null(\mgammaparam{\mu_\eps}{x}), \vz \ne \zero}
    \frac{\vz^T\mgammaparam{\mu_\eps}{x}^\dagger\vz}
         {\vz^T\mgammaparam{\unif}{x}^\dagger\vz}
\\
  &\leq \brho_{\unif,x}\,
  \max_{\vz \perp \Null(\mgammaparam{\mu_\eps}{x}), \vz \ne \zero}
  \frac{\vz^T\mgammaparam{\unif}{x}\vz}
       {\vz^T\mgammaparam{\mu_\eps}{x}\vz}
\\
  &\leq \brho_{\unif,x}\, \frac{\vz^T
    \mgammaparam{\unif}{x}\vz}{\eps \vz^T
    \mgammaparam{\unif}{x}\vz}\\
  &= \frac{\brho_{\unif,x}}{\eps}
\end{align*}

For the product slate space, using \Eq{ss':product}, which was proved within the proof of \Prop{product}, we have
\begin{align*}
\brho_{\unif,x}
&=
\sup_{\vs\in\supp\unif(\cdot\given x)} \one_\vs^T\mgammaparam{\unif}{x}^\dagger\one_\vs^{\vphantom{T}}
\\
&=
\sup_{\vs\in\supp\unif(\cdot\given x)} \Bracks{
  \sum_{j=1}^\ell \frac{1}{1/m}
  -\ell+1
  }
\\
&=\ell m-\ell+1
\enspace.
\end{align*}
For the ranking slate space, using \Prop{ranking:spectrum}, we also have
$\brho_{\unif,x}=\Ocal(\ell m)$, so for both the product slate space
and ranking slate space, we obtain $\brho_{\mu_\eps,x}=\Ocal(\ell m/\eps)$. Finally, plugging this upper
bound and \Prop{brho:bound} into the statement of
Theorem~\ref{thm:deviation:gen} completes the proof.
\end{proof}

We finally prove Claim~\ref{claim:frac:inverse}.

\begin{proof}[Proof of Claim~\ref{claim:frac:inverse}]
Let $\mU$ be the square root of matrix $\mA$, i.e., $\mU$ is a
symmetric positive semidefinite matrix with the same eigenvectors as
$\mA$, but with eigenvalues that are square root of the corresponding
eigenvalues of $\mA$. Similarly, let $\mV$ be the square root of
matrix $\mB$. Thus, we have $\mA=\mU\mU$ and
$\mA^\dagger=\mU^\dagger\mU^\dagger$ and similarly for $\mB$ and
$\mV$. Let $\mPi_\mA = \mU^\dagger\mU = \mU\mU^\dagger$ denote the
projection onto the range of $\mA$ and $\mPi_\mB$ denote the
projection onto the range of $\mB$. Since
$\Null(\mA)\subseteq\Null(\mB)$, we have
$\Range(\mA)\supseteq\Range(\mB)$.  We prove the claim as follows:
\begin{align}
\label{eq:inv:1}
  \max_{\vz\perp\Null(\mB),\,\vz\ne\zero}\;\frac{\vz^T\mB^\dagger\vz}{\vz^{T\!}\mA^\dagger\vz}
&=
  \max_{\vz\perp\Null(\mB),\,\vz\ne\zero}\;\frac{\vz^T\;\mU^\dagger\mU\;\mB^\dagger\;\mU\mU^\dagger\;\vz}{\vz^{T}\;\mU^\dagger\mU^\dagger\;\vz}
\\
\label{eq:inv:2}
&\le
  \max_{\vy\ne\zero}\;\frac{\vy^T\mU\;\mB^\dagger\;\mU\vy}{\vy^{T}\vy}
\\
\label{eq:inv:3}
&=
  \max_{\vy\ne\zero}\;\frac{\vy^T\mU\;\mV^\dagger\mV^\dagger\;\mU\vy}{\vy^{T}\vy}
=
  \max_{\vy:\:\norm{\vy}_2=1}\;\norm{\mV^\dagger\mU\vy}_2^2
\\
\label{eq:inv:4}
&=
  \max_{\vy:\:\norm{\vy}_2=1}\;\norm{\mU\mV^\dagger\vy}_2^2
\\
\notag
&=
  \max_{\vy\ne\zero}\;\frac{\vy^T\mV^\dagger\mU\;\mU\mV^\dagger\vy}{\vy^{T}\vy}
\\
\label{eq:inv:5}
&=
  \max_{\vy\perp\Null(\mB),\,\vy\ne\zero}\;
  \frac{\vy^T\;\mV^\dagger\mA\mV^\dagger\;\vy}{\vy^{T}\vy}
\\
\label{eq:inv:6}
&=
  \max_{\vz\perp\Null(\mB),\,\vz\ne\zero}\;
  \frac{\vz^T\mV\;\mV^\dagger\mA\mV^\dagger\;\mV\vz}{\vz^{T}\;\mV\mV\;\vz}
\\
\label{eq:inv:7}
&=
  \max_{\vz\perp\Null(\mB),\,\vz\ne\zero}\;
  \frac{\vz^T\mA\vz}{\vz^{T}\mB\vz}
\enspace.
\end{align}
In \Eq{inv:1} we substitute $\mU^\dagger\mU^\dagger=\mA^\dagger$ and also use the fact that $\mU\mU^\dagger=\mPi_\mA$ and $\mPi_\mA\vz=\vz$ because $\vz\in\Range(\mB)\subseteq\Range(\mA)$.
\Eq{inv:2} is obtained by substituting $\vy=\mU^\dagger\vz$ and relaxing
the maximization to be over $\vy\ne\zero$.
In \Eq{inv:3} we substitute $\mV^\dagger\mV^\dagger=\mB^\dagger$.
In \Eq{inv:4} we use the fact that the operator norm of a matrix and its transpose are equal.
In \Eq{inv:5} we substitute $\mA=\mU\mU$ and note that it suffices to consider $\vy\perp\Null(\mB)$ because $\Null(\mV^\dagger\mA\mV^\dagger)=\Null(\mB)$.
In \Eq{inv:6} we use the fact that $\vz\mapsto\mV\vz$ is a bijection on $\Range(\mB)$, which is an orthogonal complement to $\Null(\mB)$, so we can substitute $\mV\vz=\vy$. Finally, in \Eq{inv:7} we substitute $\mB=\mV\mV$ and use the fact that $\mV^\dagger\mV=\mPi_\mB$ and $\mPi_\mB\vz=\vz$ because $\vz\in\Range(\mB)$.
\end{proof}

\clearpage
\section{Supplementary plots for off-policy evaluation on semi-synthetic data} \label{app:approach_family}
We experimented with several configurations of slate spaces, logging and target policies, and whole-page metrics in the semi-synthetic evaluation setup.
This section details the plots for all configurations. The key parameters were:
\begin{enumerate}
\item Metric: NDCG or ERR. NDCG satisfies the linearity assumption, while ERR does not.
\item Slate space: $(m,l) = (100,10)$ or $(10,5)$.
\item Logging policy: Unif, $\lassoTitle$, $\treeTitle$
\item Target policy: $\lassoBody$, $\treeBody$
\item Temperature $\alpha$: Uniform, Slightly peaked, Very peaked. Uniform corresponds to $\alpha=0$. For the small slate spaces with $(m,l)=(10,5)$, $\alpha=1.0$ creates a slightly peaked logging distribution, while $\alpha=2.0$ creates a severely peaked logging distribution.
For the larger slate spaces with $(m,l)=(100,10)$, $\alpha=0.5$ is moderately peaked while $\alpha=1.0$ is severely peaked.
\end{enumerate}

The plots in Figures~\ref{fig:plots:first}--\ref{fig:plots:last} detail the top row of Figure~\ref{fig:synth_expt1} for all combinations of these parameters.

\begin{figure*}
\begin{tabular}{cc}
\includegraphics[width=0.5\textwidth]{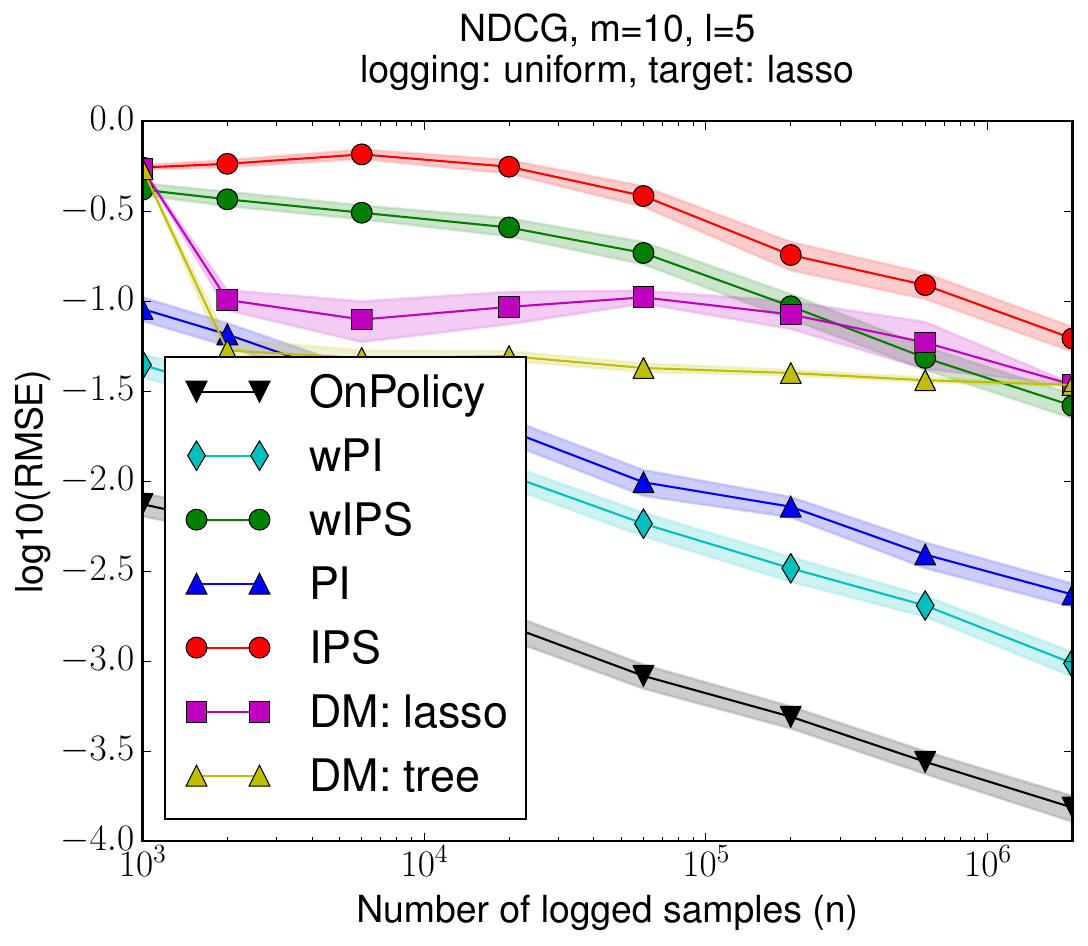}&
\includegraphics[width=0.5\textwidth]{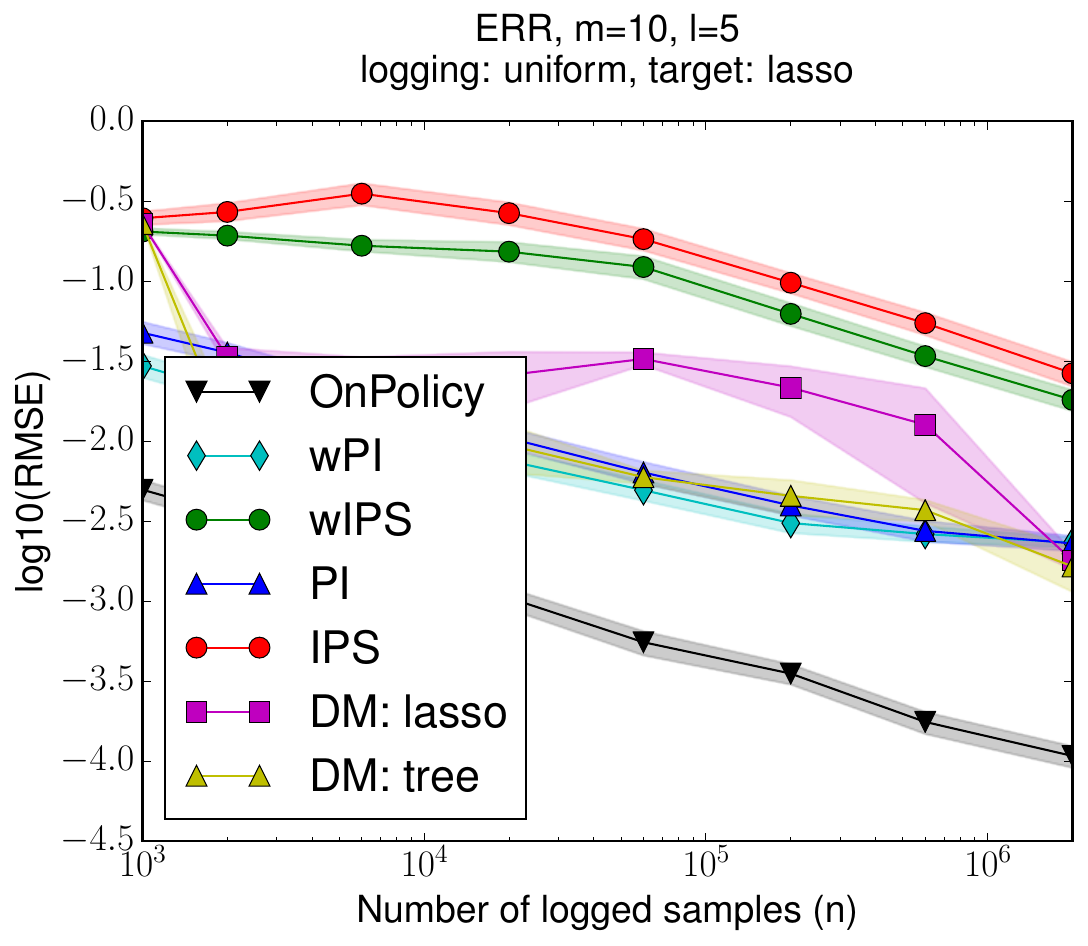}\\
\includegraphics[width=0.5\textwidth]{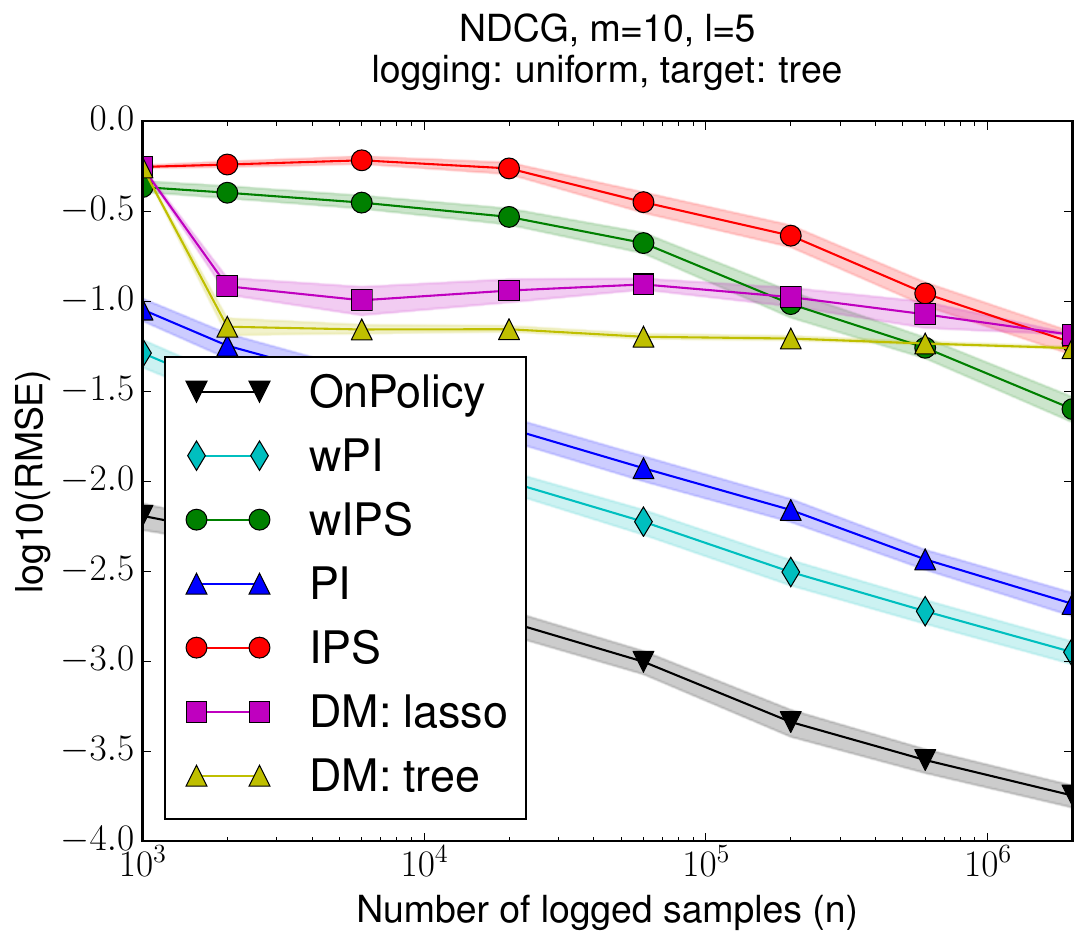}&
\includegraphics[width=0.5\textwidth]{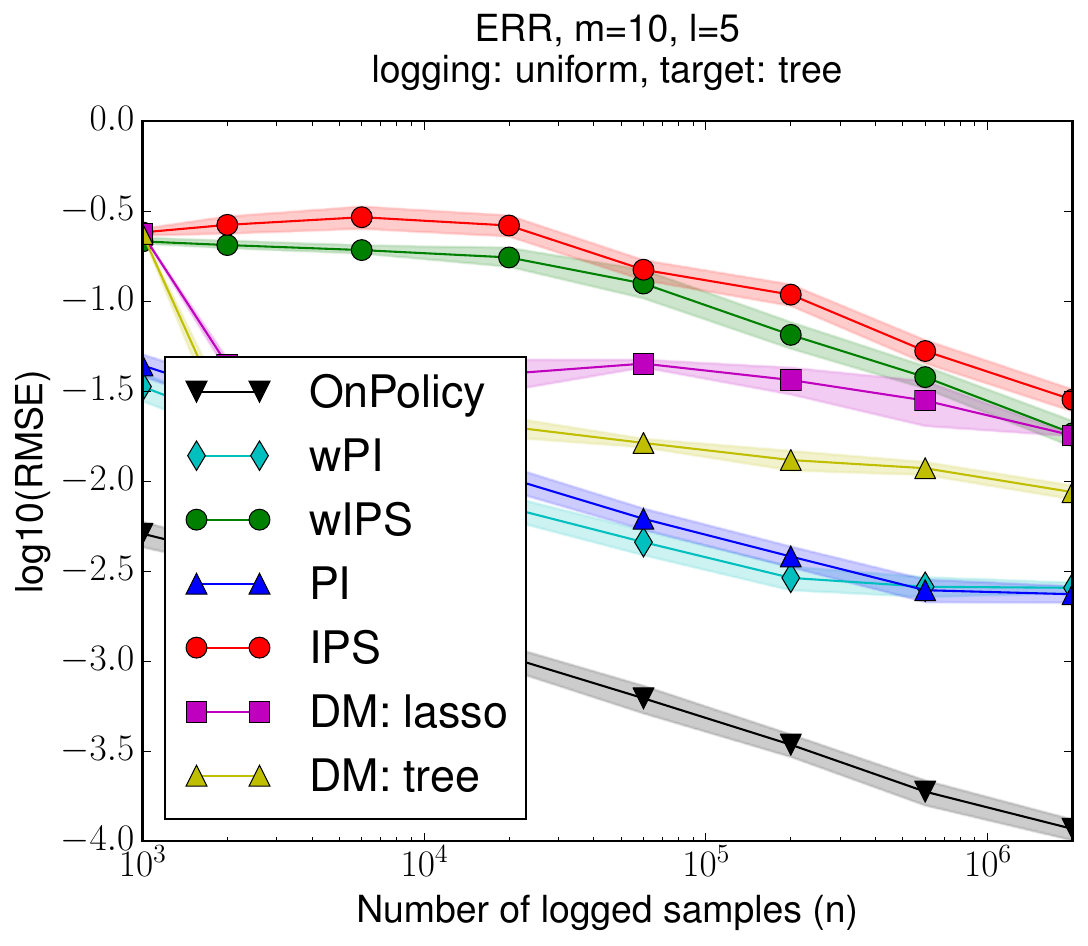}
\end{tabular}
\caption{RMSE of value estimators for an increasing logged dataset
  under a uniform logging policy with slate space $(10,5)$. Target is
  $\lassoBody$ (top panel) and $\treeBody$ (bottom panel). Metrics are
  NDCG (left) and ERR (right).}
\label{fig:plots:first}
\end{figure*}

\begin{figure*}
\begin{tabular}{cc}
\includegraphics[width=0.5\textwidth]{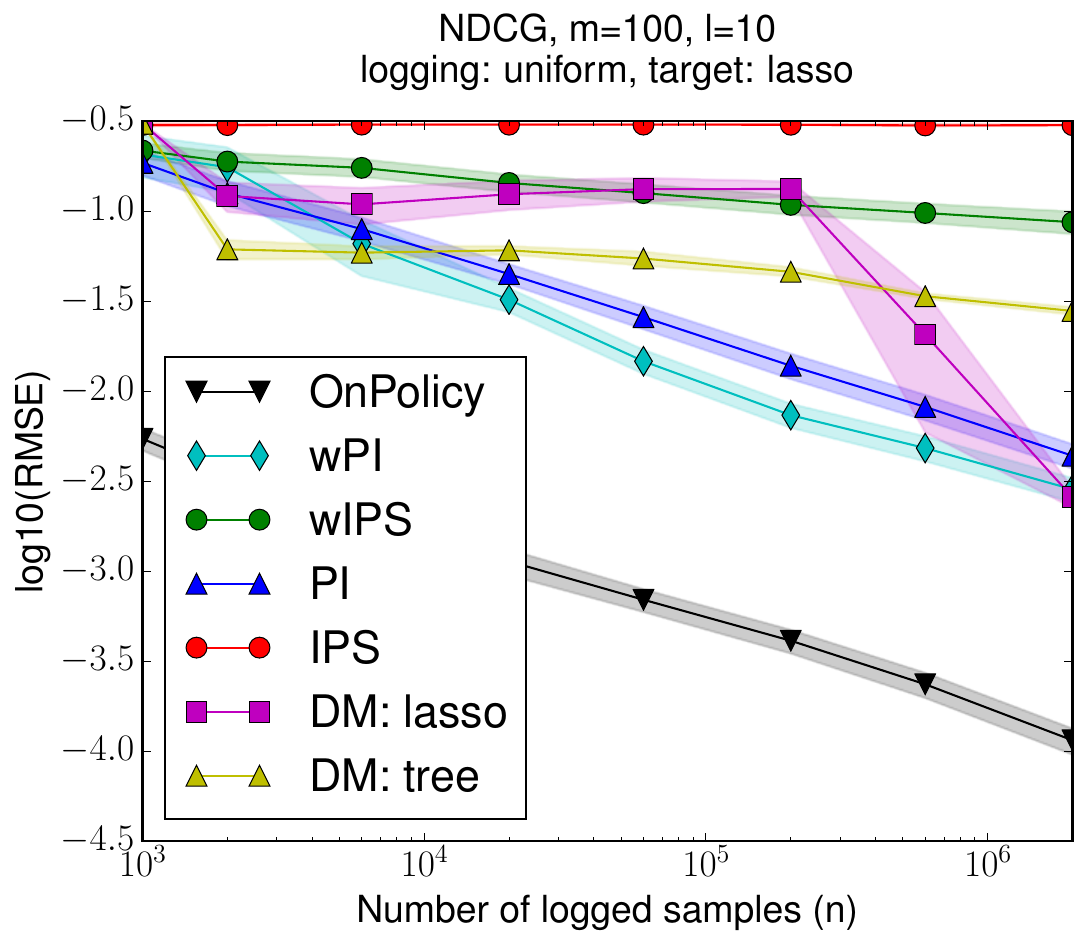}&
\includegraphics[width=0.5\textwidth]{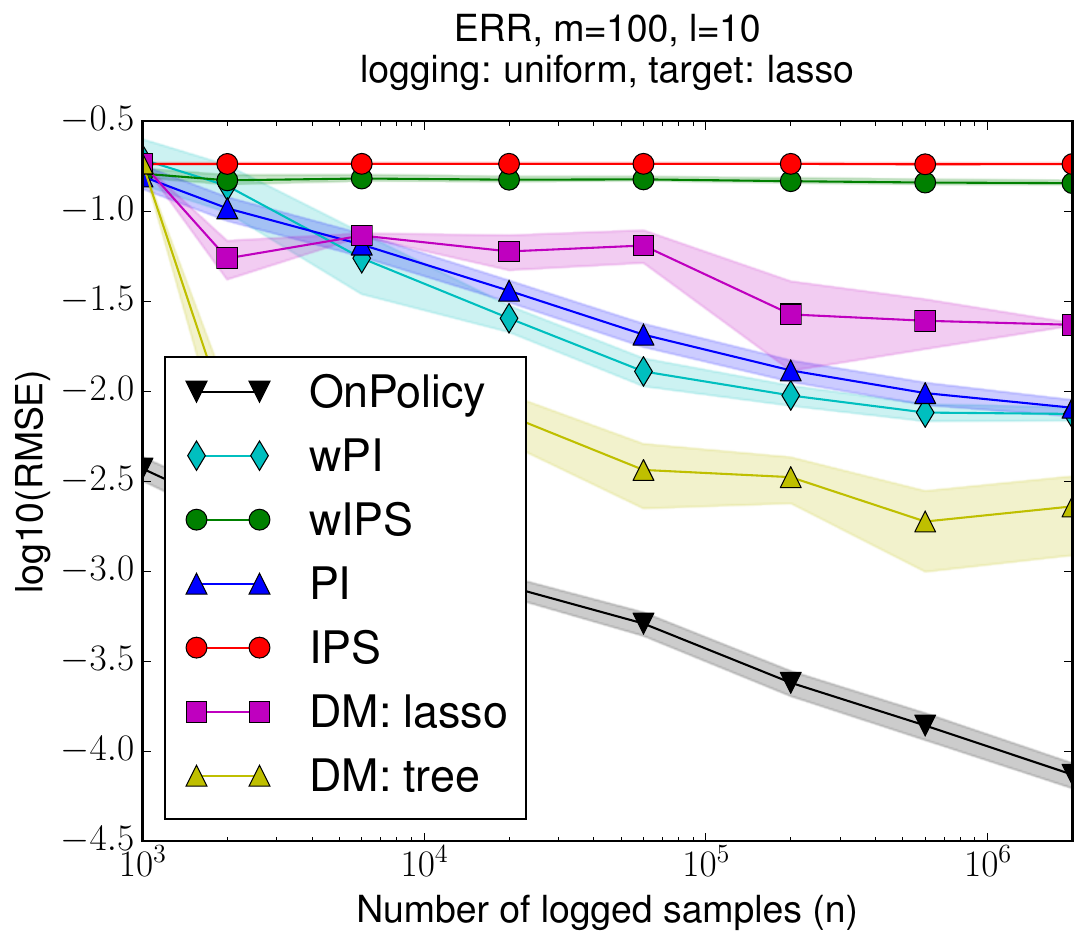}\\
\includegraphics[width=0.5\textwidth]{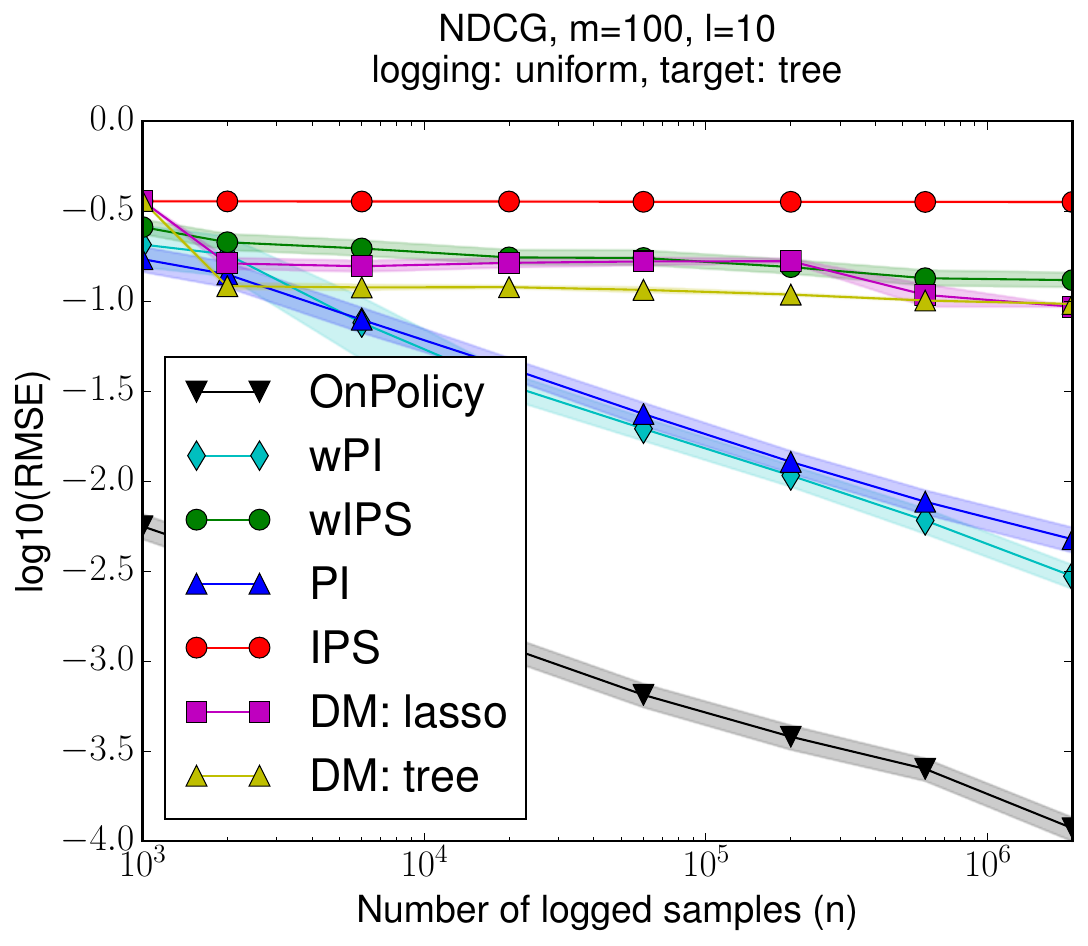}&
\includegraphics[width=0.5\textwidth]{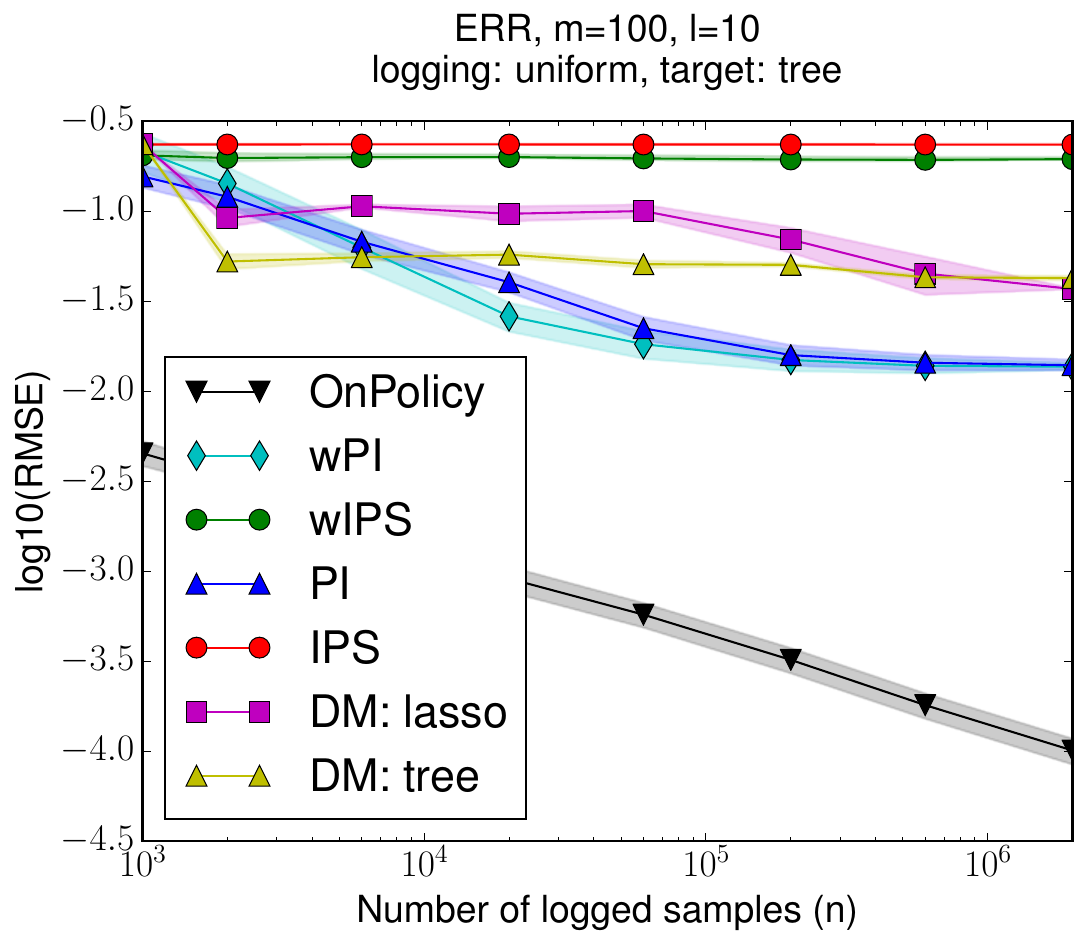}
\end{tabular}
\caption{RMSE of value estimators for an increasing logged dataset
  under a uniform logging policy with slate space $(100,10)$. Target is
  $\lassoBody$ (top panel) and $\treeBody$ (bottom panel). Metrics are
  NDCG (left) and ERR (right).}
\end{figure*}

\begin{figure*}
\begin{tabular}{cc}
\includegraphics[width=0.5\textwidth]{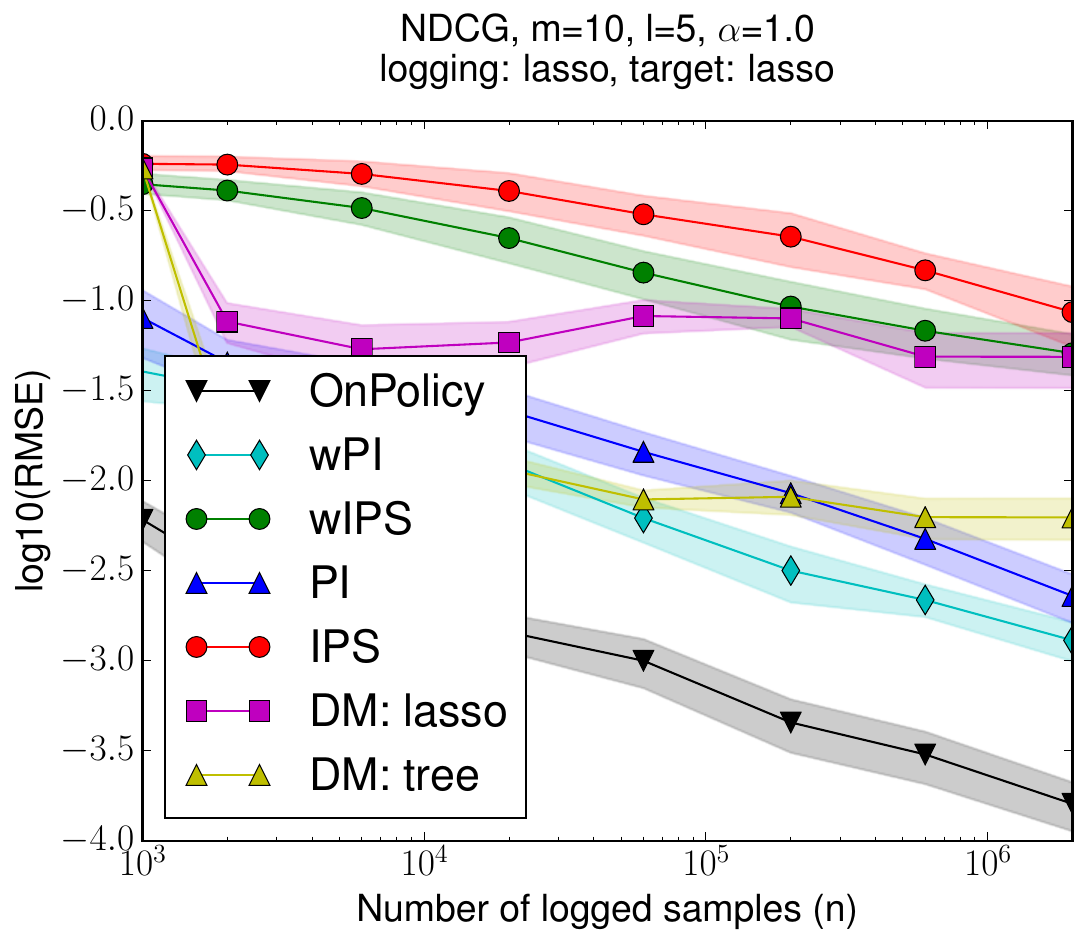}&
\includegraphics[width=0.5\textwidth]{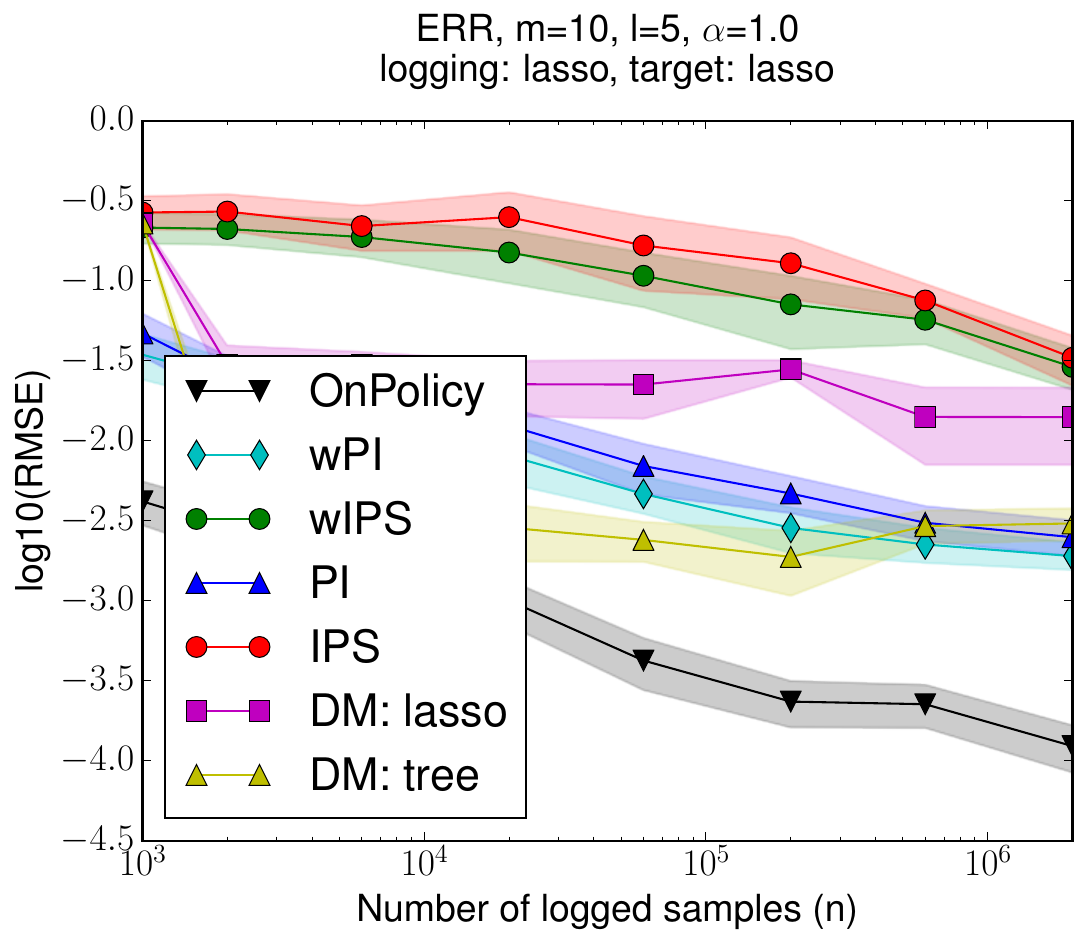}\\
\includegraphics[width=0.5\textwidth]{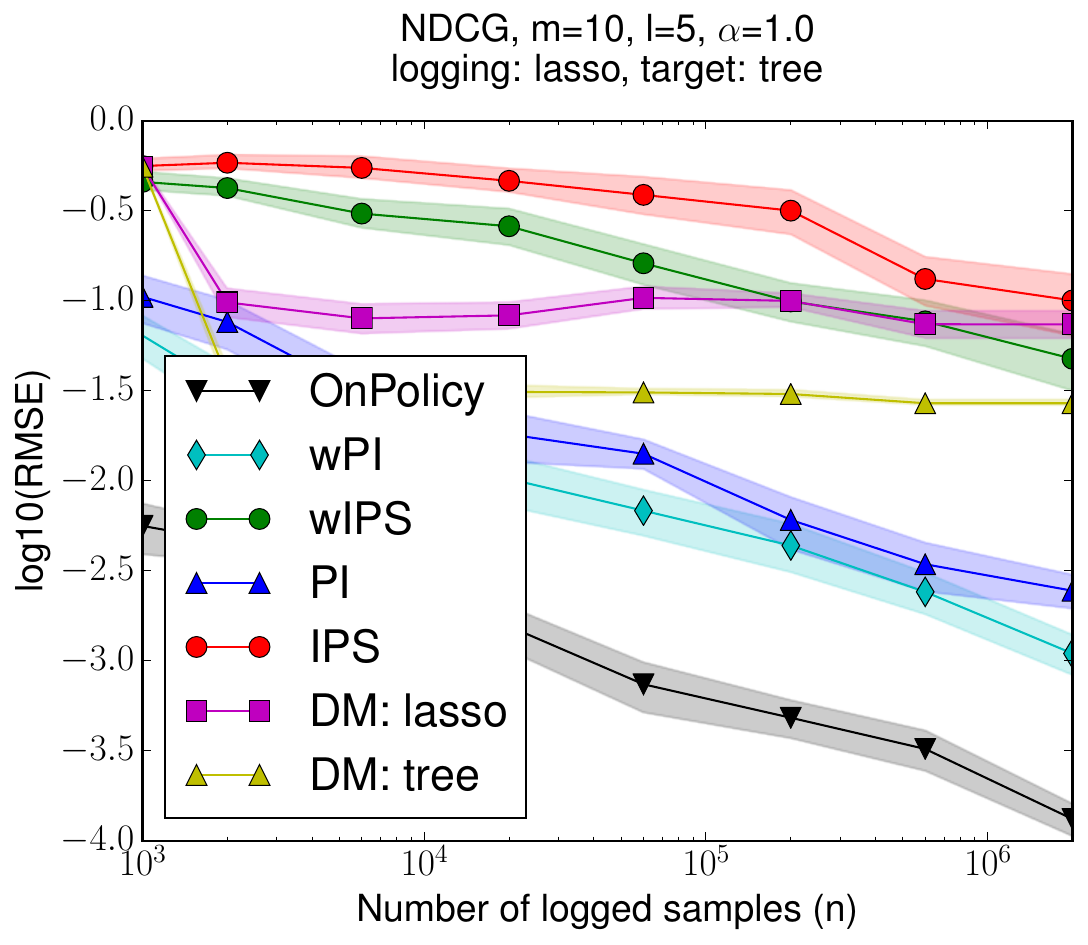}&
\includegraphics[width=0.5\textwidth]{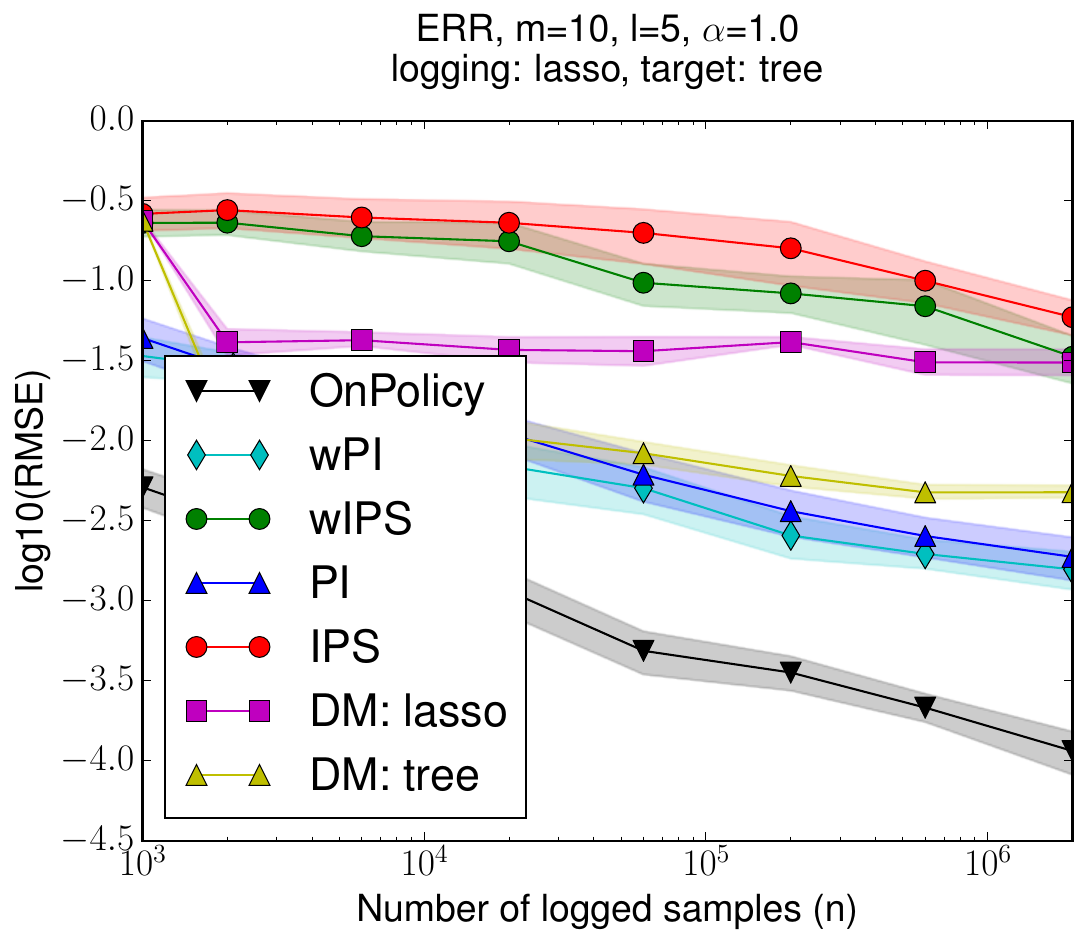}
\end{tabular}
\caption{RMSE of value estimators for an increasing logged dataset
  under a moderately peaked logging policy ($\lassoTitle, \alpha=1.0$)
  with slate space $(10,5)$. Target is $\lassoBody$ (top panel) and
  $\treeBody$ (bottom panel). Metrics are NDCG (left) and ERR
  (right).}
\end{figure*}

\begin{figure*}
\begin{tabular}{cc}
\includegraphics[width=0.5\textwidth]{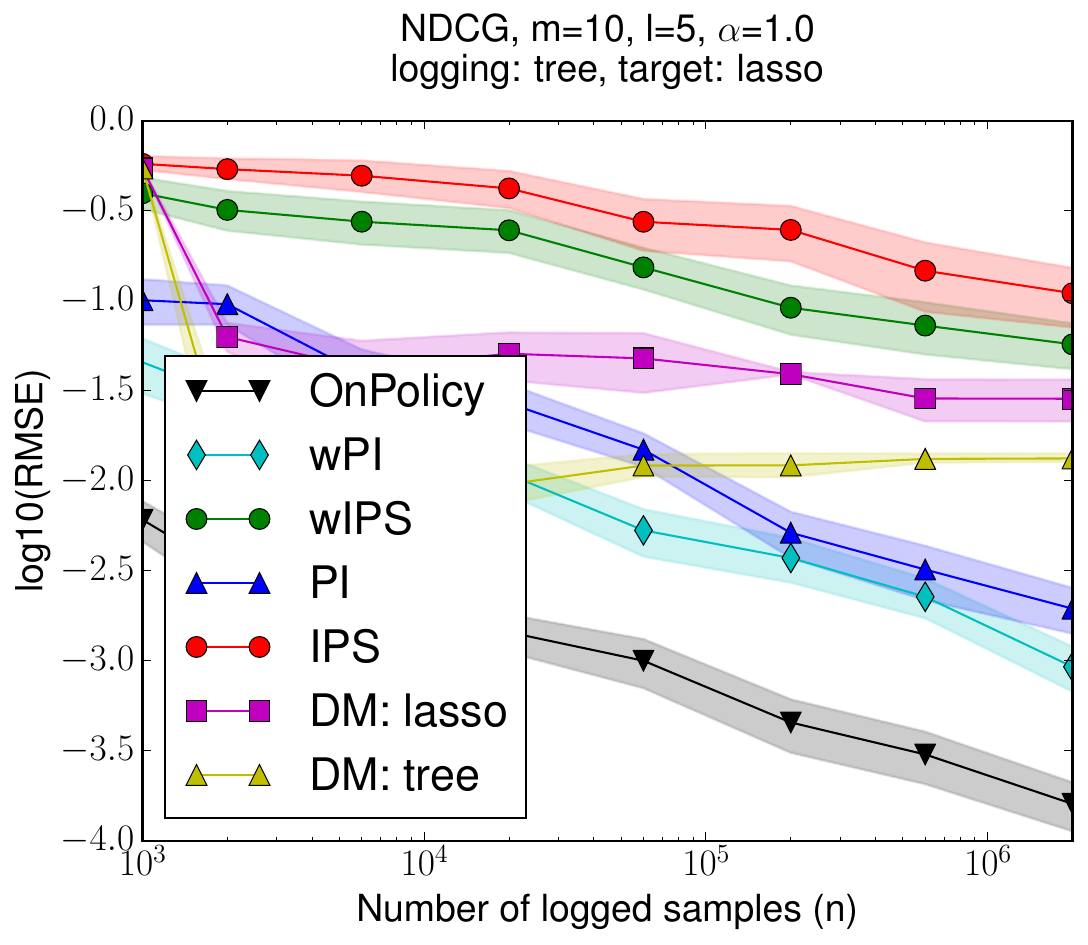}&
\includegraphics[width=0.5\textwidth]{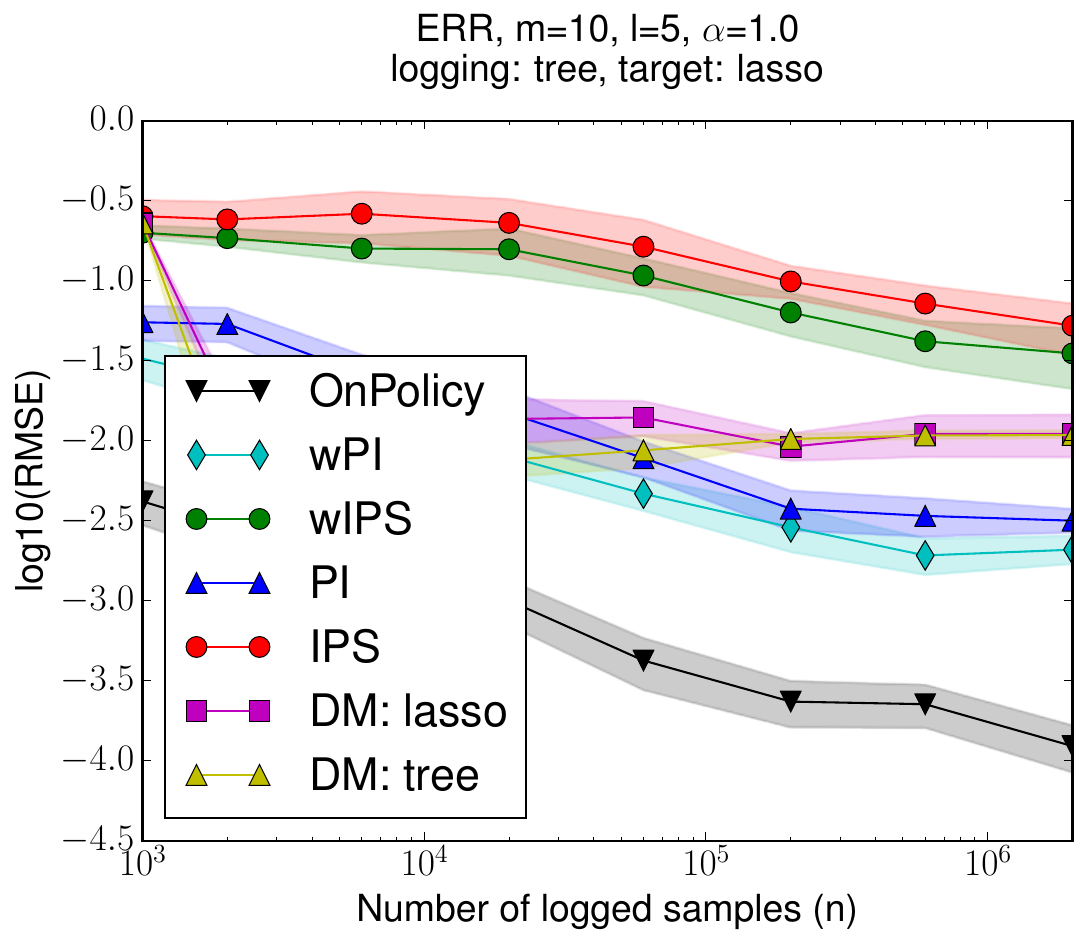}\\
\includegraphics[width=0.5\textwidth]{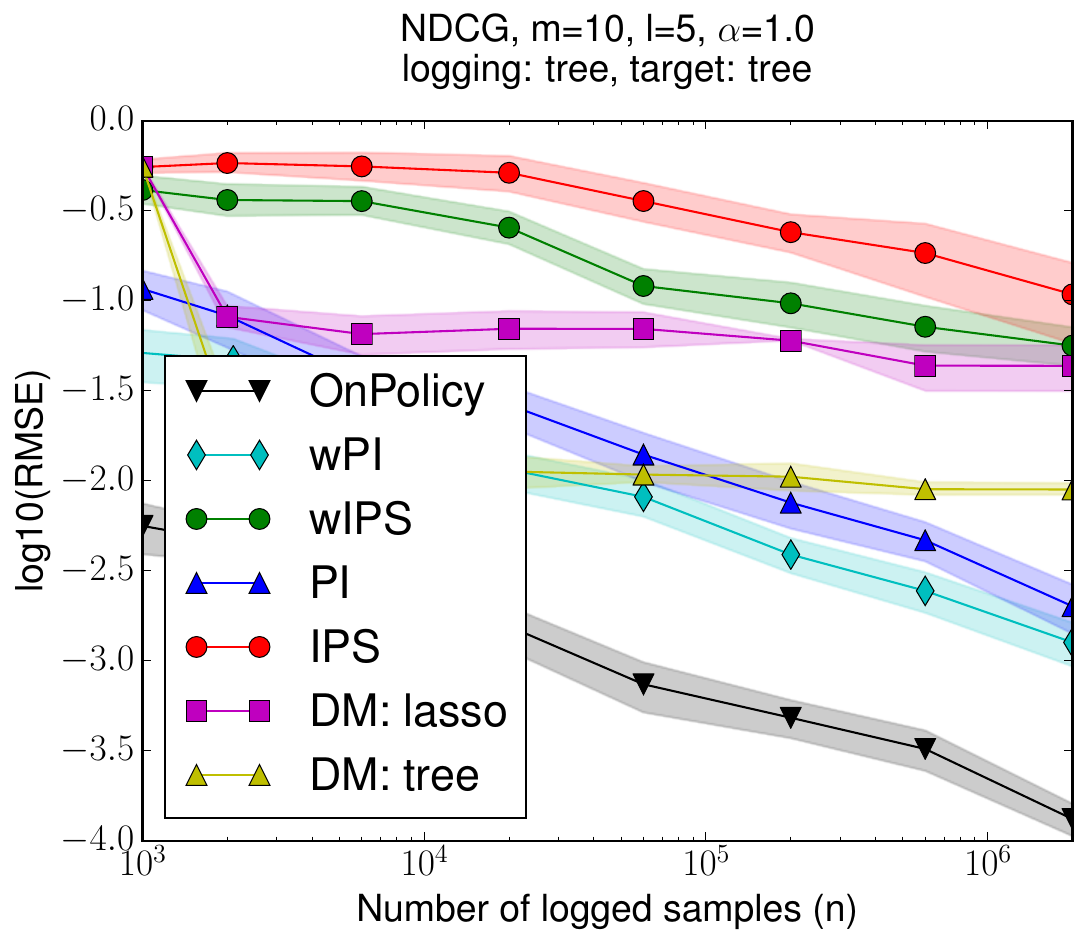}&
\includegraphics[width=0.5\textwidth]{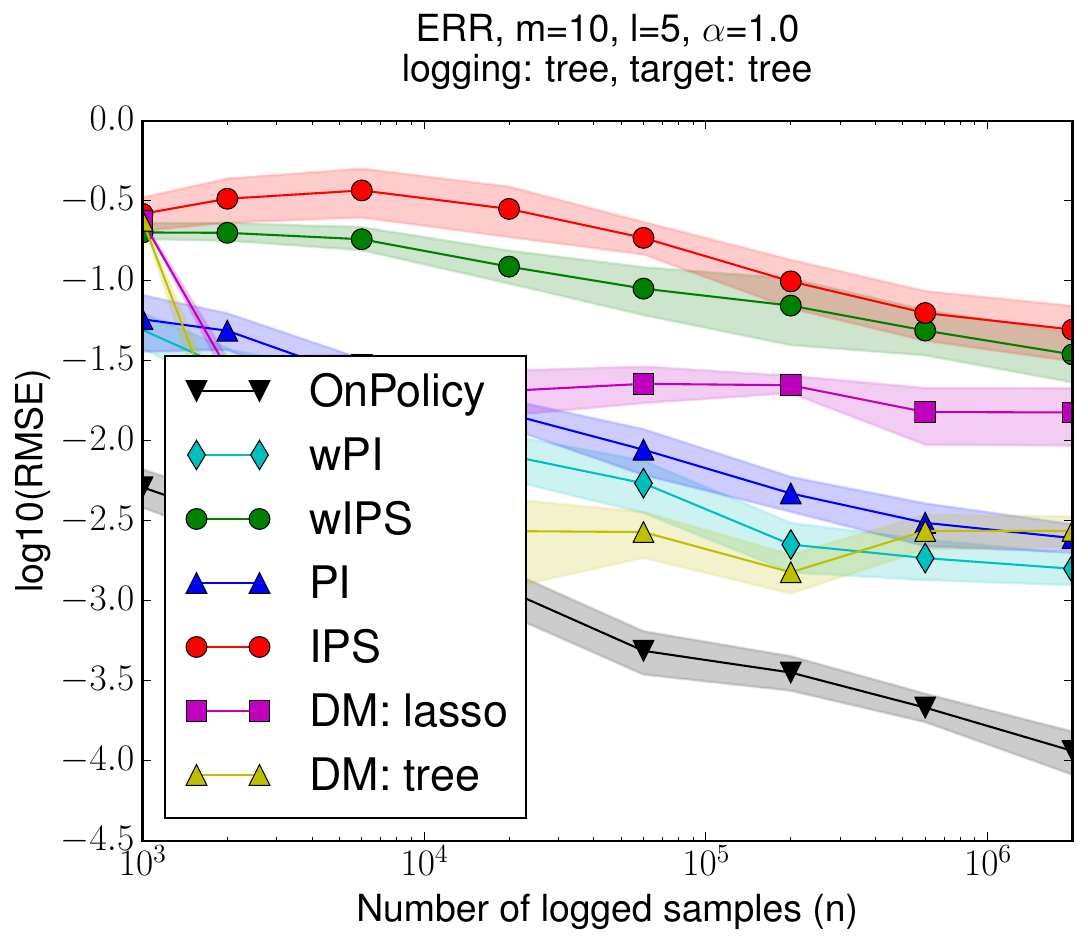}
\end{tabular}
\caption{RMSE of value estimators for an increasing logged dataset
  under a moderately peaked logging policy ($\treeTitle, \alpha=1.0$)
  with slate space $(10,5)$. Target is $\lassoBody$ (top panel) and
  $\treeBody$ (bottom panel). Metrics are NDCG (left) and ERR
  (right).}
\end{figure*}

\begin{figure*}
\begin{tabular}{cc}
\includegraphics[width=0.5\textwidth]{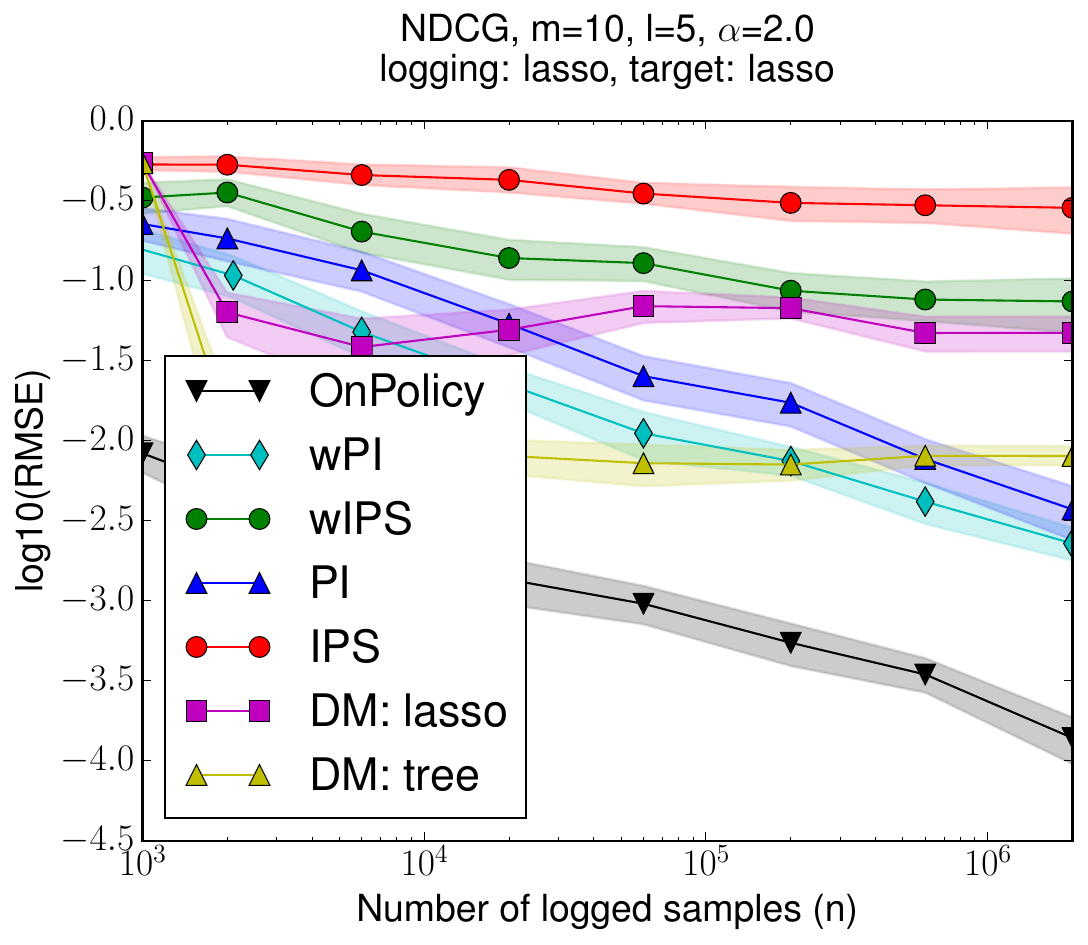}&
\includegraphics[width=0.5\textwidth]{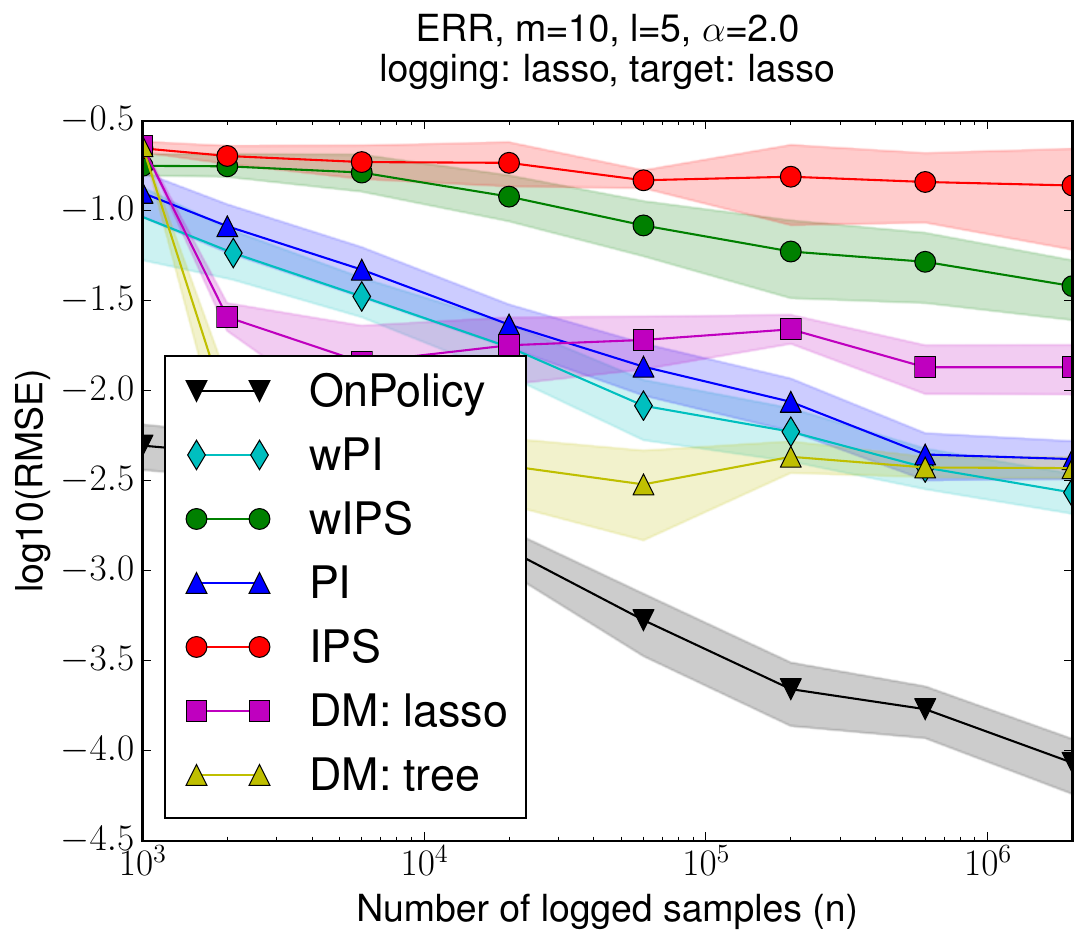}\\
\includegraphics[width=0.5\textwidth]{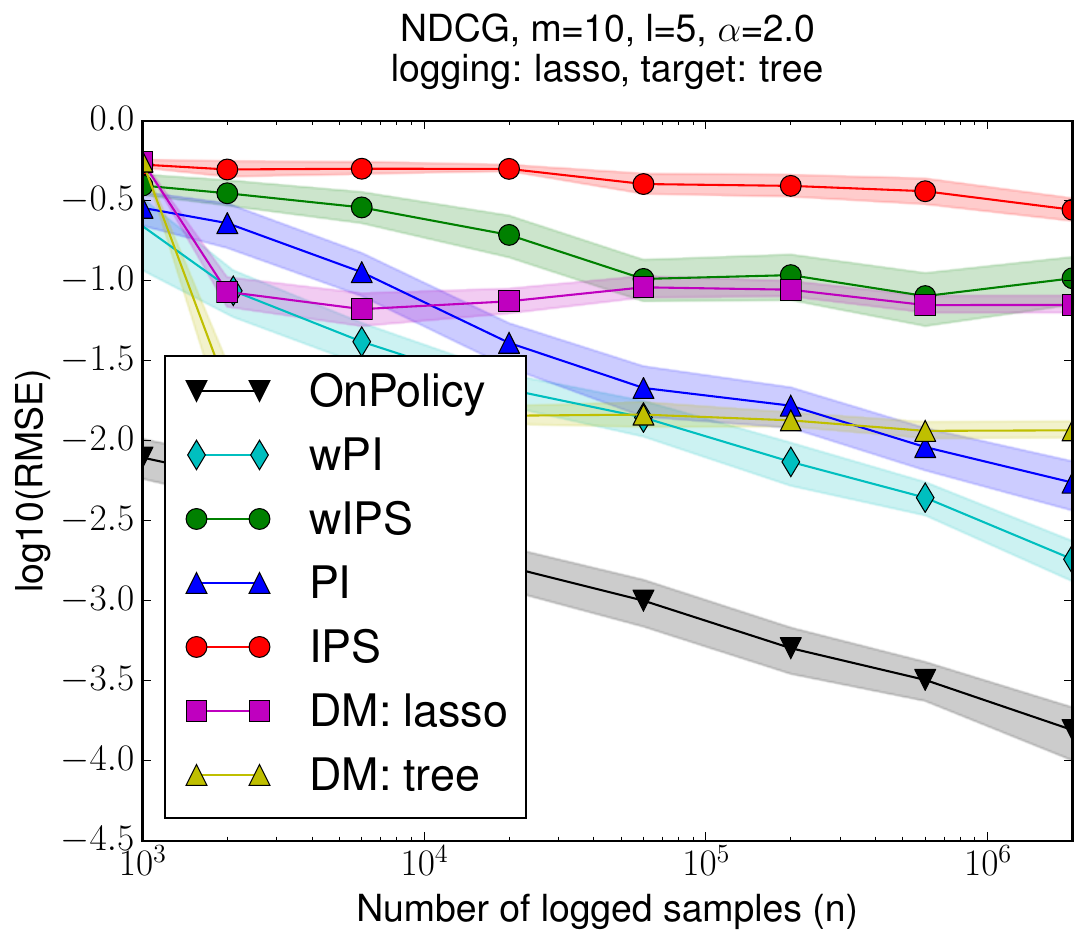}&
\includegraphics[width=0.5\textwidth]{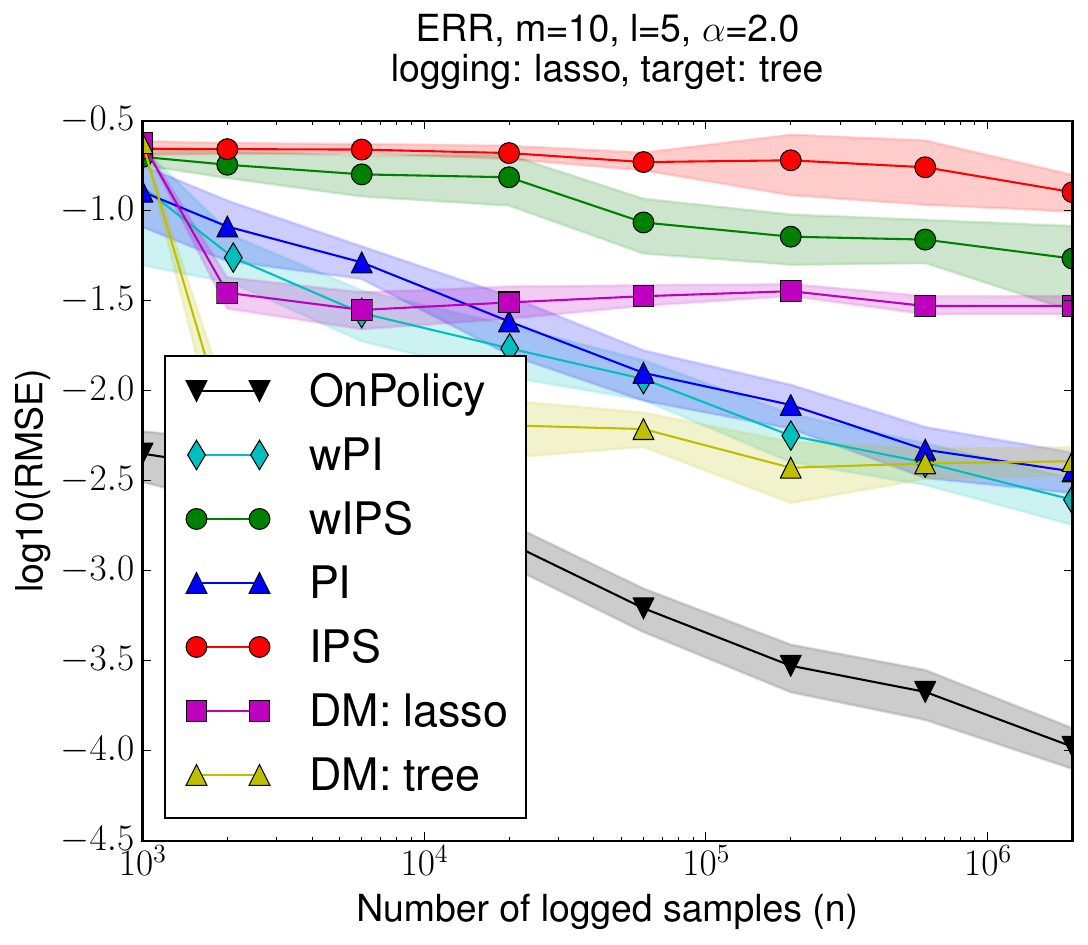}
\end{tabular}
\caption{RMSE of value estimators for an increasing logged dataset
  under a severely peaked logging policy ($\lassoTitle, \alpha=2.0$)
  with slate space $(10,5)$. Target is $\lassoBody$ (top panel) and
  $\treeBody$ (bottom panel). Metrics are NDCG (left) and ERR
  (right).}
\end{figure*}

\begin{figure*}
\begin{tabular}{cc}
\includegraphics[width=0.5\textwidth]{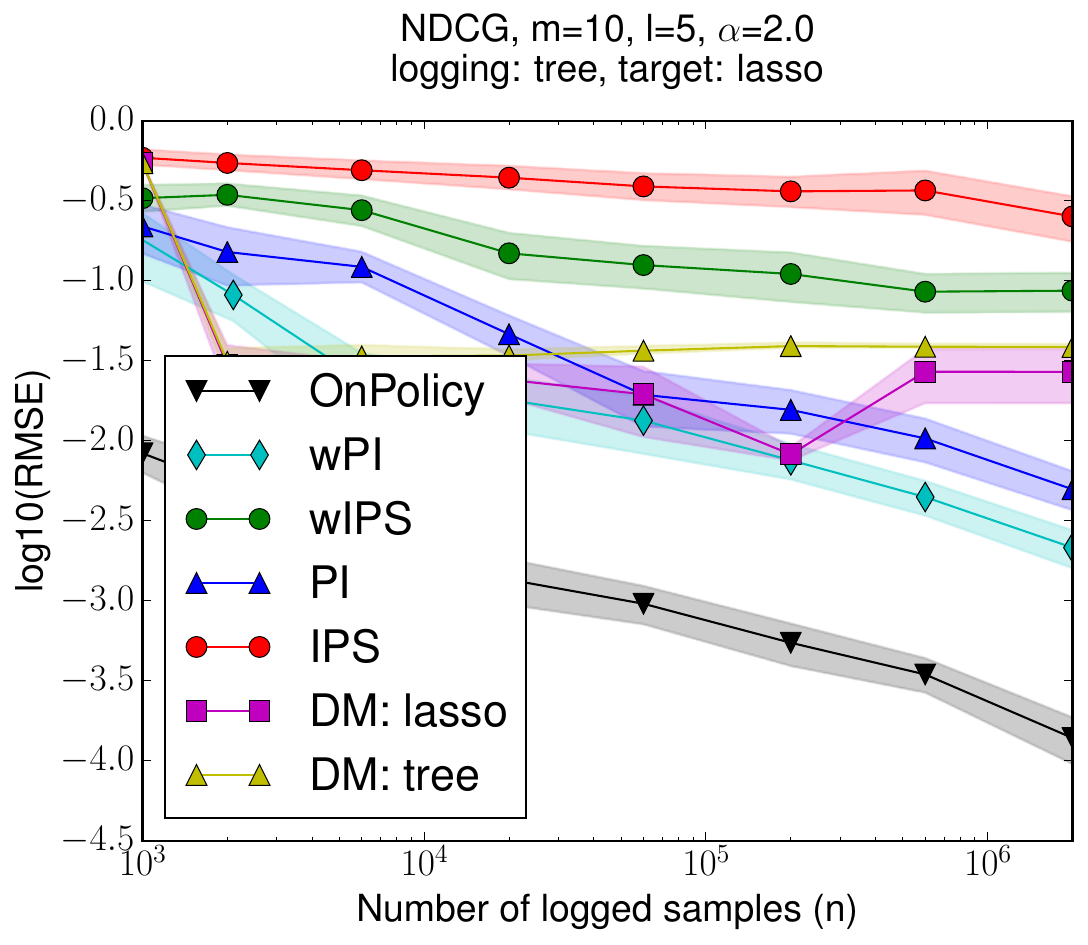}&
\includegraphics[width=0.5\textwidth]{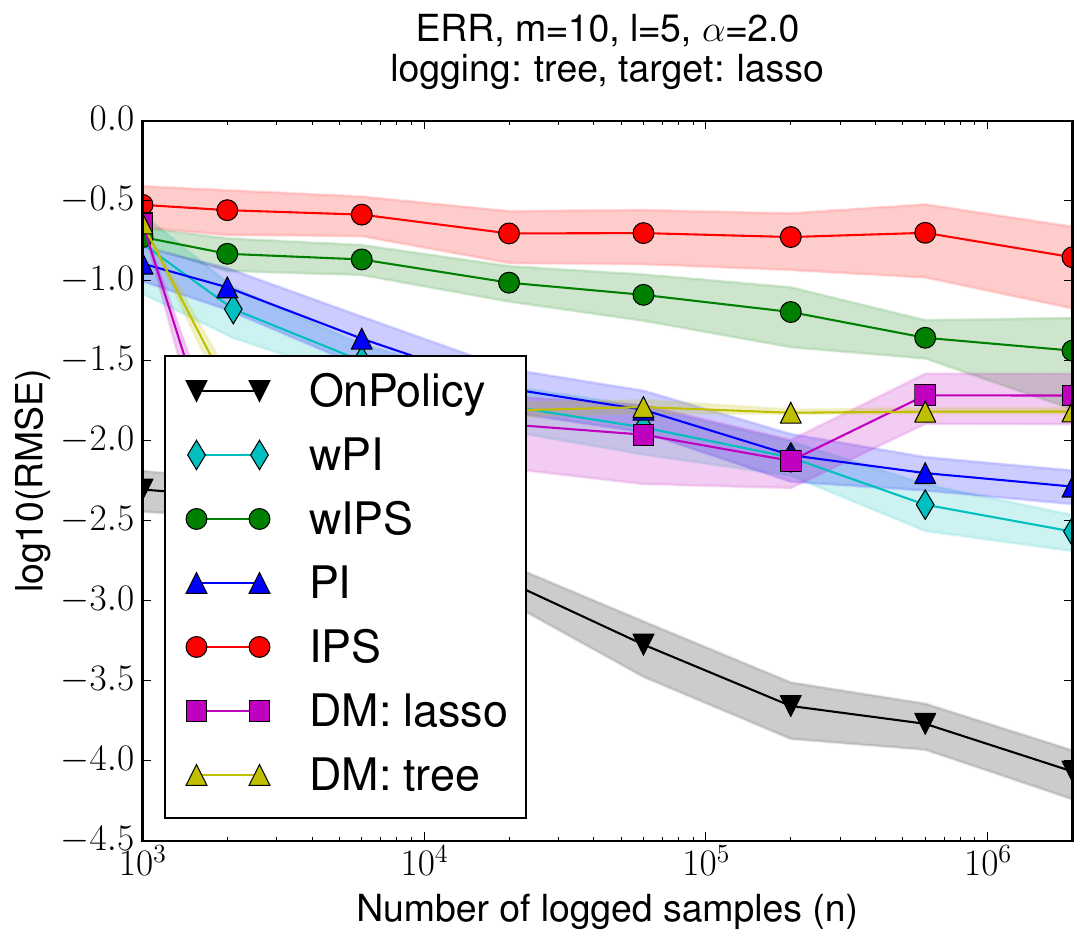}\\
\includegraphics[width=0.5\textwidth]{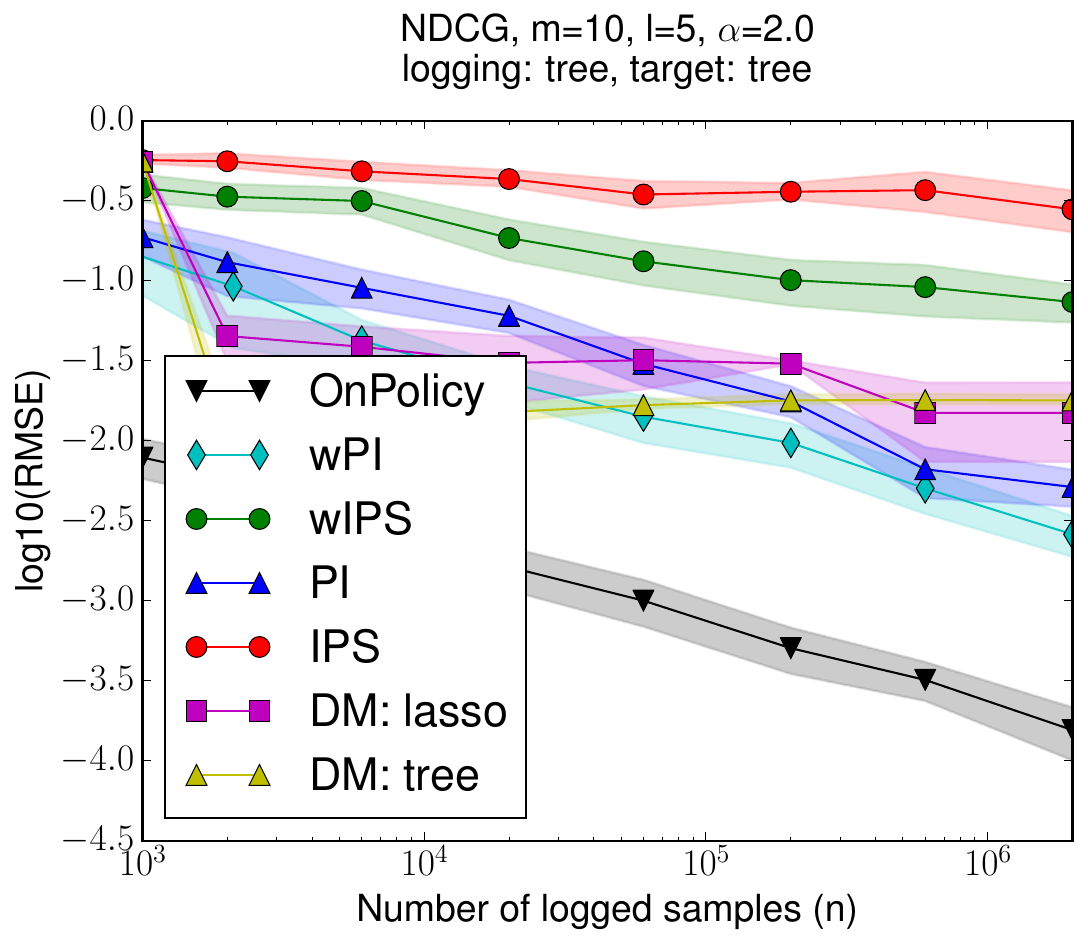}&
\includegraphics[width=0.5\textwidth]{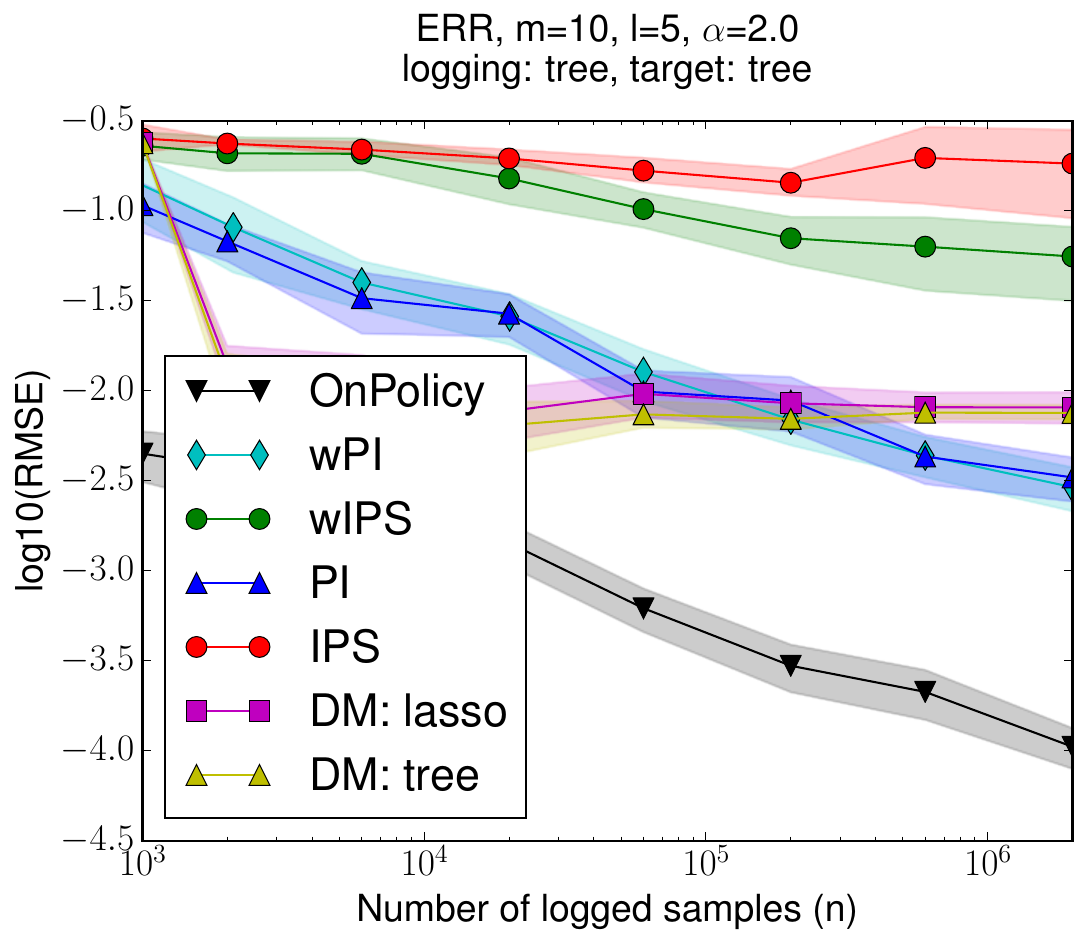}
\end{tabular}
\caption{RMSE of value estimators for an increasing logged dataset
  under a severely peaked logging policy ($\treeTitle, \alpha=2.0$)
  with slate space $(10,5)$. Target is $\lassoBody$ (top panel) and
  $\treeBody$ (bottom panel). Metrics are NDCG (left) and ERR
  (right).}
\end{figure*}

\begin{figure*}
\begin{tabular}{cc}
\includegraphics[width=0.5\textwidth]{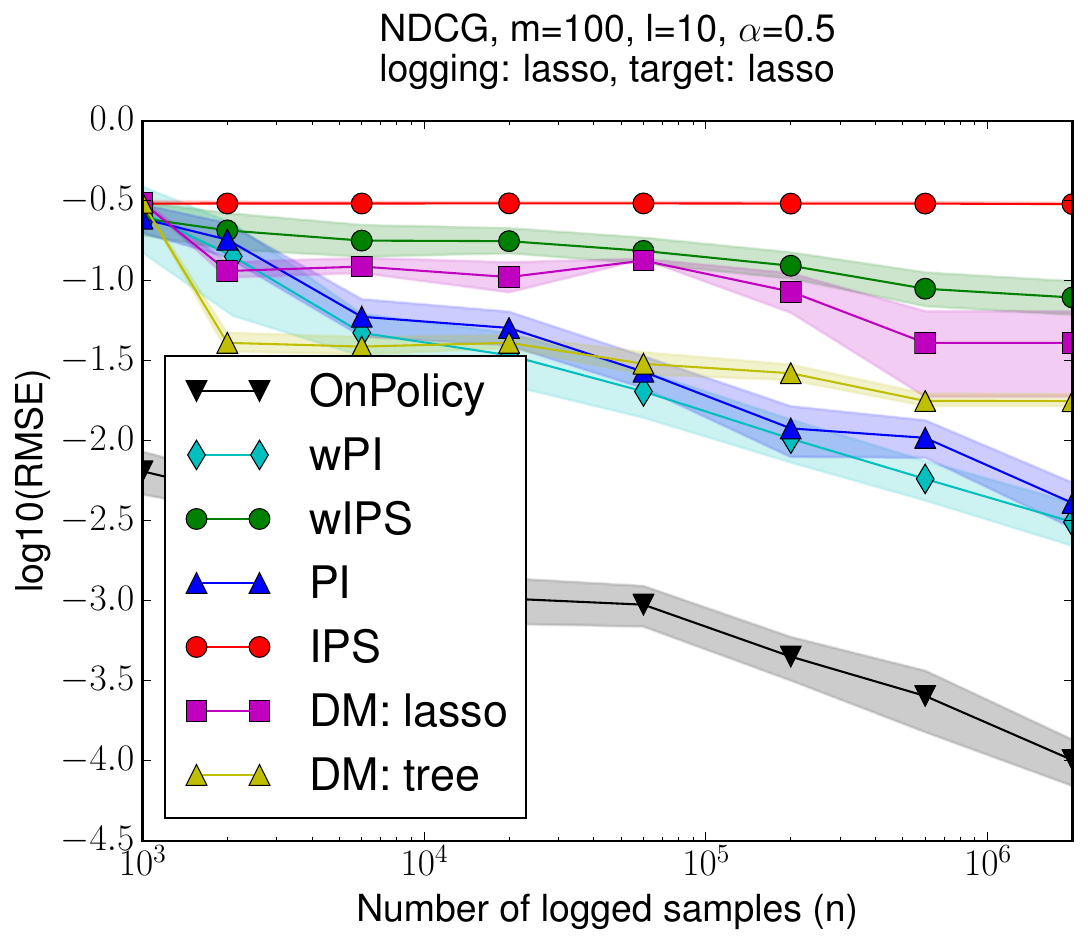}&
\includegraphics[width=0.5\textwidth]{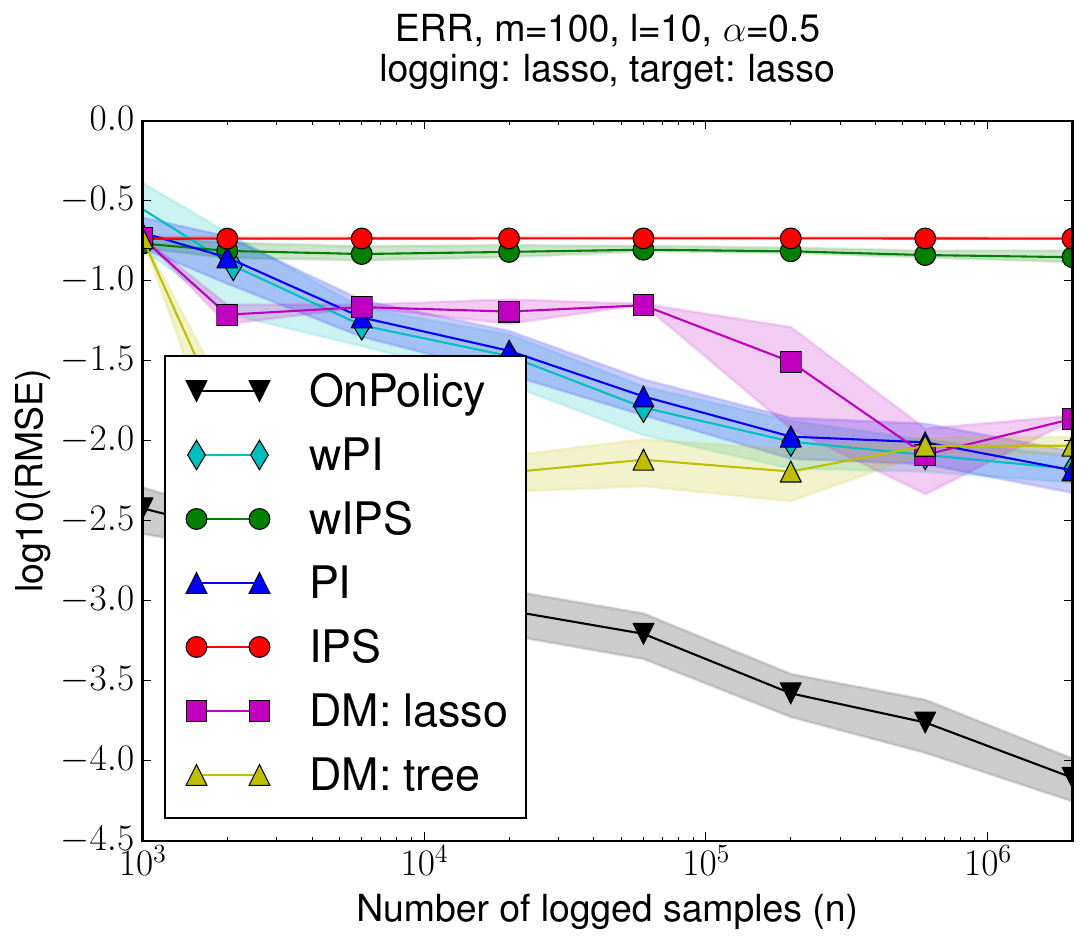}\\
\includegraphics[width=0.5\textwidth]{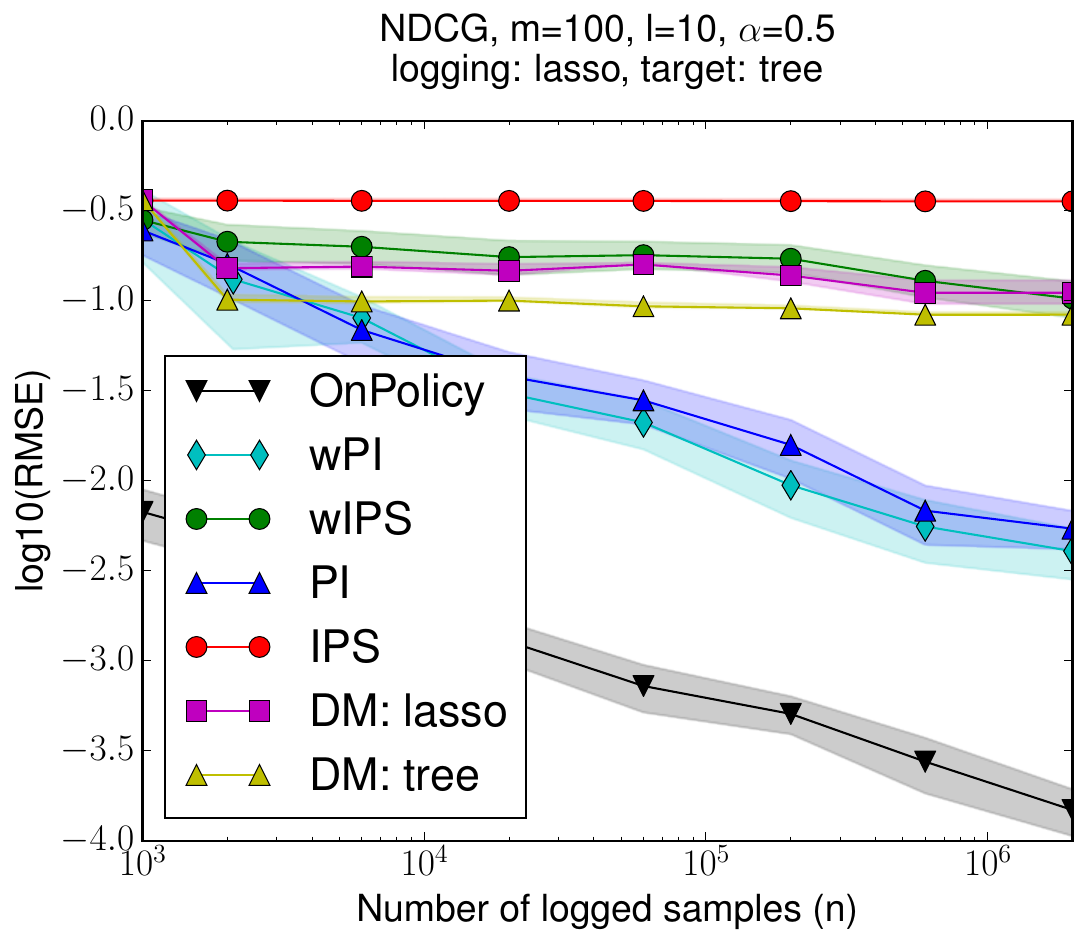}&
\includegraphics[width=0.5\textwidth]{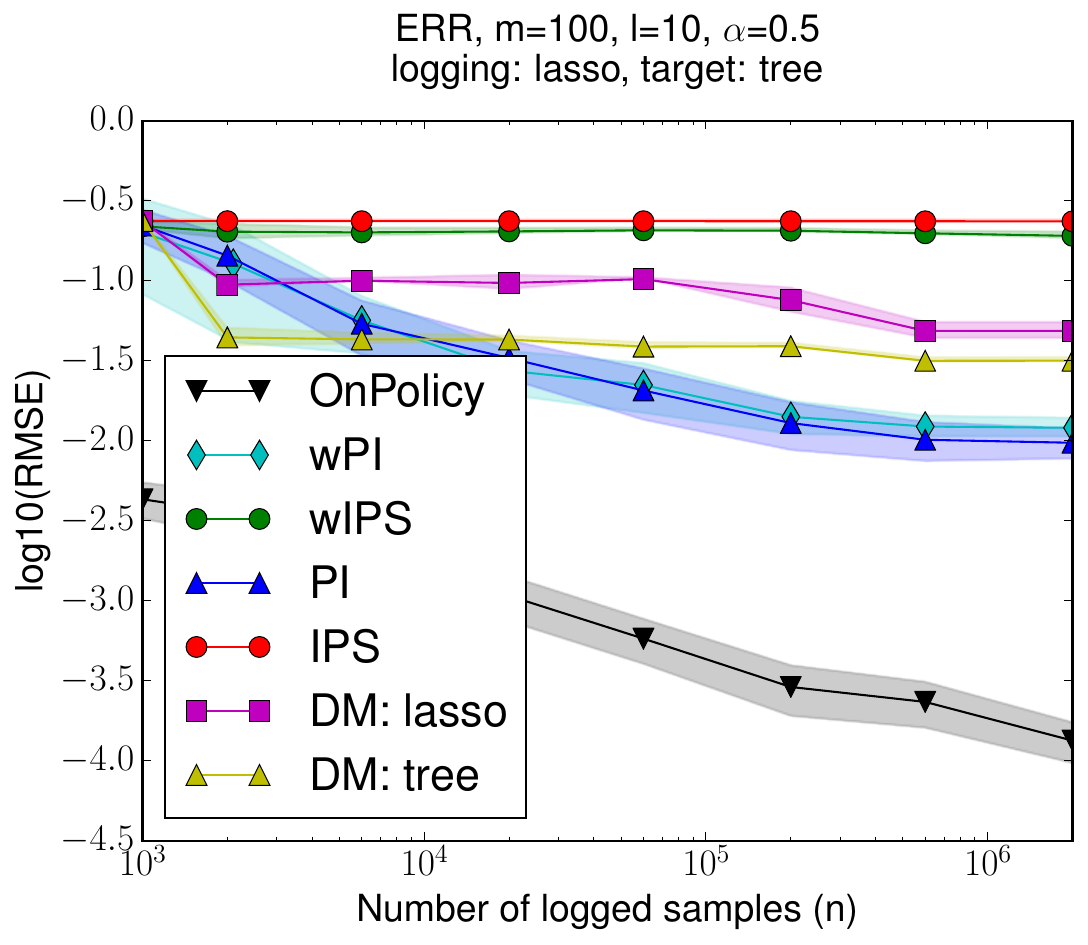}
\end{tabular}
\caption{RMSE of value estimators for an increasing logged dataset
  under a moderately peaked logging policy ($\lassoTitle, \alpha=0.5$)
  with slate space $(100,10)$. Target is $\lassoBody$ (top panel) and
  $\treeBody$ (bottom panel). Metrics are NDCG (left) and ERR
  (right).}
\end{figure*}

\begin{figure*}
\begin{tabular}{cc}
\includegraphics[width=0.5\textwidth]{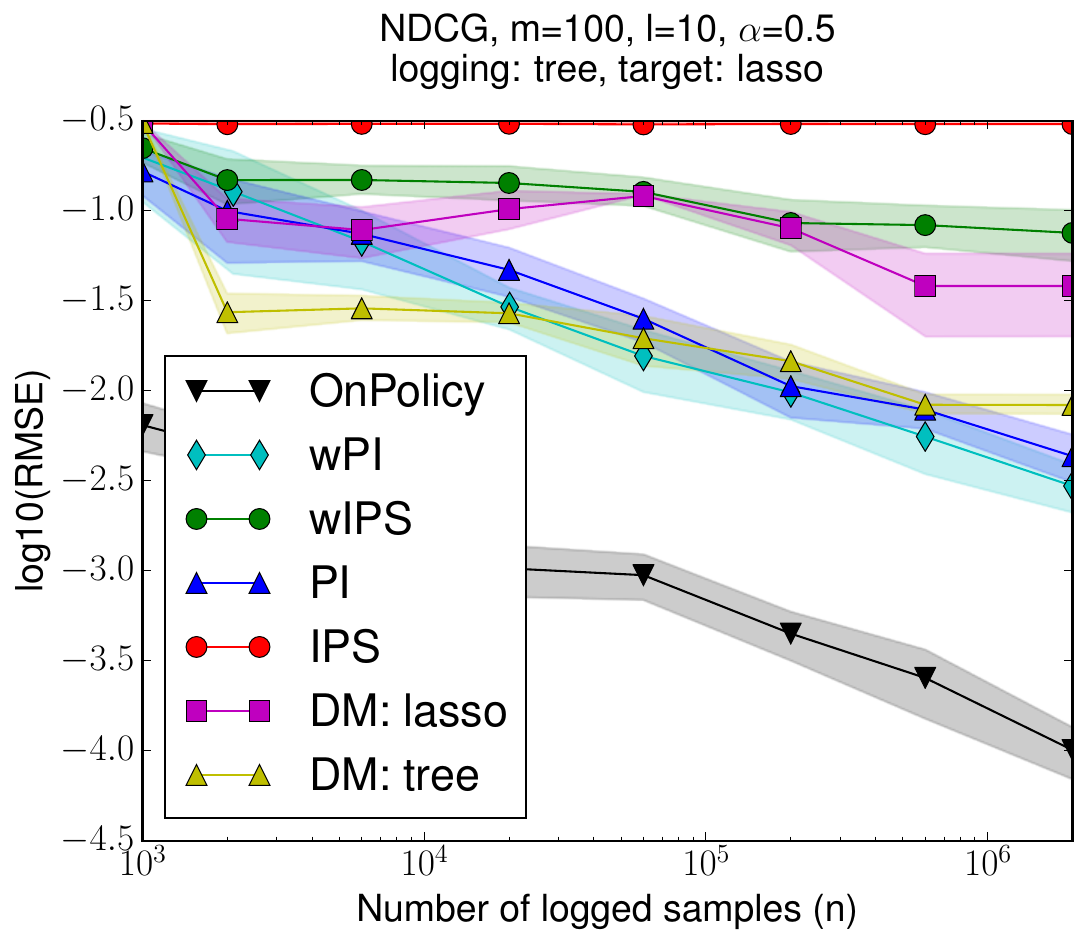}&
\includegraphics[width=0.5\textwidth]{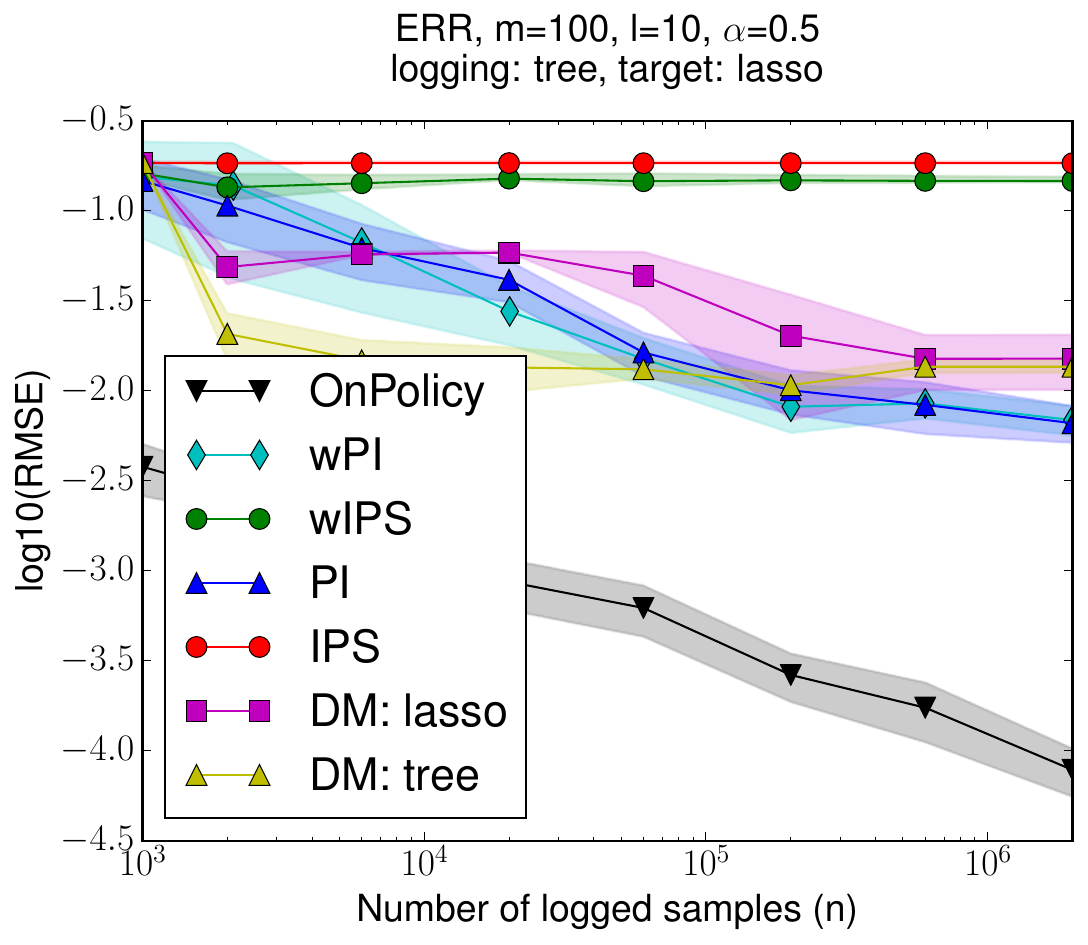}\\
\includegraphics[width=0.5\textwidth]{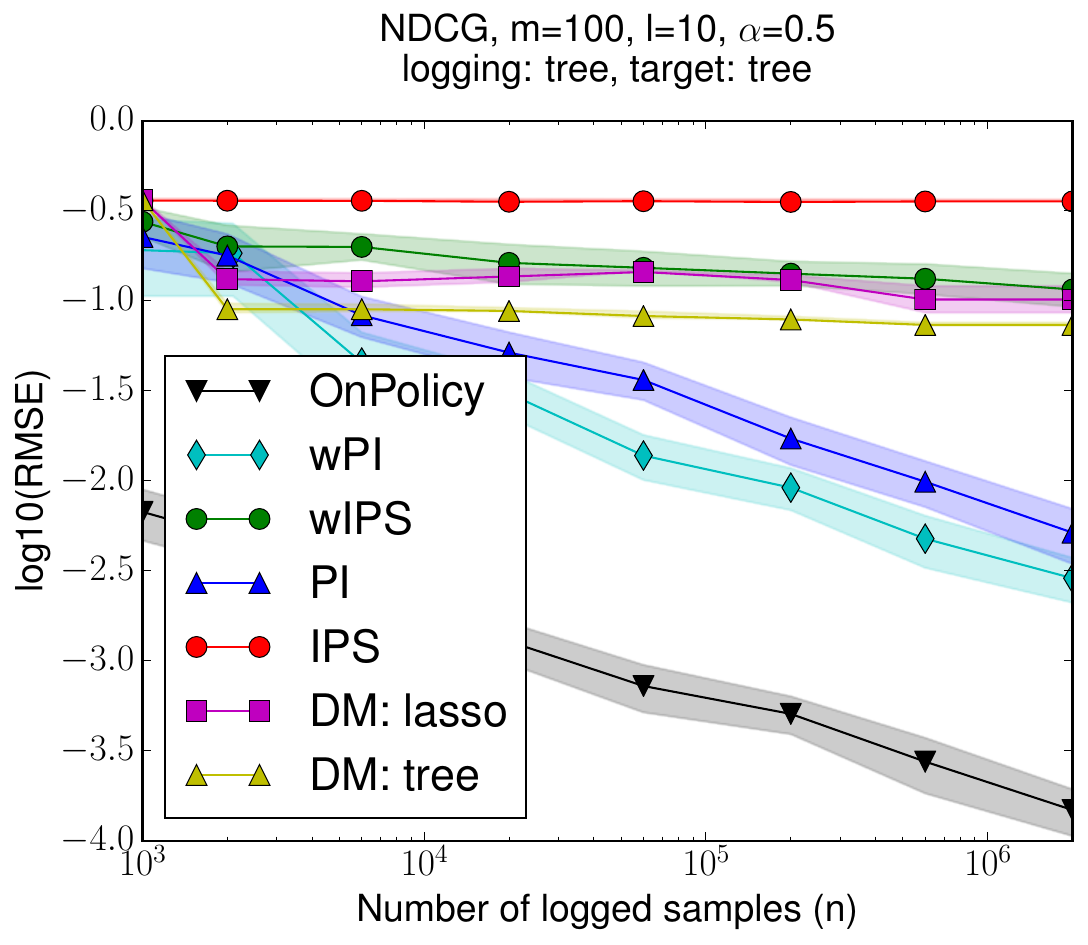}&
\includegraphics[width=0.5\textwidth]{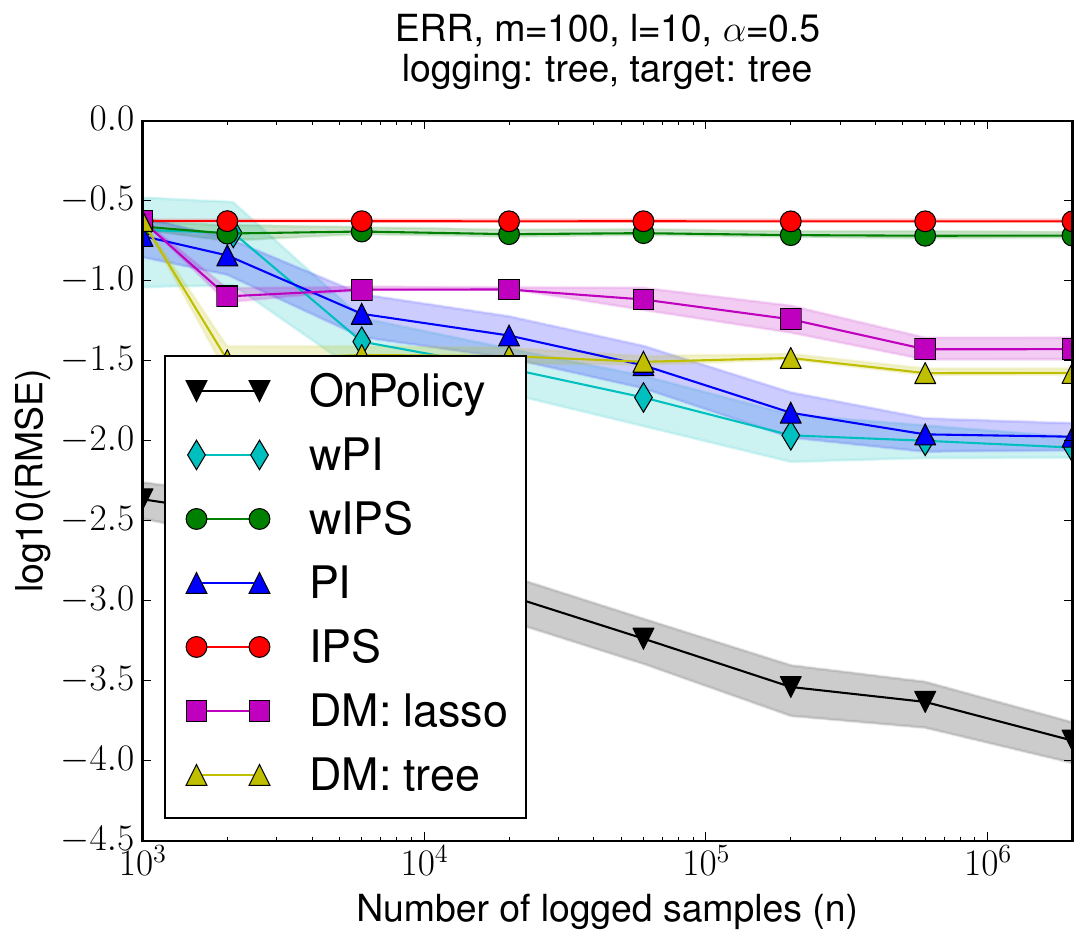}
\end{tabular}
\caption{RMSE of value estimators for an increasing logged dataset
  under a moderately peaked logging policy ($\treeTitle, \alpha=0.5$)
  with slate space $(100,10)$. Target is $\lassoBody$ (top panel) and
  $\treeBody$ (bottom panel). Metrics are NDCG (left) and ERR
  (right).}
\end{figure*}

\begin{figure*}
\begin{tabular}{cc}
\includegraphics[width=0.5\textwidth]{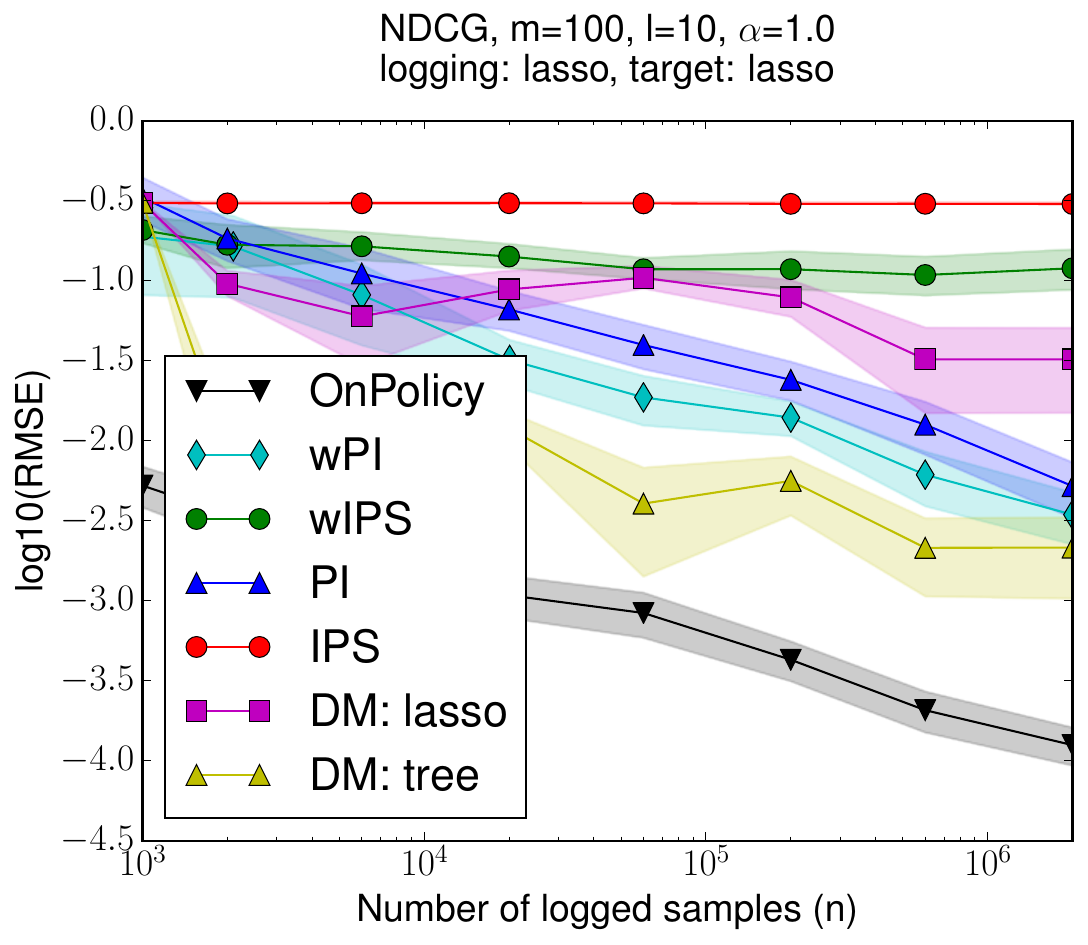}&
\includegraphics[width=0.5\textwidth]{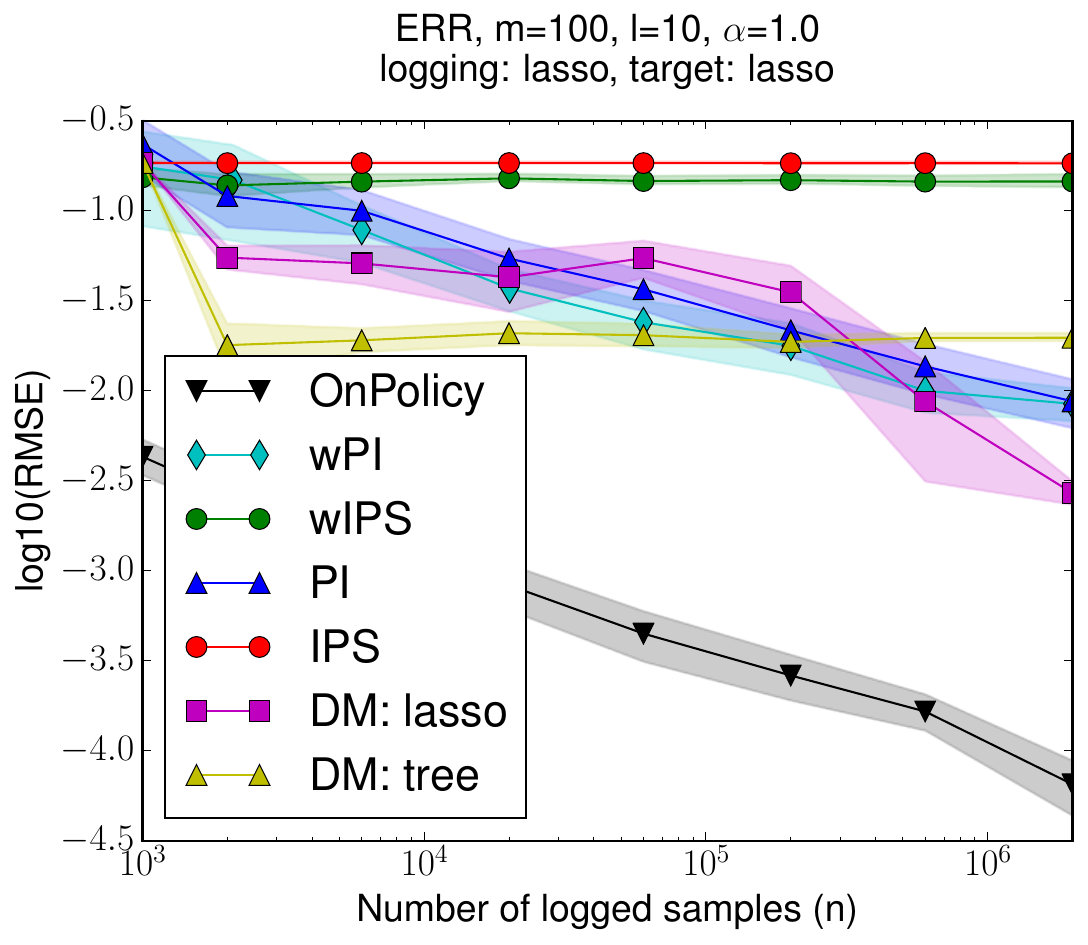}\\
\includegraphics[width=0.5\textwidth]{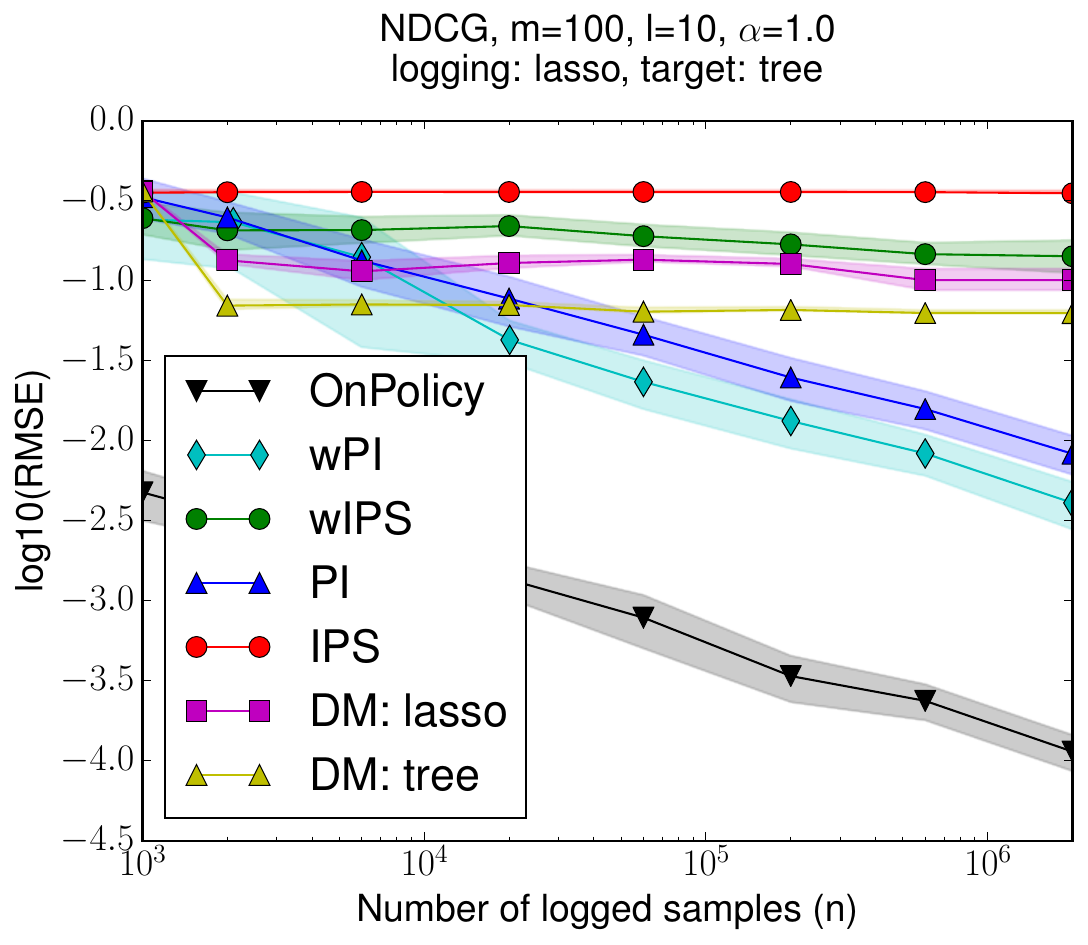}&
\includegraphics[width=0.5\textwidth]{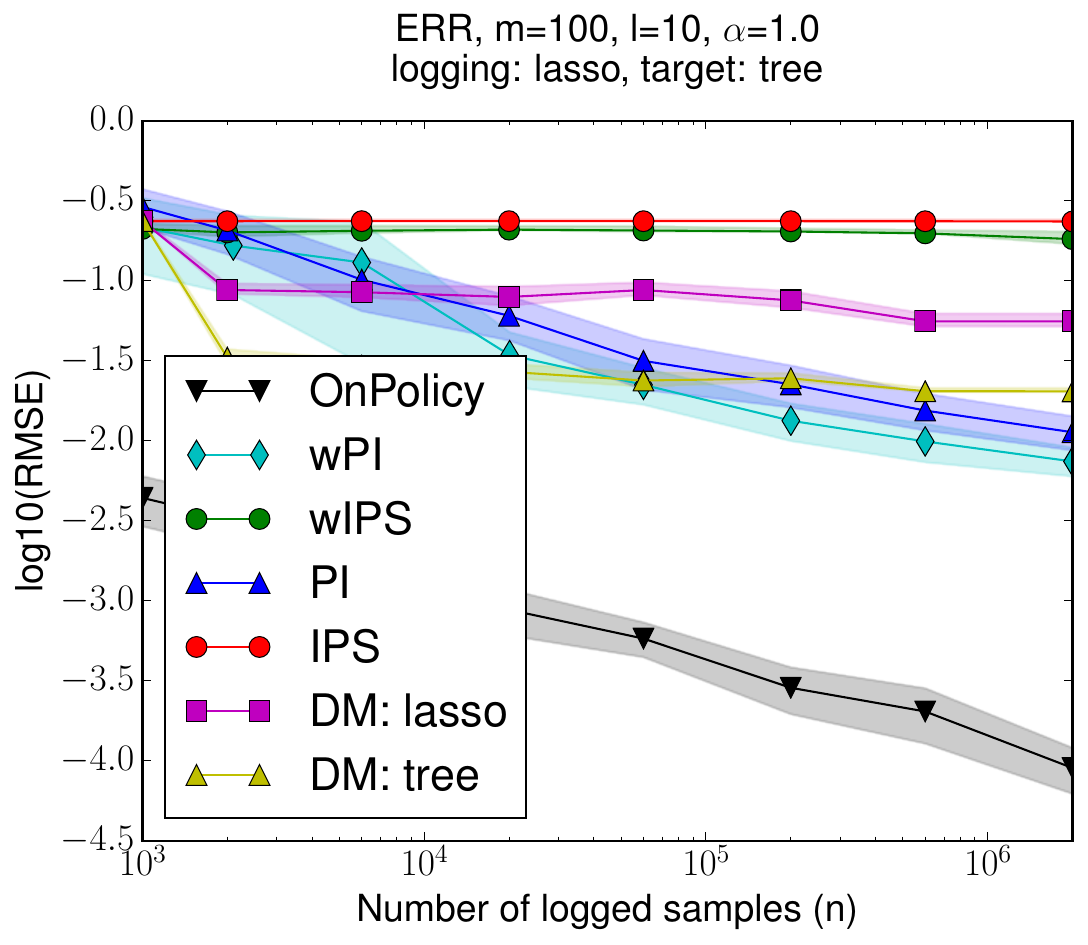}
\end{tabular}
\caption{RMSE of value estimators for an increasing logged dataset
  under a severely peaked logging policy ($\lassoTitle, \alpha=1.0$)
  with slate space $(100,10)$. Target is $\lassoBody$ (top panel) and
  $\treeBody$ (bottom panel). Metrics are NDCG (left) and ERR
  (right).}
\end{figure*}

\clearpage
\begin{figure*}
\begin{tabular}{cc}
\includegraphics[width=0.5\textwidth]{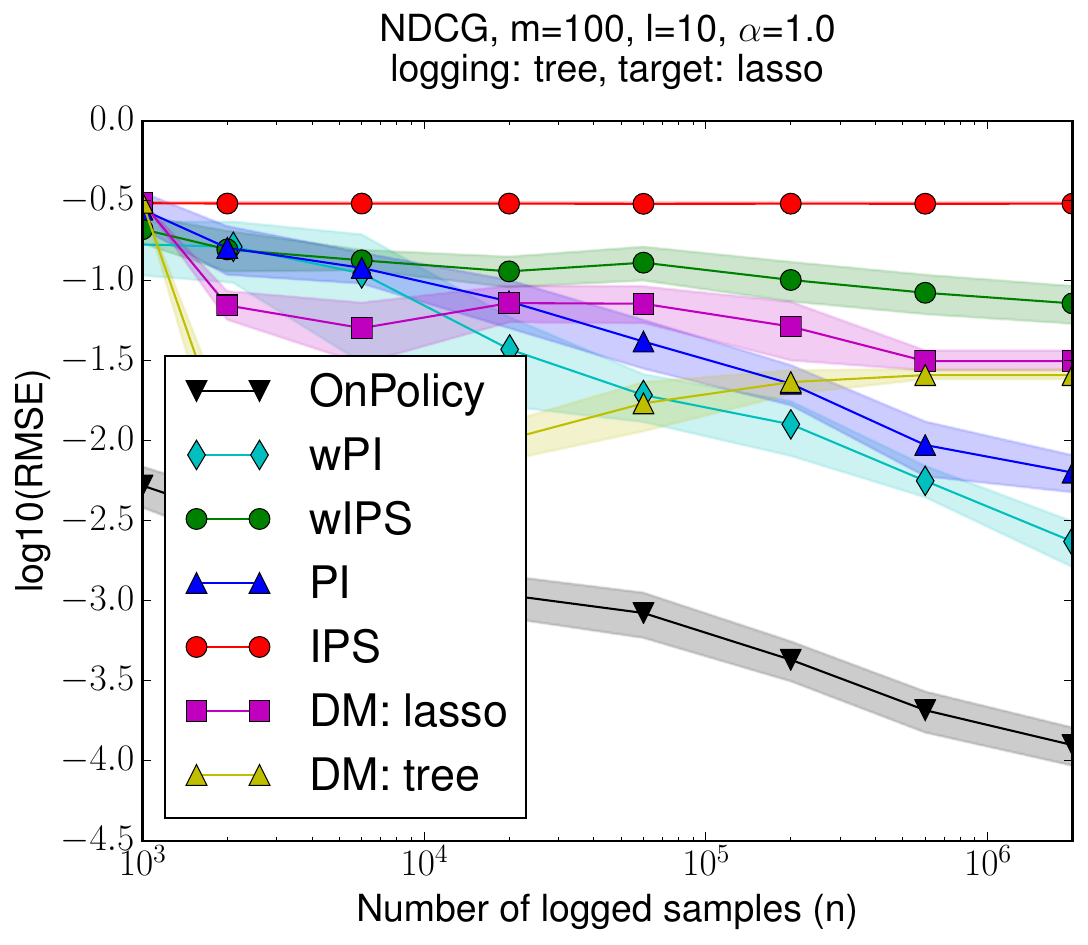}&
\includegraphics[width=0.5\textwidth]{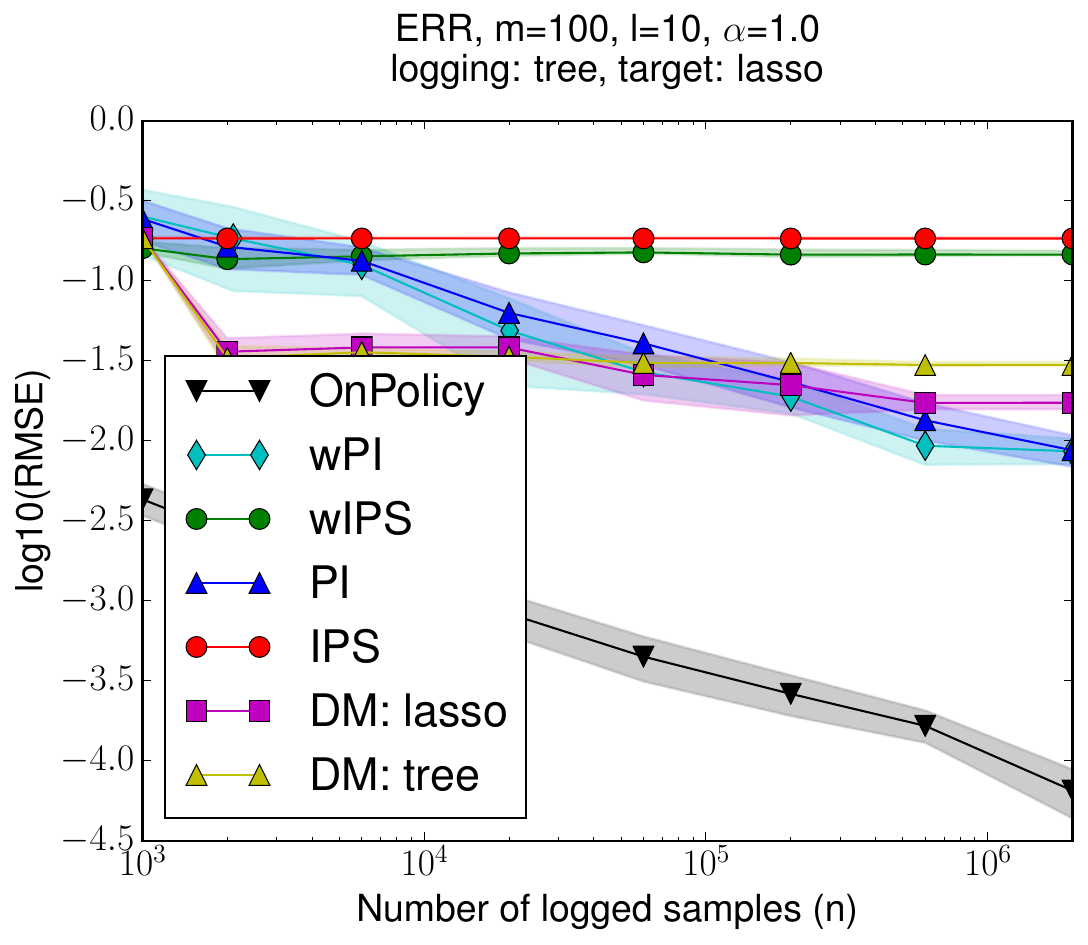}\\
\includegraphics[width=0.5\textwidth]{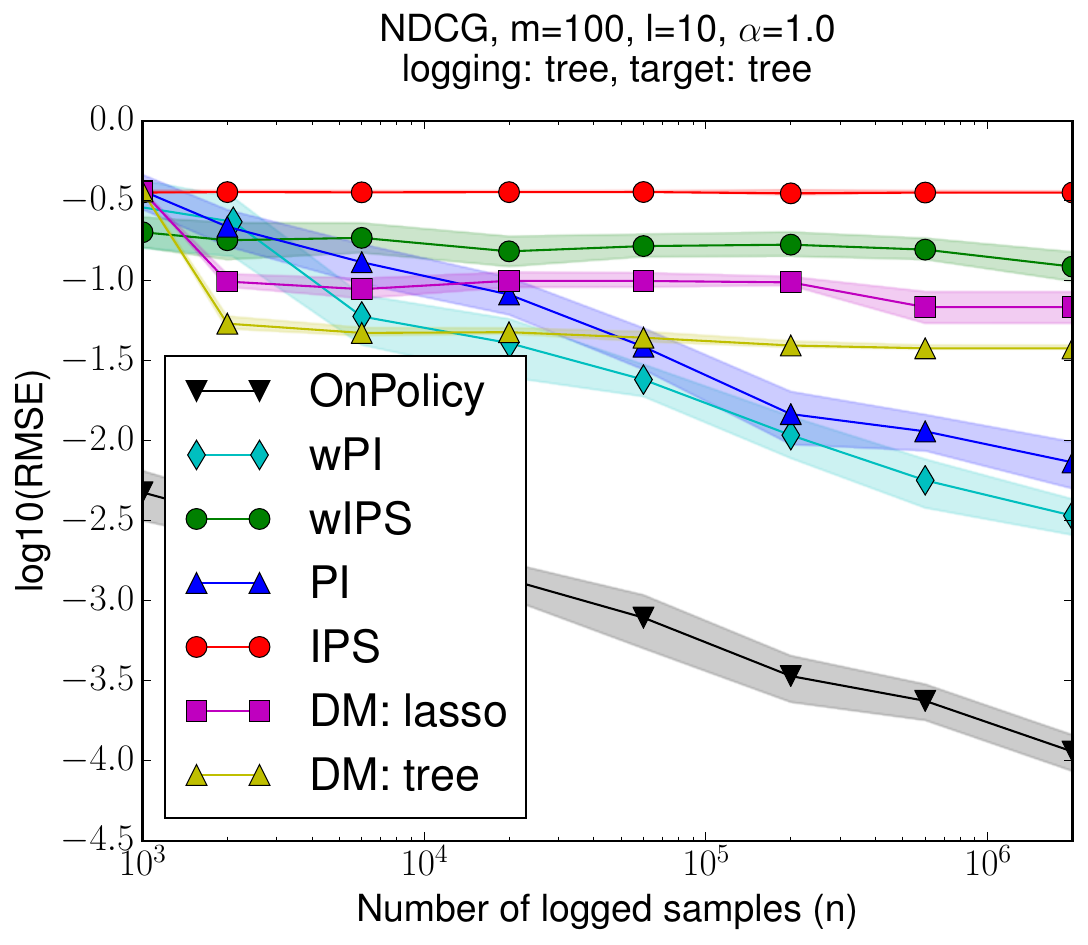}&
\includegraphics[width=0.5\textwidth]{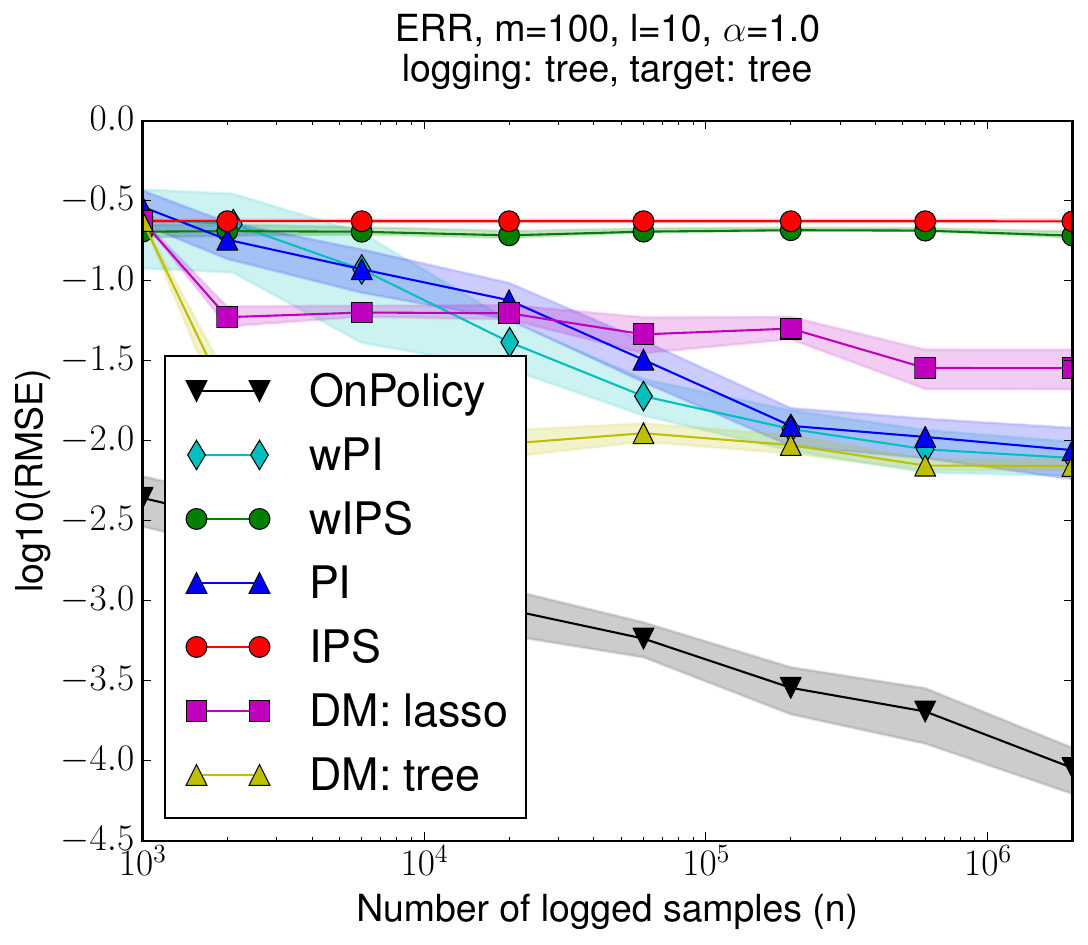}
\end{tabular}
\caption{RMSE of value estimators for an increasing logged dataset
  under a severely peaked logging policy ($\treeTitle, \alpha=1.0$)
  with slate space $(100,10)$. Target is $\lassoBody$ (top panel) and
  $\treeBody$ (bottom panel). Metrics are NDCG (left) and ERR
  (right).}
\label{fig:plots:last}
\end{figure*}

\section{Off-policy optimization}
\label{sec:opt_family}

For off-policy optimization experiments, we compare two methods -- SUP and PI-OPT on the MSLR-WEB10K dataset.
Both SUP and PI-OPT uses all the features $\vf(x,a)$ for query-document $(x,a)$ pairs in the training fold.
SUP uses regression targets as outlined in Section~\ref{sec:expt_opt}.
We also experimented with regression to raw relevance judgments, this is denoted SUP-\rel.
For PI-OPT, each query-document-\emph{position} triplet produces a regression example $(x,a,j)$ with a concatenated feature vector $\vf(x,a,j) \coloneqq [ \vf(x,a); \one_{j} ]$
where $\one_{j}$ is a $\ell$-dimensional one-hot encoding of position $j$.
Every logged sample with query $x$ yields an estimate $\hat{\phi}_j(x,a)$ for every candidate document $a$ and position $j$.
These are our natural regression targets.
There is one further optimization we can do that is computationally tractable when the set of queries $\{ x \}$ is finite.
By averaging all estimated $\hat{\phi}_j(x,a)$ for a particular query, we can create a lower-variance target for regression that remains an unbiased estimate of $\phi_j(x,a)$.

Both SUP and PI-OPT employ gradient boosted regression trees with 
$n=1000$ tree-ensembles and up to $70$ leaves in each tree,
to predict their corresponding regression targets.
With a trained model, SUP constructs slates in a straightforward way: For any input query $x$,
we score all candidate documents $a \in A(x)$ using the trained model $\vf(x,a) \mapsto score(a)$, and sort the scores in descending order.
Rankings are constructed using the top-$\ell$ scoring candidates in order.
For PI-OPT, we score every document-position pair $(a,j) \in A(x) \times \{1 \cdots \ell\}$, $\vf(x,a,j) \mapsto score(a,j)$.
Now we greedily pick the highest scored pair $(a,j)$ and insert document $a$ in slot $j$ of the slate.
After eliminating all invalid pairs, $(\ast, j)$ and $(a, \ast)$, we repeat this greedy procedure until all positions in the slate are filled.
This gives us a computationally efficient, albeit approximate, maximizer of $\argmax_\vs \sum_{j=1}^l score(s_j, j)$.
%
%
\section{Overlap between base-rankers}
\label{app:overlaps}
We use four different base-rankers $\lassoTitle,\lassoBody,\treeTitle,\treeBody$
in our semi-synthetic experiments to instantiate logging and target policies.
In Table~\ref{tbl:app:overlaps}, we report how similar the top-$\ell$ rankings
($\ell=10$) retrieved by these rankers are.
We report two metrics for every pair of rankers: the average fraction of documents retrieved in common by both rankers (and its standard deviation),
and the Kendall's tau computed over the union of documents retrieved by either ranker (documents retrieved by one ranker but not the other are assumed to be ranked at $\ell+1$ in the other ranking).
\begin{table}[ht]
\begin{center}\begin{small}
\caption{Reporting the difference between the base-rankers $\lassoTitle,\lassoBody,\treeTitle,\treeBody$ as measured by average overlap of retrieved document sets and Kendall's tau.}
\label{tbl:app:overlaps}
\begin{tabular}{|l||c|c||c|c||}
\hline
Pair & \multicolumn{2}{|c||}{Overlap} & \multicolumn{2}{|c||}{Kendall's $\tau$} \\
& Avg. & Std.Dev. & Avg. & Std.Dev.\\
\hline
$(\lassoTitle, \treeTitle)$ & 0.523 & 0.216 & -0.041 & 0.307 \\
$(\lassoBody, \treeBody)$ & 0.426 & 0.236 & -0.221 & 0.322 \\
$(\treeTitle, \treeBody)$ & 0.270 & 0.198 & -0.394 & 0.236 \\
$(\treeTitle, \lassoBody)$ & 0.274 & 0.203 & -0.405 & 0.239 \\
$(\lassoTitle, \treeBody)$ & 0.250 & 0.199 & -0.421 & 0.231 \\
$(\lassoTitle, \lassoBody)$ & 0.262 & 0.202 & -0.415 & 0.233 \\
\hline
\end{tabular}
\end{small}
\end{center}
\end{table}

\end{appendices}

\end{document}